\documentclass{article}
\usepackage{iclr2027_conference,times}

\usepackage[T1]{fontenc}
\usepackage{url}

\usepackage{amsmath,amssymb,amsthm,mathtools}
\usepackage{mathrsfs}

\usepackage{graphicx}
\usepackage{subcaption}
\usepackage{booktabs,adjustbox,array}
\usepackage{float}

\usepackage{enumitem}
\usepackage{algorithm,algorithmic}
\usepackage{placeins}

\usepackage{xcolor}
\usepackage{microtype}

\usepackage[hidelinks]{hyperref}

\theoremstyle{plain}
\newtheorem{theorem}{Theorem}
\newtheorem{lemma}[theorem]{Lemma}
\newtheorem{proposition}[theorem]{Proposition}
\newtheorem{corollary}[theorem]{Corollary}

\theoremstyle{definition}
\newtheorem{assumption}{Assumption}

\theoremstyle{remark}

\newcommand{\dif}{\,\mathrm{d}}

\newcommand{\ind}{\mathbf{1}}
\newcommand{\Frob}[1]{\lVert #1\rVert_{\mathrm F}}

\renewcommand{\leq}{\leqslant}
\renewcommand{\le}{\leqslant}
\renewcommand{\geq}{\geqslant}
\renewcommand{\ge}{\geqslant}

\newcommand{\supp}{\operatorname{supp}}
\newcommand{\softmax}{\operatorname{softmax}}
\newcommand{\R}{\mathbb R}
\newcommand{\E}{\mathbb E}
\newcommand{\norm}[1]{\left\lVert #1\right\rVert}

\newcommand{\Ntask}{N_{\mathrm{task}}}
\newcommand{\Nobs}{N_{\mathrm{obs}}}

\newcommand{\A}{\mathcal{A}}
\newcommand{\Iset}{\mathcal{I}}

\newcommand{\GL}{\operatorname{GL}}
\newcommand{\M}{\mathcal{M}}
\newcommand{\U}{\mathcal U_\tau}
\newcommand{\Ub}{\overline{\mathcal U}_\tau}
\newcommand{\Tset}{\mathcal K}
\newcommand{\Ssixteen}{\mathcal S_{16}}
\newcommand{\TF}{\operatorname{TF}}
\newcommand{\out}{\mathrm{out}}

\newcommand{\Obs}{\mathcal B_{\tau,n}}
\newcommand{\Coef}{\mathcal V}

\title{%
From Policy Evaluation to Off-Policy Control:\\[0.2em]
Provable Emergence of In-Context Q-Learning
}

\author{%
Yingjun Mo\textsuperscript{1}
\qquad
Yajie Zhao\textsuperscript{2}
\qquad
Fang Kong\textsuperscript{1}\thanks{%
Corresponding author: \texttt{kongf@sustech.edu.cn}.%
}
\\[0.8em]
{\small
\textsuperscript{1} Southern University of Science and Technology (SUSTech)
}
\\[0.3em]
{\small
\textsuperscript{2} University of Electronic Science and Technology of China (UESTC)
}
}

\date{}

\hypersetup{
    pdftitle={From Policy Evaluation to Off-Policy Control: Provable Emergence of In-Context Q-Learning},
    pdfauthor={Yingjun Mo, Yajie Zhao, Fang Kong}
}

\iclrfinalcopy

\makeatletter
\renewcommand{\@maketitle}{%
    \vbox{%
        \hsize=\textwidth
        \centering
        {\Large\bfseries \@title\par}
        \vskip 1.2em
        {\normalsize\normalfont \@author\par}
        \vskip 1.5em
    }%
}
\makeatother

\renewcommand{\headrulewidth}{0pt}

\begin{document}

\maketitle


\begin{abstract}
Transformers can implement reinforcement-learning algorithms in context, yet how off-policy control emerges within Transformers remains poorly understood. We study Q-learning, a foundational model-free control algorithm that combines off-policy bootstrapping with greedy policy improvement. We prove that a two-layer attention-only Transformer can implement a finite-temperature form of Q-learning, where the temperature $\tau>0$ controls the approximation to greedy action selection, recovering window Q-learning as $\tau\to0$. We train the structured Transformer on randomly generated MDP tasks using output-only supervision. Experiments show that the theoretically predicted parameter structure emerges during training, including the internal greedy action-selection mechanism, and that the trained Transformer achieves performance comparable to Q-learning on unseen tasks. Our theoretical analysis further shows that the equivalent realizations form a solution manifold and that, under suitable task-diversity conditions, the training iterates converge locally to this manifold at an exponential rate. For the random-task model analyzed here, the local worst-case number of gradient steps needed to reduce the distance to the manifold by a factor $\epsilon$ is $\Theta\left(\tau^{-5}\log(1/\varepsilon)\right)$, quantifying the increasing optimization cost as the learned update approaches window Q-learning.
\end{abstract}

\section{Introduction}
In-context reinforcement learning (ICRL) offers an alternative to the conventional paradigm of reinforcement learning. Rather than adapting to a new task through parameter updates, a pretrained agent can adjust its behavior directly from the interaction history provided in its context. This phenomenon has motivated an algorithmic interpretation of ICRL: a Transformer may not merely map a trajectory to a prediction, but may execute a reinforcement-learning algorithm within its forward pass. Recent theoretical work supports this view.
Transformers have been shown to implement temporal-difference (TD) methods for policy evaluation, and more recent results extend this perspective to policy-improvement methods such as SARSA and actor-critic \citep {wang2024transformers,wang2025emergence,xie2026convergencecot,xie2026beyond,liang2026transformers}.

A central regime remains largely unexplored: off-policy control. Existing policy-improvement constructions focus primarily on on-policy learning, where the behavior policy and the policy being improved are tightly coupled \citep{liang2026transformers}. This coupling is restrictive in real applications where an agent must learn from previously collected, heterogeneous, or suboptimal experience rather than from trajectories generated by its current policy. Off-policy control asks whether useful policy improvement can still be carried out when these two roles are decoupled. 
Q-learning provides a canonical setting for studying this problem. Given a transition $(s,a,r,s')$, its target is $r+\gamma\max_{a'}Q(s',a')$,
which does not depend on the successor action selected by the behavior policy. Thus, subject to sufficient state-action coverage, Q-learning can use experience generated by a suboptimal behavior policy while improving toward a greedy target policy. This makes Q-learning a natural testbed for a more demanding form of in-context adaptation: whether a Transformer can use off-policy experience to carry out control, rather than merely evaluate or improve the policy that generated its context.


This transition is substantially more than a change of update rule. At a conceptual level, it marks the passage from Bellman evaluation to Bellman optimality. An in-context implementation of Q-learning must compare competing successor actions, resolve the resulting action selection, and then use that decision to construct the bootstrap target. We show that there exists a two-layer Transformer, combining masked softmax attention with linear attention, that jointly implements next-state action selection and the corresponding window semi-gradient Q-learning update in finite-temperature form, where the softmax temperature $\tau>0$ controls the degree of approximation to greedy action selection, and standard window (hard-max) Q-learning is recovered in the limit $\tau\to0$. In this sense, the architecture of transformer does not merely reproduce the numerical output of Q-learning; its internal computation mirrors the decomposition of the algorithm itself.

What kind of pretraining can give rise to such a Transformer? We develop a gradient-based training algorithm in which a soft window-Q update serves as the teacher for each sampled context. Within the prescribed tokenization and attention mask, training supervises only the output update and imposes no numerical targets on the internal parameter blocks to implement either window Q-learning or action selection. Nevertheless, the constructed parameter structure emerges under training, recovering the underlying computation, including the internal action-selection mechanism. On unseen tasks, the trained Transformer achieves performance comparable to window Q-learning with the same step size.

To understand why a Q-learning implementation can emerge under this training procedure, we provide a theoretical analysis of its convergence. The nonlinear dependence of softmax attention on the action-score parameters makes the analysis substantially more intricate than in linear-attention settings. Meanwhile, the Q-learning realization is inherently non-identifiable, with different parameter configurations implementing the same update. This makes the relevant optimization target a manifold of functionally equivalent Q-learning realizations rather than a single parameter vector.
We characterize this realization manifold and show that, under verifiable conditions on the training distribution, the discrete-time training dynamics converge locally to this manifold at an exponential rate. The analysis further reveals a quantitative difficulty in approaching greedy Q-learning: for the random-task model considered here,
the local worst-case number of gradient steps needed to reduce the distance to the manifold by a factor $\epsilon$ is $\Theta\left(\tau^{-5}\log(1/\varepsilon)\right)$, quantifying the increasing optimization cost as the learned update approaches standard Q-learning.

Our results make three concrete advances. First, we give an explicit attention-only construction that jointly realizes successor-action comparison and the window Q-learning update. Second, we design a gradient-based training procedure and characterize a smooth manifold of functionally equivalent Q-learning realizations, showing local exponential convergence of the training dynamics to this manifold. Third, we quantify how the softmax temperature governs the tradeoff between approximating greedy action selection and optimization difficulty. Together, these results extend the mechanistic analysis of ICRL from policy evaluation and on-policy control to off-policy Q-learning.
\section{Related work}
\label{sec:related}
Transformer-based agents have been studied extensively for sequential decision
making. Decision Transformer and Trajectory Transformer formulate reinforcement
learning as sequence modeling~\citep{chen2021decision,janner2021offline}, while
Algorithm Distillation and related ICRL methods show that pretrained sequence
models can adapt to unseen tasks through interaction histories without test-time
parameter updates~\citep{laskin2023incontext,lee2023supervised,grigsby2024amago}.
On the theoretical side, \citet{lin2024transformers} establish provable ICRL
under supervised pretraining and show that Transformers can approximate
classical online decision-making algorithms. 

Recent work has taken a mechanistic view of ICRL, asking which RL algorithms can be implemented by Transformer forward computation and whether the corresponding algorithmic parameters can emerge through pretraining. \citet{wang2024transformers,wang2025emergence} show that linear-attention Transformers can implement TD updates for policy evaluation, while \citet{xie2026beyond} extend this result to standard softmax attention via a kernel-weighted TD construction. Beyond policy evaluation, \citet{liang2026transformers} show that Transformers can realize semi-gradient SARSA and actor-critic updates for policy improvement. However, analogous results for off-policy control, including classical Q-learning, where trajectory data are coupled with subsequent action selection, remain largely unexplored. We study the emergence of Q-learning updates within Transformers, including the Transformer construction, pretraining algorithm, and its convergence analysis. Unlike \citet{liang2026transformers}, we focus on the effect of the action-selection mechanism on convergence, and our analysis accounts for discrete-time training iterates and provides verifiable convergence conditions.

Off-policy learning is particularly relevant to ICRL when context trajectories are offline, heterogeneous, or generated by suboptimal behavior policies. In such settings, the quality of the learned policy need not be tied to that of the behavior generating the trajectories, making off-policy value-based learning particularly relevant.
\citet{dong2025suboptimal} study ICRL from suboptimal historical trajectories
and propose an advantage-weighted policy-learning framework to improve upon the
behavior policies in the data. Scalable In-Context Q-Learning
\citep{liu2026sicql} combines context-conditioned value estimation, dynamic
programming, and policy extraction to enable iterative improvement from
suboptimal offline experience. In-Context Compositional Q-Learning
\citep{xu2026icql} instead formulates Q-function estimation as contextual inference and uses retrieved transitions to infer local Q-functions for offline
RL. Our work differs fundamentally from these studies by taking a mechanistic view of how Transformers realize Q-learning in ICRL, characterizing both the forward-pass implementation of the Q-learning update and the emergence of the corresponding parameter structure through training.
\section{Background}
\subsection{Transformers and Self‑Attention} All vectors are treated as column vectors. Let $I_n$ denote the $n\times n$ identity matrix, $0_{m\times n}$ the $m\times n$ zero matrix, and $0_m$ the $m$-dimensional zero vector. 

Let $\mathbf H=[h_1,\ldots,h_n]\in\mathbb R^{D\times n}$ denote a prompt representation, where each column $h_i\in\mathbb R^D$ corresponds to one token. The standard single-head softmax self-attention \citep{vaswani2017attention} processes the prompt via $\operatorname{SoftAttn}_{\mathbf P,\mathbf V}(\mathbf H)
\triangleq
(\mathbf V\mathbf H)
\operatorname{softmax}\left(\mathbf H^\top\mathbf P\mathbf H\right)$, where $\mathbf V\in\mathbb R^{D\times D}$ is the value projection matrix,
$\mathbf P:=\mathbf W_K^\top\mathbf W_Q\in\mathbb R^{D\times D}$ collects the key and query projection parameters, the conventional scaling factor in scaled dot-product attention is absorbed into $\mathbf P$, and the softmax is applied independently to each query column, normalizing the attention scores over all key positions. 

As a widely studied alternative to softmax self-attention, linear self-attention replaces the softmax-normalized attention weights with unnormalized bilinear attention scores (\cite{ahn2023transformers}, \cite{gatmiry2024can}, \cite{zhang2024trained}). Its formulation is given by
$\operatorname{LinAttn}_{\mathbf P,\mathbf V}(\mathbf{H})\triangleq \frac{1}{n}(\mathbf{V}\mathbf{H})(\mathbf{H}^\top \mathbf{P} \mathbf{H})$, where the factor $\frac{1}{n}$ normalizes the aggregation with respect to the prompt length.

To model a multi-layer Transformer structure, we stack $L$ attention layers with residual connections.
Let $\mathbf{H}^{(0)} = \mathbf{H} \in \mathbb{R}^{D\times n}$ denote the initial input prompt representation.
For each layer $\ell \in \{1,2,\dots,L\}$, the layer-wise update rule reads
\begin{equation}\label{eq:Transformer}
\mathbf{H}^{(\ell)} := \mathbf{H}^{(\ell-1)} + \operatorname{Attn}_{\mathbf{P}^{(\ell)},\mathbf{V}^{(\ell)}}\big(\mathbf{H}^{(\ell-1)}\big) \in \mathbb{R}^{D\times n},
\end{equation}
where $\operatorname{Attn}$ denotes either $\operatorname{SoftAttn}$ or $\operatorname{LinAttn}$, and $\{\mathbf{P}^{(\ell)},\mathbf{V}^{(\ell)}\}$ are the layer-specific projection parameters at layer $\ell$.
After stacking $L$ attention layers, the final Transformer output is defined as
\begin{align}\label{eq:TFout}
    \mathrm{TF}_L\big(\mathbf{H};\{\mathbf{P}^{(\ell)},\mathbf{V}^{(\ell)}\}_{\ell=1}^L\big):= f_{\mathrm{read}}\big(\mathbf{H}^{(L)}\big),
\end{align}
where $f_{\mathrm{read}} \colon \mathbb{R}^{D\times n} \to \mathbb{R}^{D_{\mathrm{out}}}$ denotes a fixed readout map. Equation \eqref{eq:TFout} establishes the notation convention we adopt for defining the output of an L-layer Transformer. Specifically, the attention mechanism produces a matrix, while we require a vector output for policy improvement. Following prior works (\cite{xie2026beyond}, \cite{liang2026transformers},\cite{wang2024transformers}), we use the bottom-right \(d\)-dimensional vector of the output matrix as this vector.

\subsection{Reinforcement Learning and Q-learning}
We consider an infinite-horizon discounted Markov decision process (MDP) $\mathcal E=(\mathcal{S},\mathcal{A},P,R,\gamma)$,
where $\mathcal{S}$ is the state space, $\mathcal{A}$ is the action space,
$P(\cdot\mid s,a)\in\mathcal{P}(\mathcal{S})$ is the transition kernel,
$R:\mathcal{S}\times\mathcal{A}\times\mathcal{S}\to\mathbb{R}$ is the reward function, and
$\gamma\in[0,1)$ is the discount factor.
Here, $\mathcal{P}(\mathcal{X})$ denotes the set of probability distributions over $\mathcal{X}$.
At time $t$, an agent in state $S_t$ selects an action
$A_t\sim\pi(\cdot\mid S_t)$ according to a policy
$\pi:\mathcal{S}\to\mathcal{P}(\mathcal{A})$,
receives a reward $R_{t+1}=R(S_t,A_t,S_{t+1})$, and transitions to a successor state
$S_{t+1}\sim P(\cdot\mid S_t,A_t)$.
The discounted return from time $t$ is defined as $G_t:=\sum_{k=0}^{\infty}\gamma^k R_{t+k+1}$. The state-value and action-value functions associated with a policy $\pi$ are respectively defined as $v_\pi(s):=\mathbb{E}_\pi\!\left[G_t\mid S_t=s\right]$, $q_\pi(s,a):=\mathbb{E}_\pi\!\left[G_t\mid S_t=s,A_t=a\right]$.
The objective of RL is to find a policy that maximizes the expected discounted return.

Q-learning \citep{watkins1989learning, watkins1992q} is a classical off-policy TD control algorithm for estimating the optimal action-value function.
In the tabular setting, an estimate $Q_t(s,a)$ is maintained explicitly for each state-action pair.
Given a transition $(S_t,A_t,R_{t+1},S_{t+1})$, Q-learning updates the value of the visited state-action pair according to
\begin{equation}
Q_{t+1}(S_t,A_t)
=
(1-\alpha)Q_t(S_t,A_t)
+
\alpha
\left[
R_{t+1}
+
\gamma\max_{a'\in\mathcal{A}}Q_t(S_{t+1},a')
\right],
\end{equation}
where $\alpha\in(0,1]$ is the learning rate.
All other state-action values remain unchanged.
To handle large state-action spaces, we consider linear function approximation \citep{tsitsiklis1997analysis, melo2008analysis}. 
Given a feature map $\phi:\mathcal{S}\times\mathcal{A}\to\mathbb{R}^d$ and a parameter vector $\mathbf{w}_t\in\mathbb{R}^d$ at time $t$, the action-value function is approximated by
$q_{\mathbf{w}_t}(s,a):=\mathbf{w}_t^\top\phi(s,a)$.
Given a transition $(S_t,A_t,R_{t+1},S_{t+1})$, the semi-gradient Q-learning update is
\begin{equation}
\mathbf{w}_{t+1}
=
\mathbf{w}_t
+
\alpha
\left(
R_{t+1}
+
\gamma\max_{a'\in\mathcal{A}}
\mathbf{w}_t^\top\phi(S_{t+1},a')
-
\mathbf{w}_t^\top\phi(S_t,A_t)
\right)
\phi(S_t,A_t).
\end{equation}
The tabular setting is recovered as a special case when $\mathcal{S}$ and $\mathcal{A}$ are finite, $d=|\mathcal{S}||\mathcal{A}|$, and each state-action pair $(s,a)$ is assigned a distinct one-hot feature vector ${\phi}(s,a)=\mathbf{e}_{(s,a)}$. 
In this case, each coordinate of $\mathbf{w}_t$ corresponds to a single state-action value $Q_t(s,a)$, and the above semi-gradient update reduces exactly to the standard tabular Q-learning update.

\section{Transformer can implement  In-context soft Q‑learning}
\label{sec:construction}
We consider a Boltzmann-smoothed Q-learning update that replaces the hard maximum over successor actions with a Boltzmann-weighted average. Given the current parameter  $\mathbf{w}$, trajectories $(s_0,a_0,r_1,\dots,s_{n-1},a_{n-1},r_n,s_n)$ can be obtained by interacting with environment under its behavioral policy, where behavioral policy is a behavior policy $\pi_b$, which may differ from greedy target policy.
Let $\mathbf{z}$ denote a context consisting of an $n$-transition window, parameter $\mathbf{w}$, and the features of all candidate successor actions. 
For $i=0,\ldots,n-1$, define $\phi_i=\phi(s_i,a_i)$, $\phi^+_{i,a}=\phi(s_{i+1},a)$, $a\in\mathcal{A}$ and
\begin{equation}\label{eq:piw}
\pi^\tau_{i,a}=\frac{\exp(\mathbf{w}^\top\phi^+_{i,a}/\tau)}{\sum_{b\in \mathcal{A}}\exp(\mathbf{w}^\top\phi^+_{i,b}/\tau)},\qquad
\bar\phi_i^{+,\tau}=\sum_{a\in \mathcal{A}}\pi^\tau_{i,a}\phi^+_{i,a}.
\end{equation}
The corresponding batched semi-gradient update is
$\mathbf{w}^{(\tau)}_{\text{soft}}=\mathbf{w}+\frac{\alpha}{n}\sum_{i=0}^{n-1}\phi_i(r_{i+1}+\gamma \mathbf{w}^\top\bar\phi_i^{+,\tau}-\mathbf{w}^\top\phi_i)$.
Here, $\mathbf{w}^{\top}\bar{{\phi}}_{i}^{+,\tau}$
is the Boltzmann-weighted average of the successor action values, which replaces the hard maximum
$\max_{a\in\mathcal{A}}\mathbf{w}^{\top}{\phi}_{i,a}^{+}$
in Q-learning. This soft backup recovers the Q-learning backup in the zero-temperature limit $\tau\downarrow0$; see Lemma~\ref{lem:Zero}.

We now give an explicit Transformer construction and show that its forward pass implements the soft Q-learning update. 
Each token is partitioned into six blocks,
$[\Phi;G;R;\Psi;U;W]$ of dimensions $[d,d,1,d,d,d+1]$, respectively. The prompt consists of three types of tokens: transition tokens, which encode the observed transitions; candidate-action tokens, which encode the features of candidate actions at the successor states; and a query token, which carries the current parameter vector. Letting $\widetilde{\mathbf{w}}=[1;\mathbf{w}]$, for each transition $i=0,\ldots,n-1$ and candidate action $a\in \mathcal{A}$, we define
\begin{equation}
\begin{split}
 t_i=[\phi_i;0;r_{i+1};0;\mathbf{w};0],\quad
 c_{i,a}=[0;0;0;\phi^+_{i,a};0;0],\quad
 p=[0;0;0;0;0;\tilde{\mathbf{w}}]
\end{split}\label{eq:prompt}
\end{equation}
Thus, the token dimension is $D=5d+2$, and the prompt contains
$N=n(m+1)+1$ tokens.
For each transition $i$, we group the candidate-action tokens as $\mathbf C_i=[c_{i,a^1},c_{i,a^2},\ldots,c_{i,a^m}]$, where $m:=|\mathcal{A}|<\infty$, so that $\mathbf C_i$ contains exactly the candidate-action tokens associated with all actions available at the successor state $s_{i+1}$.
The resulting input prompt is
\begin{align}
\mathbf H^{(0)}
=
\big[
t_0,\mathbf C_0,\,
t_1,\mathbf C_1,\,
\dots,\,
t_{n-1},\mathbf C_{n-1},\,
p
\big].
\label{eq:Hzero}
\end{align}

Our Transformer consists of two successive attention layers. 
The first layer uses single-head softmax attention to compute a 
Boltzmann-weighted aggregation over candidate successor actions, thereby 
implementing a smooth approximation to the maximization operator in Q-learning. 
The second layer then uses linear attention to realize the corresponding update. Specifically, the first layer is defined as
\begin{align}\label{eq:promH0}
    \mathbf{H}^{(1)} = \mathbf{H}^{(0)} + \bigl(\mathbf{V}_1 \mathbf{H}^{(0)}\bigr) \mathrm{softmax}\Bigl((  \mathbf{H}^{(0)})^\top \mathbf{P}_1 \mathbf{H}^{(0)} + \mathbf{M}\Bigr).
\end{align}
We use $\star$ to denote an arbitrary matrix block of the corresponding 
dimension. These blocks need not vanish and may take arbitrary values, 
since they do not affect the entries relevant to the construction under 
the prescribed attention mask and token structure. We choose
\begin{align}\label{eq:layer1}
\mathbf{P}_1=
\begin{bmatrix}
\star & \star & \star & \star & \star & \star\\
\star & \star & \star & \star & \star & \star\\
\star & \star & \star & \star & \star & \star\\
0_{d\times d} &
\star &
0_{d\times 1} &
\star &
\mathbf{A}_\theta &
\star\\
\star & \star & \star & \star & \star & \star\\
\star & \star & \star & \star & \star & \star
\end{bmatrix},\quad
\mathbf{V}_1 =
\begin{bmatrix}
\star&\star&\star&\mathbf{D}_\Phi&\star&0\\
\star&\star&\star&\mathbf{L}&\star&0\\
\star&\star&\star&\mathbf{D}_R&\star&0\\
\star&\star&\star&0&\star&0\\
\star&\star&\star&\mathbf{D}_U&\star&0\\
\star&\star&\star&0&\star&\mathbf{E}_\rho
\end{bmatrix}
\end{align}
where $\mathbf{A}_\theta=\tau^{-1}I_d$, $\mathbf{D}_\Phi$, $\mathbf{D}_U\in \mathbb{R}^{d\times d}$, $\mathbf{L}\in\mathrm{GL}(d)$, $\mathbf{D}_R\in \mathbb{R}^{1\times d}$, and $\mathbf{E}_\rho=\mathbf{Q}-I_{d+1}$. Here $\mathbf{Q}= 
 \begin{bmatrix}
  \rho&b^\top\\0_{d\times1}&I_d
 \end{bmatrix}$, $b\in \mathbb{R}^d$, $\rho\neq0$, and $\mathrm{GL}(d)$ denotes the set of invertible $d\times d$ matrices.

We define the attention mask $\mathbf M\in\{0,-\infty\}^{N\times N}$, 
where rows index keys and columns index queries. 
The mask is chosen such that each transition token $t_i$ attends only to its 
corresponding candidate-action tokens $\{c_{i,a}\}_{a\in\mathcal A}$, while 
each $c_{i,a}$ and the query token $p$ attend only to $p$. All other pairs are masked out, ensuring that each $t_i$ aggregates only its corresponding successor action candidates.
 
The second layer employs single-head linear attention:
\begin{align}\label{eq:promH1}
    \mathbf{H}^{(2)} = \mathbf{H}^{(1)} + \frac1n(\mathbf{V}_2\mathbf{H}^{(1)})\;\big((\mathbf{H}^{(1)})^\top \mathbf{P}_2\mathbf{H}^{(1)}\big).
\end{align}
We choose
\begin{align}\label{eq:layer2}
\mathbf{P}_2=
\begin{bmatrix}
\star&\star&\star&\star&\star&\mathbf{B}_-\\
\star&\star&\star&\star&\star&\mathbf{B}_+\\
\star&\star&\star&\star&\star&\mathbf{B}_r\\
\star&\star&\star&\star&\star&\star\\
\star&\star&\star&\star&\star&\mathbf{B}_U\\
\star&\star&\star&\star&\star&\star
\end{bmatrix}, \mathbf{V}_2 =
\begin{bmatrix}
\star_{(4d+2)\times d} &\star_{(4d+2)\times (d)}& \star_{(4d+2)\times (3d+2)}  \\
\mathbf{C} &\mathbf{C}_G& 0_{d\times (3d+2)}  
\end{bmatrix}.
\end{align}
where $\mathbf{B}_-=-c\mathbf{J}\mathbf{Q}^{-1}$, $\mathbf{B}_+ =c\mathbf{L}^{-\top}[-\mathbf{D}_R^{\top},\gamma I_d+\mathbf{D}_\Phi^\top-\mathbf{D}_U^\top\Omega]\mathbf{Q}^{-1} $, $\mathbf{B}_r = ce_0^\top \mathbf{Q}^{-1}$, $\mathbf{B}_U=c\Omega \mathbf{J}\mathbf{Q}^{-1}$, $\mathbf{C} = ({\alpha}/c) I_d$, and $\mathbf{C}_G= (-{\alpha}/c) \mathbf{D}_\Phi \mathbf{L}^{-1}$. Here, $\Omega^\top=-\Omega$ is arbitrary and may, in particular, be nonzero, $\mathbf{J}=[0,I_d]$, $e_0^\top=[1,0_{1\times d}]$, and $c\neq0$.

\begin{theorem}[Forward pass as soft Q-learning]
\label{Th:Qlearning1}
Consider two-layer Transformer defined above, with masked softmax attention \eqref{eq:promH0} parameterized by $\mathbf{P}_1,\mathbf{V}_1$ in the first layer, and linear attention \eqref{eq:promH1} parameterized by $\mathbf{P}_2,\mathbf{V}_2$ in the second layer.
For the input prompt $\mathbf{H}^{(0)}$ in \eqref{eq:Hzero}, the following hold.

\textnormal{(i)} The Transformer exactly implements the soft Q-learning update:
\[
\mathbf{w}_{\mathrm{TF}}^{(\tau)}:=\operatorname{TF}^{(\tau)}(\mathbf{H}^{(0)};\mathbf{P}_1,\mathbf{V}_1,\mathbf{P}_2,\mathbf{V}_2)
=
\mathbf{w}
+
\frac{\alpha}{n}\sum_{i=0}^{n-1}\phi_i
\left[
r_{i+1}
+\gamma\mathbf{w}^{\top}\bar{\phi}_i^{+,\tau}
-\mathbf{w}^{\top}\phi_i
\right],
\]
where
$\bar{\phi}_i^{+,\tau}
=
\sum_{a\in\mathcal A}\pi_{i,a}^{\tau}\phi_{i,a}^{+}$,
$\pi_{i,a}^{\tau}
=
\exp(\mathbf{w}^{\top}\phi_{i,a}^{+}/\tau)/\sum_{b\in\mathcal A}\exp(\mathbf{w}^{\top}\phi_{i,b}^{+}/\tau)$.

\textnormal{(ii)} For each $i$, define
$q_{i,a}:=\mathbf w^\top \phi_{i,a}^{+}$,
$q_i^\star:=\max_{a\in\mathcal A} q_{i,a}$, and
$\mathcal A_i^\star:=\arg\max_{a\in\mathcal A} q_{i,a}$.
Let $m:=|\mathcal A|<\infty$ and $K_i:=|\mathcal A_i^\star|$. Suppose that $\|\phi(s,a)\|_2\le B_\phi$ and
$|q_{\mathbf w}(s,a)|\le B_q$ for all state-action pairs. 
For every $i$ such that $\mathcal A_i^\star\neq\mathcal A$, define the action gap $\Delta_i := q_i^\star-\max_{a\notin\mathcal A_i^\star}q_{i,a}
>0$. Then $\lim_{\tau\downarrow0}\mathbf w_{\mathrm{TF}}^{(\tau)}=\mathbf w_Q$ and
\[
\left\|
\mathbf w_Q-\mathbf w_{\mathrm{TF}}^{(\tau)}
\right\|_2
\le
2\alpha\gamma B_q B_\phi 
\sum_{i:\,\mathcal A_i^\star\neq\mathcal A}
\frac{m-K_i}{nK_i}
e^{-\Delta_i/\tau}, 
\]
where $\mathbf{w}_Q
=
\mathbf{w}
+
\frac{\alpha}{n}\sum_{i=0}^{n-1}\phi_i
\left[
r_{i+1}
+\gamma\max_{a\in\mathcal A}
\mathbf{w}^{\top}\phi_{i,a}^{+}
-\mathbf{w}^{\top}\phi_i
\right]$ is window Q-learning update.
\end{theorem}

The proof of Theorem~\ref{Th:Qlearning1} is provided in Appendix~\ref{app:two_proof}, together with the corresponding result for the tabular setting. 
Moreover, if each maximizer is unique and
$\Delta_{\min}:=\min_i\Delta_i>0$, then
$\|
\mathbf w_Q-\mathbf w_{\mathrm{TF}}^{(\tau)}
\|_2
\le
2\alpha\gamma B_q B_\phi(m-1)
e^{-\Delta_{\min}/\tau}$, $m>1$,
showing that the Transformer update converges exponentially fast to the standard window Q-learning update as $\tau\downarrow0$. 
The theorem thus establishes that a two-layer Transformer can exactly implement soft Q-learning in its forward pass and recover standard Q-learning in the zero-temperature limit. 
We call this in-context Q-learning.

\section{IC-Q-Learning algorithm and training dynamics analysis}\label{IC-Q-Learning}
\subsection{IC-Q-Learning Algorithm and Training Dynamics}
\label{sec:training-dynamics}

We now design an ICRL training procedure for the proposed soft Q-learning Transformer, inspired by prior ICRL training frameworks \citep{laskin2023incontext,liang2026transformers}, as summarized in Algorithm \ref{alg:fresh}.
The training objective follows the teacher-mimicking principle:
for each sampled context, the Transformer is trained to predict the parameter
obtained after one soft Q-learning update. In contrast to directly supervising the internal attention matrices, the loss is imposed only on the final parameter prediction. Hence, the Transformer is not explicitly provided with the analytical parameter construction derived in the preceding section.

\begin{algorithm}[t]
\caption{IC-Q-Learning Transformer Training with Independent-Context Minibatches}
\label{alg:fresh}
\begin{algorithmic}[1]
\REQUIRE Transformer parameters
$\boldsymbol{\theta}^{(0)}
=
(\mathbf{P}_1^{(0)},\mathbf{V}_1^{(0)},\mathbf{P}_2^{(0)},\mathbf{V}_2^{(0)})$,
trajectory window length $n$,
number of optimizer updates $U$, minibatch size $B$, 
$\epsilon$-greedy parameter $\epsilon$,
Transformer learning rate $\eta$,
Q-learning step size $\alpha$,
and temperature $\tau$.
\FOR{$k=0,\ldots,U-1$}
  \FOR{$b=1,\ldots,B$}
    \STATE Independently sample a new MDP $\M_b^{(k)}$, feature map $\phi_b^{(k)}$, $\mathbf{w}_b^{(k)}$, and initial state $s_b$.
    \STATE Estimate
    $q_{\mathbf{w}_b}^{(k,b)}(s,a)=(\mathbf{w}_b)^\top\phi_b^{(k)}(s,a)$, $\forall(s,a)\in\mathcal S\times\mathcal A$.
\STATE Sample an $n$-step trajectory $(s_{b,0},a_{b,0},r_{b,1},s_{b,1},a_{b,1},\ldots,r_{b,n},s_{b,n})$ from $\M_b^{(k)}$ according to the $\epsilon$-greedy behavior policy derived from $q_{\mathbf{w}}^{(k,b)}$.
\STATE Formulate $\mathbf{H}_b$ as in \eqref{eq:Hzero} using the sampled trajectory, the candidate features $\{\phi_b^{(k)}(s_{b,i+1},a):a\in\mathcal A\}$, and the current parameter $\mathbf{w}_b$.
\STATE Compute $\widetilde{\mathbf{H}}_b=\mathbf{H}_b+(\mathbf{V}_1\mathbf{H}_b)\operatorname{softmax}\left((\mathbf{H}_b)^\top\mathbf{P}_1\mathbf{H}_b+\mathbf{M}\right)$.
\STATE Compute $\mathbf{w}_{\mathrm{TF},b} =\left[ \widetilde{\mathbf{H}}_b+ \frac{1}{n}(\mathbf{V}_2\widetilde{\mathbf{H}}_b)\left((\widetilde{\mathbf{H}}_b)^\top\mathbf{P}_2\widetilde{\mathbf{H}}_b\right)\right]_{-d:,-1}$.
\STATE Compute the teacher update $\mathbf{w}^{(\tau)}_{\text{soft},b}$ according to soft Q-learning.
  \ENDFOR
  \STATE Compute $\widehat{\mathcal L}_B^{(k)}=\frac{1}{2B}\sum_{b=1}^{B}\|\mathbf{w}_{\mathrm{TF},b} - \mathbf{w}^{(\tau)}_{\text{soft},b} \|_2^2$.
  \STATE Update
        $\boldsymbol{\theta}^{(k+1)} \leftarrow \boldsymbol{\theta}^{(k)} - \eta \nabla_{\theta} \widehat{\mathcal L}_B^{(k)}$ 
\ENDFOR
\RETURN trained Transformer parameters $\boldsymbol{\theta}$.
\end{algorithmic}
\end{algorithm}

Fix integers $d\ge1$, $m\ge2$, $n\ge 1$ and constants $\alpha>0$, $0<\gamma<1$, $\tau>0$. Let $\mathbf{z}:=\bigl(\mathbf{w},\{\phi_i,r_{i+1},\phi^+_{i,j}:0\le i<n,\ j\in[m]\}\bigr)\sim\mathcal D_n$. Let
$\mathbf{w}^{(\tau)}_{\text{soft}} (\mathbf{z})$ denote the corresponding soft Q-learning target.
We define the population teacher-mimicking objective as
\begin{equation}
    \mathcal{L}(\boldsymbol{\theta}) :=
    \mathbb{E}_{\mathbf{z}}
    \left[ \frac{1}{2} \| \operatorname{TF}_{\theta} (\mathbf{H}(\mathbf{z})) - \mathbf{w}^{(\tau)}_{\text{soft}}(\mathbf{z}) \|_2^2
    \right].
    \label{eq:population-q-loss}
\end{equation}
where the expectation is taken jointly over the
training-MDP distribution, random feature maps, behavior trajectories, and
independently sampled input Q-function parameters.

We optimize the population objective in~\eqref{eq:population-q-loss} using
independently sampled training contexts. At each optimization step, we draw
a minibatch $\mathbf{z}_1,\ldots,\mathbf{z}_B \overset{\mathrm{i.i.d.}}{\sim}\mathcal D_n$,
where each context is generated from an independently sampled MDP and feature
map. The empirical loss is
$\widehat{\mathcal{L}}_B(\boldsymbol{\theta}) := \frac{1}{2B}\sum_{b=1}^{B}\|\operatorname{TF}_{\theta}(\mathbf H(\mathbf z_b))-\mathbf w^{(\tau)}_{\mathrm{soft}}(\mathbf z_b)\|_2^2$. 
Here, $B$ controls the amount of averaging across independently sampled
training contexts within each parameter update, and hence the stochasticity of
the resulting gradient estimate. Further details are provided in
Proposition \ref{prop:minibatch-gradient}.

Write $\Tset=(\Phi,R,U)$ and
$\Iset=(\Phi,G,R,U)$. Let $W_{\mathrm{out}}$ denote the last $d$
coordinates of $W$. Let $\boldsymbol{\theta}_{\mathrm{out}}$ collect all entries of
$\boldsymbol{\theta}=(\mathbf{P}_1,\mathbf{V}_1,\mathbf{P}_2,\mathbf{V}_2)$ outside the following $16$ blocks: $
 \mathbf{P}_{1,\Psi\Tset}$, 
 $\mathbf{V}_{1,\Iset\Psi}$, $\mathbf{V}_{1,WW}$, 
 $\mathbf{P}_{2,\Iset W}$, $\mathbf{V}_{2,W_{\mathrm{out}}\Iset}$.
Under the full-matrix population gradient flow
$\dot{\boldsymbol{\theta}}=-\nabla\mathcal L(\boldsymbol{\theta})$, if
$\boldsymbol{\theta}_{\mathrm{out}}(0)=0$, then
$\boldsymbol{\theta}_{\mathrm{out}}(t)=0$ for all $t\ge0$ (Proposition~\ref{prop:invariant-groups}).
Consequently, the training dynamics are entirely determined by
the retained blocks. We refer to these parameter blocks as valid parameter blocks. However, the resulting parameterization remains highly redundant, with additional degrees of freedom associated with $\mathbf{D}_\Phi$, $\mathbf{D}_U$, $\rho$, $b$, $c$, $\mathbf{L}$, and $\Omega$. To isolate the essential parameter blocks, we define ${\boldsymbol{\theta}}_{\mathrm{sup}}:=(\mathbf{A}_{\boldsymbol{\theta}},\mathbf{L}, \mathbf{B}_-, \mathbf{B}_+, \mathbf{B}_r,\mathbf{C})
\in\Theta_{\mathrm{sup}}$. For any $\mathbf{L} \in\mathrm{GL}(d)$ and $c\neq0$, let
\begin{equation}\label{eq:manifold}
\bar{\boldsymbol{\theta}}_\tau(\mathbf{L},c)
=
\left(
\tau^{-1}I_d,\,
\mathbf{L},\,
[0,-cI_d],\,
[0,c\gamma\mathbf{L}^{-\top}],\,
[c,0],\,
\frac{\alpha}{c}I_d
\right).
\end{equation}
These six blocks form an inclusion-minimal sufficient support within the specified architecture: they realize the target update in
Theorem~\ref{Th:Qlearning1} on all admissible prompts, whereas no proper subset does so when every other block is fixed to zero  (Proposition~\ref{prop:six-block-minimal-support}). We call these six parameter blocks the support parameter blocks.  

\subsection{Empirical Analysis}\label{ex:main}
Each MDP has $n_S = 9$ states and $n_A =m= 4$ actions, with discount $\gamma = 0.5$.  For each MDP, the initial-state distribution $P_0$ and transition kernel $P(\cdot \mid s, a)$ are drawn independently from $\text{Dirichlet}(1, \dots, 1)$ (uniform over the probability simplex on $n_S$ states), and the rewards $r(a, s') \sim \text{Unif}(-1, 1)$ are sampled i.i.d.\ per $(a, s')$.  
For each MDP, the state-action feature map $\phi\colon \mathcal{S}\times\mathcal{A}\to \mathbb{R}^d$ ($d = n_S n_A = 36$) and the initial weight $\mathbf{w}\in \mathbb{R}^d$ have entries drawn i.i.d.\ from $\text{Unif}(-1,1)$. The behavior policy is $\epsilon$-greedy ($\epsilon=0.1$) with $Q(s,a) = \phi(s,a)^\top \mathbf{w}$, and the soft window Q-learning target $\mathbf{w}^{(\tau)}_\text{soft}$ is computed from the trajectory window with step size $\alpha=0.2$. 

The Transformer uses embedding dimension $D = 5d + 2 = 182$. The prompt supplies $\mathbf{w}$, current features $\phi_i$, rewards $r_{i+1}$, and
all next-state candidate features $\phi^+_{i,a}$.
Tabular tests replace features by $\phi(s,a)=e_{(s-1)m+a}$ and are explicitly
separated from dense-feature tests. Training is performed over 46.40 million tasks per model and five paired seeds. Adam uses cosine learning-rate decay
from the first update; the default temperature, batch size and context length are
$\tau=0.2$, $B=32$ and $n=20$.  By Proposition~\ref{prop:invariant-groups}, the sixteen valid parameter blocks are initialized using Xavier normal initialization with gain \(0.1\). All remaining entries of the two-layer attention construction are initialized to zero, as they remain fixed during training. We conduct the following experiments, with full details provided in Appendix~\ref{ex:Valid}.

Experiments show that parameter blocks other than the sixteen valid parameter blocks stay at zero as training progresse, confirming Proposition \ref{prop:invariant-groups}. We evaluate the learned update using three complementary metrics: Soft error measures how accurately the Transformer realizes the soft Q update, Hard error measures its discrepancy from standard Q-learning, and Greedy agreement measures how often it selects the same greedy action as Q-learning. The learned Transformer closely approximates the soft Q-learning update and achieves nearly 100\% greedy agreement.  

Figure \ref{fig:matrices} presents the final learned matrices $\mathbf{P}_1$, $\mathbf{V}_1$, $\mathbf{P}_2$, and $\mathbf{V}_2$ obtained after training. Remarkably, the learned matrices recover the parameter structure predicted in Theorem \ref{Th:Qlearning1}. In particular, the diagonal entries of $\mathbf{A}_\theta$ in $\mathbf{P}_1$ approach $1/\tau$, indicating that the Transformer has learned the internal action-selection computation required by the soft Q-learning update. This provides direct evidence of parameter emergence. To quantitatively assess parameter recovery, we examine several construction relations.
The learned parameters closely approximate
 analytical relations, and the task-exposure curves show improved
recovery as training progresses.  See Appendix \ref{ex:Valid} for details of the heatmaps.

\begin{figure}[H]
\centering
\begin{minipage}{0.49\textwidth}
\includegraphics[width=\linewidth,keepaspectratio]{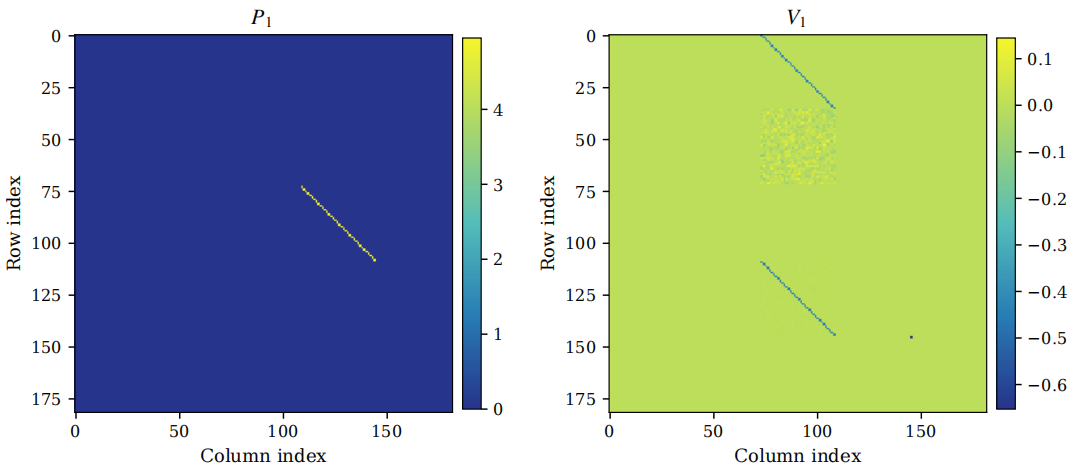}
\end{minipage}
\hskip 0pt 
\begin{minipage}{0.49\textwidth}
\includegraphics[width=\linewidth,keepaspectratio]{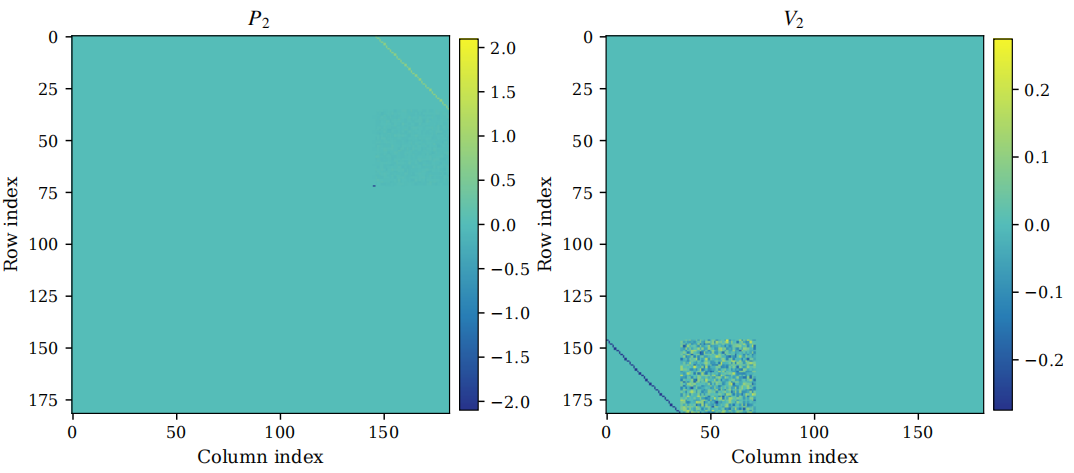}
\end{minipage}
\caption{Final learned $\mathbf{P}_1$, $\mathbf{V}_1$, $\mathbf{P}_2$ and $\mathbf{V}_2$ matrices after training at $\tau=0.2$, $B=32$, and $n=20$. the non-zero structure matches the theoretical constructions $(\mathbf{P}_1, \mathbf{V}_1, \mathbf{P}_2, \mathbf{V}_2)$ in \eqref{eq:layer1} and \eqref{eq:layer2}.}\label{fig:matrices}
\end{figure}

The left panel of Figure \ref{fig:cos-control1} deploys frozen dense-feature-trained Transformers on unseen tabular MDPs and measures exact policy returns under a common interaction budget. 
Comparisons include hard and soft window Q-learning, a model-informed
Oracle, and uniform Random.
The Transformer achieves returns close to those of the soft window teacher and a higher sample-mean return than the hard window-Q baseline using the same fixed step size.

\begin{figure}[H]
\centering
\includegraphics[width=.43\linewidth]{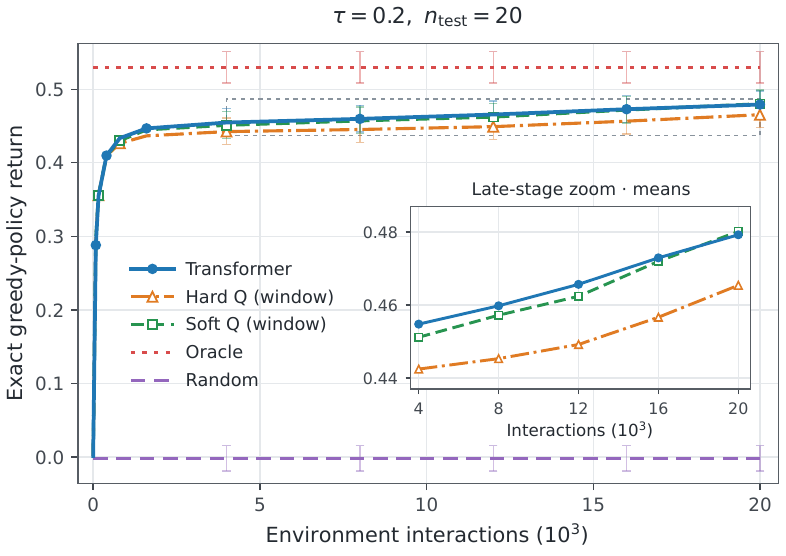}\hfill
\includegraphics[width=.46\linewidth]{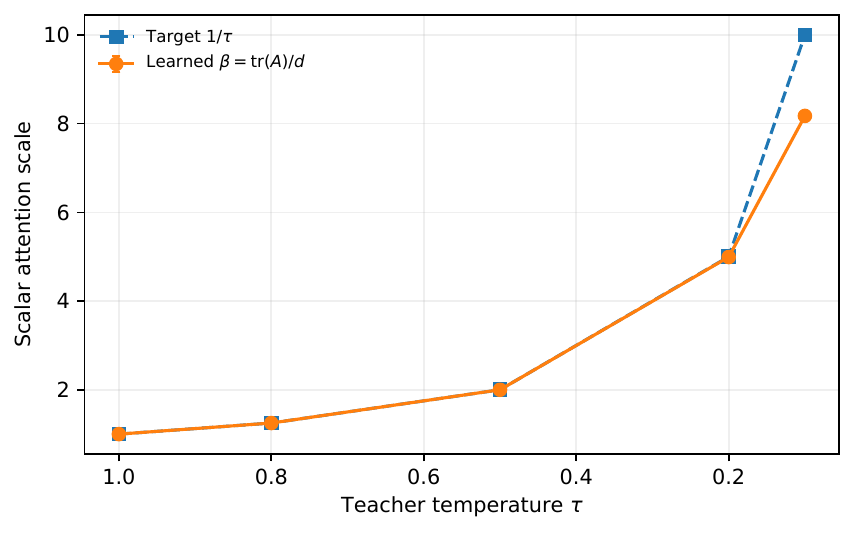}
\caption{Left: Performance of  Transformers on unseen tabular MDPs (sixteen valid parameter blocks). Right: The learned scalar projection $\beta_{\mathrm{att}}$ and its target $1/\tau$ (six support parameter blocks).}\label{fig:cos-control1}
\end{figure}

\paragraph{The effect of the action-selection operator.} To separately study how temperature affects the emergence of parameters, we analyze the six-block parameterization, which retains the complete soft Q-learning update mechanism and is inclusion-minimal (Proposition~\ref{prop:six-block-minimal-support}). All six blocks remain jointly trainable under output-only supervision. Training is performed over 20.48 million tasks per model. At a given temperature $\tau$, the trained Transformer is shown to approximate the soft Q-learning update and produce emergent parameters, verifying Proposition \ref{prop:six-block-minimal-support}. Decreasing $\tau$ brings the soft Q update closer to standard Q-learning, but increases the approximation error between the Transformer and the soft update. With a fixed training budget, the recovery error of the predicted parameter structure increases gradually as \(\tau\) decreases. The most pronounced discrepancy appears in the scalar attention scale $\beta_{\mathrm{att}} := \mathrm{tr}(\mathbf{A}_\theta)/d$. As shown in the right panel of Figure \ref{fig:cos-control1}, the learned scale falls below its theoretical target $1/\tau$ at low temperatures. This implies that smaller $\tau$ requires greater task exposure for accurate parameter recovery. We further examine the effects of batch size and training-context length. See Appendix~\ref{app:ApQ} for details.

\subsection{Theoretical Analysis}\label{sec:convergence}
In this section, we provide a theoretical account of the above empirical findings by explaining how parameter emergence arises from the pretraining algorithm. Our primary focus is to investigate the effect of the action-selection mechanism on convergence. To this end, we restrict optimization to the support parameter blocks, reducing parameter redundancy and isolating the Transformer’s Q-learning mechanism. Define the optimal parameter manifold as $\mathcal M_\tau
:=
\{\bar{\boldsymbol{\theta}}_\tau(\mathbf{L} ,c):
\mathbf{L} \in\mathrm{GL}(d),\ c\neq0\}$, we study whether deterministic gradient descent ${\boldsymbol{\theta}}_{k+1}={\boldsymbol{\theta}}_k-\eta\nabla \mathcal{L}({\boldsymbol{\theta}}_k)$ converges to $\mathcal M_\tau$ and characterize how its convergence rate and iteration complexity depend on $\tau$.

Let $\Coef$ be a vector space with $\dim\Coef= 4d^2+3d$ and
$\zeta:=(E_C,e_q,e_-,E_-,e_+,E_+,E_A)\in\Coef$, where the four $E$-symbols are $d\times d$ matrices and the three lowercase $e$-symbols are vectors in $\R^d$. Let the teacher-weighted covariance be $\Sigma_i^\tau=\sum_{a=1}^m\pi_{i,a}^\tau
(\phi^+_{i,a}-\bar\phi_i^{+,\tau})(\phi^+_{i,a}-\bar\phi_i^{+,\tau})^\top
\in\R^{d\times d}$ and
\begin{align}
\Obs(\mathbf{z})[\zeta]=&\frac1n\sum_{i=0}^{n-1}\Bigl\{
E_C\phi_i\delta_i^\tau+\phi_i\bigl[
\phi_i^\top(e_-+E_-\mathbf{w})+(\bar\phi_i^{+,\tau})^\top(e_++E_+\mathbf{w})\nonumber\\
&\qquad+r_{i+1}e_q^\top \mathbf{w}+\mathbf{w}^\top\Sigma_i^\tau E_A\mathbf{w}\bigr]\Bigr\}.
\label{eq:observable}
\end{align}
where $\delta_i^\tau(\mathbf{z}) =r_{i+1}-\mathbf{w}^\top\phi_i+\gamma \mathbf{w}^\top\bar\phi_i^{+,\tau}(\mathbf{z})$. In these coordinates, the adjoint $\Obs(\mathbf{z})^*:\R^d\to\Coef$ is represented by the transpose of $\Obs(\mathbf{z})$. We define the $\dim\Coef\times\dim\Coef$ data Gram matrix as $\mathsf H_n(\tau)=\E[\Obs(\mathbf{z})^*\Obs(\mathbf{z})]$, and denote its smallest eigenvalue by $\chi=\lambda_{\min}(\mathsf H_n(\tau))$. 

We now state the richness assumptions on the training-MDP distribution

\begin{assumption}\label{assume}
Let $\Pi$ denote the class of behavior policies used to generate training trajectories. There exist positive constants $B_w,B_\phi,B_r$ such that, for every $\pi \in \Pi$, $\mathcal D_n$ is independent of ${\boldsymbol{\theta}}$ and 

(i) Boundedness. $\|\mathbf{w}\|_2\le B_w$, $\|\phi_i\|_2,\|\phi^+_{i,a}\|_2\le B_\phi$, $|r_{i+1}|\le B_r$, $\mathcal D_n\text{-a.s.}$

(ii) Joint observable richness. For the fixed $(n,\tau)$, $\mathsf H_n(\tau)\succ0$, meaning
$\langle\zeta,\mathsf H_n(\tau)\zeta\rangle>0$ for every $\zeta\ne0$. Equivalently,
$\Obs(\mathbf{z})[\zeta]=0$, $\mathcal D_n\text{-a.s.}\Longrightarrow \zeta=0$.
\end{assumption}
Assumption~\ref{assume} holds for a broad class of task populations with a nonzero absolutely continuous component and sufficient trajectory support, including the distribution used in our experiments. Explicit sufficient conditions and their verification are given in Appendix~\ref{app:verification}.

Fix a nonempty relatively compact open set $\boldsymbol{\Omega}\Subset\GL(d)\times (\mathbb R\setminus\{0\})$.
Its closure $\overline{\boldsymbol{\Omega}}$ is compact and remains inside
$\GL(d)\times(\mathbb R\setminus\{0\})$. Define its images
$\U=\{\bar{\boldsymbol{\theta}}_\tau(\mathbf{L},c):(\mathbf{L},c)\in\boldsymbol{\Omega}\}$, 
$\Ub=\{\bar{\boldsymbol{\theta}}_\tau(\mathbf{L},c):(\mathbf{L},c)\in\overline{\boldsymbol{\Omega}}\}$.
The set $\Ub$ is compact at each fixed $\tau$.

At a point $\bar{\boldsymbol{\theta}}\in\M_\tau$, the tangent space $T_{\bar{\boldsymbol{\theta}}}\M_\tau$ consists of derivatives at zero of smooth curves lying in $\M_\tau$ and passing through $\bar{\boldsymbol{\theta}}$. The normal space is its orthogonal complement, 
$N_{\bar{\boldsymbol{\theta}}}\M_\tau=(T_{\bar{\boldsymbol{\theta}}}\M_\tau)^\perp$, define the tubular neighborhood of $\U$ as
\begin{equation}
\mathcal T:=\mathcal T_r(\U)=\{\bar{\boldsymbol{\theta}}+\Delta:\bar{\boldsymbol{\theta}}\in\U,
\ \Delta\in N_{\bar{\boldsymbol{\theta}}}\M_\tau,\ \|\Delta\|<r\}.
\label{eq:tube}
\end{equation}
A perturbation $\mathbf{U}\in\Theta_{\mathrm{sup}}$ has blocks denoted by $\Delta A,\ldots,\Delta C$ in the same order as \eqref{eq:manifold}.
Let $\kappa$ denote the minimum local curvature normal to $\mathcal{M}_\tau$, defined by
\begin{equation}\label{eq:kappa}
    \kappa
:=\inf_{\substack{\bar{\boldsymbol{\theta}}\in {\mathcal U}_\tau,\mathbf{U}\perp T_{\bar{\boldsymbol{\theta}}}\mathcal M_\tau,\|\mathbf{U} \|=1}} \left\langle \mathbf{U},\nabla^2\mathcal{L}(\bar{\boldsymbol{\theta}})\mathbf{U}\right\rangle> 0.
\end{equation}
Let $T_\eta(\boldsymbol{\theta}) = \boldsymbol{\theta} - \eta \nabla \mathcal{L}(\boldsymbol{\theta})$ and $d(\boldsymbol{\theta}):=\operatorname{dist}(\boldsymbol{\theta},\mathcal M_\tau)$, define
$$
K_{\text{GD}}(\varepsilon;\boldsymbol{\theta}_0) := \min\big\{ K \in \mathbb{N}_0 : d\big(T_\eta^K(\boldsymbol{\theta}_0)\big) \le \varepsilon d(\boldsymbol{\theta}_0) \big\}.
$$

\begin{theorem}\label{thm:main}
Under Assumption~\ref{assume}, there exist  a neighborhood $\mathcal{T}$, a constant $M_J>0$ such that, for any $\varepsilon\in (0, 1/2]$, $0<\eta\le\frac{1}{2M_J^2}$, and ${\boldsymbol{\theta}}_0\in \mathcal{A}_\tau :=\{{\boldsymbol{\theta}}\in\mathcal T:d({\boldsymbol{\theta}})>0,\frac{4M_J}{\sqrt{\kappa}} d(\boldsymbol{\theta})<\operatorname{dist}(\boldsymbol{\theta},\mathcal{T}^c)\}$, 

(i) The gradient descent
${\boldsymbol{\theta}}_{k+1}={\boldsymbol{\theta}}_k-\eta\nabla \mathcal{L}({\boldsymbol{\theta}}_k)$,
converges to some ${\boldsymbol{\theta}}_\infty\in\U$, and $\|{\boldsymbol{\theta}}_k-{\boldsymbol{\theta}}_\infty\|\le \frac{4M_J}{\sqrt{\kappa}}d(\boldsymbol{\theta}_0)(1-\frac{1}{4}\eta\kappa)^{k/2}$, where $1-\frac{1}{4}\eta\kappa\in(0,1)$.

(ii) $K_{\text{GD}}(\varepsilon;\boldsymbol{\theta}_0) = O (\frac{1}{\eta\kappa}\log\frac{1}{\varepsilon} )$ for every $\boldsymbol{\theta}_0\in\mathcal{A}_\tau$,
and $\sup_{\boldsymbol{\theta}_0\in\mathcal{A}_\tau} K_{\text{GD}}(\varepsilon;\boldsymbol{\theta}_0)
= \Theta_{\boldsymbol{\Omega}} (\frac{1}{\eta\kappa}\log\frac{1}{\varepsilon} )$.
\end{theorem}
In particular, under our experimental setting, we obtain the following corollary.

\begin{corollary}
\label{thm:sharp-temperature1}
Let feature table $\Phi:\mathcal S\times\mathcal A\to\mathbb R^d$, reward table $R:\mathcal A\times\mathcal S \to \mathbb{R}$, and an input weight $\mathbf{w}$, $\mathcal N\subset\mathbb N$ be a nonempty finite set of context
lengths. Suppose $n_S,d,m\ge2$, $P_0$ and all transition rows are mutually independent $\operatorname{Dirichlet}(1,\ldots,1)$, independent of $(\mathbf{w},\Phi,R)$, whose entries are mutually independent $\operatorname{Unif}(-1,1)$. The behavior policy is $\epsilon_b$-greedy with uniform exploration, $0<\varepsilon_b\le1$. There exists 
\(\tau_*>0\) such that, for 
 \(0<\tau\le\tau_*\), $n\in \mathcal{N}$ and $\eta$ satisfying Theorem~\ref{thm:main},
$K_{\text{GD}}(\varepsilon;\boldsymbol{\theta}_0) = O (\frac{1}{\eta\tau^5}\log\frac{1}{\varepsilon} )$ for every $\boldsymbol{\theta}_0\in\mathcal{A}_\tau$,
and $\sup_{\boldsymbol{\theta}_0\in\mathcal{A}_\tau} K_{\text{GD}}(\varepsilon;\boldsymbol{\theta}_0,\tau,n)
=\Theta_{\boldsymbol{\Omega},\mathcal{N}} (\frac{1}{\eta\tau^5}\log\frac{1}{\varepsilon} )$.
\end{corollary}
The proofs of Theorem \ref{thm:main} and Corollary \ref{thm:sharp-temperature1} are provided in Appendices \ref{app:Theoretical} and \ref{app:verification}. Corollary 
\ref{thm:sharp-temperature1} shows that as $\tau$ decreases, the local curvature diminishes, slowing convergence and requiring more iterations to achieve the same optimization accuracy, consistent with our empirical results.

\section{Discussion and limitations}
\label{main:lastpage}
This work takes a step toward white-boxing in-context reinforcement learning for off-policy control. We show that an attention-only Transformer can implement a finite-temperature Q-learning update in its forward pass and recover standard window Q-learning as the temperature approaches zero. We further show that the corresponding parameter structure can emerge through teacher-mimicking training, and establish local exponential convergence to a manifold of functionally equivalent soft Q-learning realizations, with particular attention to how action selection in soft Q-updates affects parameter emergence. These results provide a mechanistic account of how Q-learning can be represented and learned within Transformer computation.

Our analysis has several limitations. First, our theoretical analysis is restricted to the support parameter blocks because the valid parameter blocks exhibit substantial redundancy; extending the convergence analysis to all valid blocks remains future work. Second, we do not analyze minibatch SGD, which would be needed to characterize the dependence of convergence on minibatch size. 
\clearpage

\section*{AI use statement}
Generative AI tools were used to assist with code review, experimental implementation and analysis, verification of mathematical derivations, and linguistic refinement of the manuscript. In Appendix \ref{app:verification}, AI tools contributed key proof steps that were adopted by the authors. The core ideas of the other proofs were developed by the authors, with AI assistance in checking mathematical arguments, carrying out partial algebraic calculations, and identifying and correcting errors. The authors checked the proofs step by step and executed and tested the AI-assisted programs. All numerical results reported in the paper were obtained from executable simulation or model-evaluation code and its saved outputs. All AI-assisted content was reviewed and verified by the authors, who retain full responsibility for the accuracy, validity, and final presentation of this work.

\section*{Ethics Statement}
This work investigates the representation and learning of Q-learning updates in structured Transformers through mathematical analysis and experiments on synthetically generated Markov decision processes. The study does not involve human participants, personal or sensitive data, or real-world deployment. We do not identify specific ethical concerns arising from the reported experiments. The results apply under the stated architectural and mathematical assumptions and do not constitute a safety guarantee for real-world reinforcement-learning deployment.

\section*{Reproducibility Statement}
The Transformer architecture, prompt construction, attention mask, and teacher-mimicking objective are specified in Sections \ref{sec:construction} and \ref{IC-Q-Learning} and Algorithm \ref{alg:fresh}. The assumptions and proofs supporting the theoretical results are provided in Appendices \ref{app:two_proof}, \ref{app:algorithm}, \ref{app:Theoretical}, and \ref{app:verification}. Section \ref{ex:main} and Appendix \ref{sec:Experimental} describe the synthetic task generation, training configurations, evaluation protocols, and definitions of the reported metrics. Appendix \ref{sec:Experimental} also reports results across five paired random seeds and describes the uncertainty estimates used in the experiments. An anonymized implementation is provided in the supplementary material.
\bibliography{iclr2027_conference}
\bibliographystyle{iclr2027_conference}

\clearpage
\appendix
\numberwithin{equation}{section}
\numberwithin{theorem}{section}
\section{Two-layer Transformer for soft Q-learning}
\label{app:two_proof}

Softmax weights at finite temperature $\tau$ are not exactly equivalent to $\max$. Recall the definition of $\pi^\tau_{i,a}$ in \eqref{eq:piw}. Lemma \ref{lem:Zero} states that softmax operator can approximate $\max$ arbitrarily well, with the true hard-max limit recovered only as the softmax temperature $\tau \to 0$.
\begin{lemma}
\label{lem:Zero}
For each $i$, define
$q_{i,a}:=\mathbf w^\top \phi_{i,a}^{+}$,
$q_i^\star:=\max_{a\in\mathcal A} q_{i,a}$, and
$\mathcal A_i^\star:=\arg\max_{a\in\mathcal A} q_{i,a}$.
Let $m:=|\mathcal A|<\infty$ and $K_i:=|\mathcal A_i^\star|$.
Then, as $\tau\downarrow0$, 
$\pi_{i,a}^{\tau}\to \frac{1}{K_i}\mathbf{1}\{a\in\mathcal A_i^\star\}$.
If $\mathcal A_i^\star\neq\mathcal A$, define the action gap $\Delta_i := q_i^\star-\max_{a\notin\mathcal A_i^\star}q_{i,a}
>0$, then $1-\sum_{a\in\mathcal A_i^\star}\pi_{i,a}^{\tau} \le [(m-K_i)/K_i] e^{-\Delta_i/\tau}$.
\end{lemma}

\begin{proof}[Proof of Lemma~\ref{lem:Zero}]
Fix $i$. Dividing the numerator and denominator of
$\pi_{i,a}^{\tau}$ by $\exp(q_i^\star/\tau)$ gives
\[
\pi_{i,a}^{\tau}
=
\frac{\exp(-(q_i^\star-q_{i,a})/\tau)}
{\sum_{b\in\mathcal A}
\exp(-(q_i^\star-q_{i,b})/\tau)}.
\]
For every $a\in\mathcal A_i^\star$, we have
$q_i^\star-q_{i,a}=0$, whereas for
$a\notin\mathcal A_i^\star$,
$q_i^\star-q_{i,a}>0$. Hence, as $\tau \to 0$,
$\sum_{b\in\mathcal A}
\exp(-(q_i^\star-q_{i,b})/\tau)
\rightarrow K_i$,
which yields $\pi_{i,a}^{\tau}\to 1/K_i$, for $a\in\mathcal A_i^\star$, $\pi_{i,a}^{\tau}\to 0$, for $a\notin\mathcal A_i^\star$.
Therefore, $\sum_{a\in\mathcal A_i^\star}\pi_{i,a}^{\tau}\to1$.
Now suppose $\mathcal A_i^\star\neq\mathcal A$ and define $\Delta_i
=
q_i^\star-\max_{a\notin\mathcal A_i^\star}q_{i,a}>0$. Then
\[
1-\sum_{a\in\mathcal A_i^\star}\pi_{i,a}^{\tau}
=
\frac{
\sum_{a\notin\mathcal A_i^\star}
e^{-(q_i^\star-q_{i,a})/\tau}
}{
K_i+
\sum_{a\notin\mathcal A_i^\star}
e^{-(q_i^\star-q_{i,a})/\tau}
}.
\]
Since $q_i^\star-q_{i,a}\ge\Delta_i$ for all
$a\notin\mathcal A_i^\star$,
$$1-\sum_{a\in\mathcal A_i^\star}\pi_{i,a}^{\tau}
\le
\frac{m-K_i}{K_i}e^{-\Delta_i/\tau}.$$

When $K_i=1$, this reduces to $1-\pi_{i,a_i^\star}^{\tau}
\le
(m-1)e^{-\Delta_i/\tau}$.
\end{proof}

We proceed to prove Theorem \ref{Th:Qlearning1}. Our Transformer architecture is two-layered, where the first layer adopts single-head softmax attention to produce Boltzmann-weighted aggregation across all candidate successor actions. To facilitate the proof, we first formally characterize the output behavior of this layer in Lemma \ref{Lem:Q1}, before establishing the full result of Theorem \ref{Th:Qlearning1}.

We first recall the first-layer attention mask. 
Let $T_i$, $C_{i,a}$, and $P_o$ denote the positions of $t_i$, $c_{i,a}$, and $p$, respectively. 
For each transition query $T_i$, only its associated candidate-action tokens are unmasked:
\[
\mathbf M_{C_{j,a},T_i}
=
\begin{cases}
0, & j=i,\\
-\infty, & j\neq i,
\end{cases}
\qquad
\mathbf M_{T_j,T_i}
=
\mathbf M_{P_o,T_i}
=
-\infty.
\]
For each candidate-action query and for the parameter query, only the parameter token $p$ is unmasked: $\mathbf M_{P_o,C_{i,a}}=0$, 
$\mathbf M_{P_o,P_o}=0$, 
and all remaining entries are set to $-\infty$. 
Hence, each transition token $t_i$ attends exactly to its own candidate-action set $\{c_{i,a}\}_{a\in\mathcal A}$.

\begin{lemma}[First layer computes the Boltzmann-weighted successor feature]
\label{Lem:Q1}
Consider the single-head masked softmax attention layer in
\eqref{eq:promH0} with mask $\mathbf{M}$ and parameters $\mathbf{P}_1,\mathbf{V}_1$ specified
in \eqref{eq:layer1}. Define
\[
 \zeta_Q=(\mathbf{Q}-I_{d+1})\widetilde{\mathbf{w}}
 =\begin{bmatrix}\rho-1+b^\top \mathbf{w}\\0_d\end{bmatrix},
 \qquad
 \widetilde{\mathbf{w}}_Q={\mathbf{Q}}\widetilde{\mathbf{w}}
 =\begin{bmatrix}\rho+b^\top \mathbf{w}\\\mathbf{w}\end{bmatrix}.
\]
For the input prompt $\mathbf{H}^{(0)}$ in \eqref{eq:Hzero}, the output is
\[
 \mathbf{H}^{(1)}=
 \bigl[t_0^{(1)},\mathbf{C}_0^{(1)},t_1^{(1)},\mathbf{C}_1^{(1)},\ldots,
       t_{n-1}^{(1)},\mathbf{C}_{n-1}^{(1)},p^{(1)}\bigr],
\]
where $\mathbf{C}_i^{(1)}=[c_{i,a}^{(1)}]_{a\in\A}$ and
\begin{equation}
 t_i^{(1)}=
 \begin{bmatrix}
  \phi_i+\mathbf{D}_\Phi\bar\phi_i^{+,\tau}\\
  \mathbf{L}\bar\phi_i^{+,\tau}\\
r_{i+1}+\mathbf{D}_R\bar\phi_i^{+,\tau}\\
  0_d\\
\mathbf{w}+\mathbf{D}_U\bar\phi_i^{+,\tau}\\
  0_{d+1}
 \end{bmatrix},
 \quad
 c_{i,a}^{(1)}=
 \begin{bmatrix}
  0_d\\0_d\\0\\\phi^+_{i,a}\\0_d\\\zeta_Q
 \end{bmatrix},
 \quad
 p^{(1)}=
 \begin{bmatrix}
  0_d\\0_d\\0\\0_d\\0_d\\\widetilde{\mathbf{w}}_Q
 \end{bmatrix}.
 \label{eq:first-layer-tokens}
\end{equation}
\end{lemma}

\begin{proof}
Write $\mathbf{H}^{(0)}=[h_1,\ldots,h_N]$. Let $\bar{\mathbf{A}}:=\softmax \left((\mathbf{H}^{(0)})^\top \mathbf{P}_1\mathbf{H}^{(0)}+\mathbf{M}\right)$,
where the softmax is applied columnwise over the key positions. Thus,
if $h_k$ and $h_q$ denote the tokens at key position $k$ and query
position $q$, respectively, then
\begin{equation}
 \bar{\mathbf{A}}_{k,q}
 =\frac{\exp(h_k^\top \mathbf{P}_1h_q+\mathbf{M}_{k,q})}
 {\sum_{\ell=1}^N\exp(h_\ell^\top \mathbf{P}_1h_q+\mathbf{M}_{\ell,q})}.
 \label{eq:attention-entry}
\end{equation}
Since the $k$th column of $\mathbf{V}_1\mathbf{H}^{(0)}$ is $\mathbf{V}_1h_k$, the $q$th column
of the attention output satisfies 
 $\bigl[(\mathbf{V}_1\mathbf{H}^{(0)})\bar{\mathbf{A}}\bigr]_{:,q}
 =\sum_{k=1}^N\bar{\mathbf{A}}_{k,q}\mathbf{V}_1h_k$. 
Therefore, using the residual connection, the output corresponding to
query token $h_q$ is
\begin{equation}
 h_q^{(1)}=h_q+\sum_{k=1}^N\bar{\mathbf{A}}_{k,q}\mathbf{V}_1h_k.
 \label{eq:first-layer-token-residual}
\end{equation}

We first consider a transition token $t_i$. By construction of the mask,
the query $t_i$ can attend only to its associated candidate-action tokens
$\{c_{i,a}\}_{a\in\A}$. Hence, the only finite attention scores in the
column corresponding to $t_i$ are $c_{i,a}^\top \mathbf{P}_1t_i$, $a\in\A$.

Recall that $t_i=[\phi_i;0_d;r_{i+1};0_d;\mathbf{w};0_{d+1}]$,
 $c_{j,a}=[0_d;0_d;0;\phi^+_{j,a};0_d;0_{d+1}]$. 
Since $c_{j,a}$ is nonzero only in its $\Psi$ block, the inner product
$c_{j,a}^\top \mathbf{P}_1t_i$ depends only on the $\Psi$ block of $\mathbf{P}_1t_i$.
From the fourth block row of $\mathbf{P}_1$,
 $(\mathbf{P}_1t_i)_\Psi
 =0_{d\times d}\phi_i+\star0_d+0_{d\times1}r_{i+1}
   +\star0_d+\mathbf{A}_\theta \mathbf{w}+\star0_{d+1}$. 
   
All terms except the fifth vanish, so $(\mathbf{P}_1t_i)_\Psi=\mathbf{A}_\theta \mathbf{w}$.
Since $\mathbf{A}_\theta=\tau^{-1}I_d$, we obtain
$(\mathbf{P}_1t_i)_\Psi=\tau^{-1}\mathbf{w}$. Therefore, for any candidate token $c_{j,a}$, 
 $c_{j,a}^\top \mathbf{P}_1t_i
 =(\phi^+_{j,a})^\top(\mathbf{P}_1t_i)_\Psi
 =\frac1\tau \mathbf{w}^\top\phi^+_{j,a}$.
The arbitrary $\star$ blocks in $\mathbf{P}_1$ do not affect the relevant
attention scores. Although $\mathbf{P}_1t_i$ may have nonzero components in other
blocks, the candidate token $c_{j,a}$ has zeros in all of those
coordinates and therefore removes them in the inner product.

The mask now restricts the query $t_i$ to the candidates with the same
transition index $i$. Hence, $M_{C_{i,a},T_i}=0$, and $M_{k,T_i}=-\infty$
for every other key position $k$. It follows from
\eqref{eq:attention-entry} that
\begin{equation}
 \bar{\mathbf{A}}_{C_{i,a},T_i}
 =\frac{\exp(c_{i,a}^\top \mathbf{P}_1t_i)}
        {\sum_{b\in\A}\exp(c_{i,b}^\top \mathbf{P}_1t_i)}
 =\frac{\exp(\mathbf{w}^\top\phi^+_{i,a}/\tau)}
        {\sum_{b\in\A}\exp(\mathbf{w}^\top\phi^+_{i,b}/\tau)}
 =\pi_{i,a}^\tau.
 \label{eq:transition-attention-weights}
\end{equation}
All other attention weights in the $T_i$th column are zero.

We next compute the corresponding value vectors. Since $c_{i,a}$ is
nonzero only in its $\Psi$ block, only the fourth block column of ${\mathbf{V}}_1$
contributes to ${\mathbf{V}}_1c_{i,a}$. From the prescribed form of ${\mathbf{V}}_1$, this
block column is
\[
 ({\mathbf{V}}_1)_{:,\Psi}
 =\begin{bmatrix}
   {\mathbf{D}}_\Phi\\{\mathbf{L}}\\{\mathbf{D}}_R\\0_{d\times d}\\{\mathbf{D}}_U\\0_{(d+1)\times d}
  \end{bmatrix},\quad {\mathbf{V}}_1c_{i,a}
 =\begin{bmatrix}
   {\mathbf{D}}_\Phi\phi^+_{i,a}\\
   {\mathbf{L}}\phi^+_{i,a}\\
   {\mathbf{D}}_R\phi^+_{i,a}\\
   0_d\\
   {\mathbf{D}}_U\phi^+_{i,a}\\
   0_{d+1}
  \end{bmatrix}
\]
Hence, the attention contribution to the transition query $t_i$ is
\begin{align}
\bigl[({\mathbf{V}}_1\mathbf{H}^{(0)})\bar{\mathbf{A}}\bigr]_{:,T_i}
&=\sum_{a\in\A}\bar{\mathbf{A}}_{C_{i,a},T_i}{\mathbf{V}}_1c_{i,a}
   =\sum_{a\in\A}\pi_{i,a}^\tau {\mathbf{V}}_1c_{i,a}\nonumber\\
 &=\begin{bmatrix}
{\mathbf{D}}_\Phi\sum_{a\in\A}\pi_{i,a}^\tau\phi^+_{i,a}\\
   {\mathbf{L}}\sum_{a\in\A}\pi_{i,a}^\tau\phi^+_{i,a}\\
   {\mathbf{D}}_R\sum_{a\in\A}\pi_{i,a}^\tau\phi^+_{i,a}\\
   0_d\\
   {\mathbf{D}}_U\sum_{a\in\A}\pi_{i,a}^\tau\phi^+_{i,a}\\
   0_{d+1}
  \end{bmatrix} =\begin{bmatrix}
   {\mathbf{D}}_\Phi\bar\phi_i^{+,\tau}\\
   {\mathbf{L}}\bar\phi_i^{+,\tau}\\
   {\mathbf{D}}_R\bar\phi_i^{+,\tau}\\ 0_d\\ {\mathbf{D}}_U\bar\phi_i^{+,\tau}\\
   0_{d+1}
  \end{bmatrix},
 \label{eq:first-layer-transition-increment}
\end{align}
where we used $\bar{\phi}_i^{+,\tau}
=\sum_{a\in\A}\pi_{i,a}^\tau\phi^+_{i,a}$ in the last equality.
 
Finally, applying the residual connection gives
\begin{align}
t_i^{(1)}=t_i+\bigl[({\mathbf{V}}_1\mathbf{H}^{(0)})\bar{\mathbf{A}}\bigr]_{:,T_i} =\begin{bmatrix}\phi_i\\0_d\\r_{i+1}\\0_d\\\mathbf{w}\\0_{d+1}\end{bmatrix}
  +\begin{bmatrix} {\mathbf{D}}_\Phi\bar\phi_i^{+,\tau}\\{\mathbf{L}}\bar\phi_i^{+,\tau}\\
   {\mathbf{D}}_R\bar\phi_i^{+,\tau}\\0_d\\{\mathbf{D}}_U\bar\phi_i^{+,\tau}\\0_{d+1}
   \end{bmatrix} =\begin{bmatrix}
   \phi_i+{\mathbf{D}}_\Phi\bar\phi_i^{+,\tau}\\
   {\mathbf{L}}\bar\phi_i^{+,\tau}\\
   r_{i+1}+{\mathbf{D}}_R\bar\phi_i^{+,\tau}\\
   0_d\\\mathbf{w}+{\mathbf{D}}_U\bar\phi_i^{+,\tau}\\0_{d+1}
   \end{bmatrix}.
 \label{eq:first-layer-transition-residual}
\end{align}

It remains to compute the candidate-action tokens and the parameter
token. By the mask construction, each candidate-action query $c_{i,a}$
and the parameter query $p$ can attend only to $p$. There is exactly one
unmasked key in each of these columns, so
\[
 \bar{\mathbf{A}}_{P_o,C_{i,a}}
 =\frac{\exp(p^\top \mathbf{P}_1c_{i,a})}{\exp(p^\top \mathbf{P}_1c_{i,a})}=1,
 \qquad
 \bar{\mathbf{A}}_{P_o,P_o}
 =\frac{\exp(p^\top \mathbf{P}_1p)}{\exp(p^\top \mathbf{P}_1p)}=1.
\]
All other attention weights in these columns are zero.

Recall that $p=[0_d;0_d;0;0_d;0_d;\widetilde{\mathbf{w}}]$. The last block column
of ${\mathbf{V}}_1$, corresponding to the $W$ block, is zero except for its last
block ${\mathbf{Q}}-I_{d+1}$. Hence
\[
 {\mathbf{V}}_1p
 =\begin{bmatrix}
    0_d\\0_d\\0\\0_d\\0_d\\({\mathbf{Q}}-I_{d+1})\widetilde{\mathbf{w}}
   \end{bmatrix}
 =\begin{bmatrix}0_d\\0_d\\0\\0_d\\0_d\\\zeta_Q\end{bmatrix},
\]
where, explicitly, 
 $({\mathbf{Q}}-I_{d+1})\widetilde{\mathbf{w}}
 =\begin{bmatrix}\rho-1&b^\top\\0_{d\times1}&0_{d\times d}\end{bmatrix}
   \begin{bmatrix}1\\{\mathbf{w}}\end{bmatrix}
 =\begin{bmatrix}\rho-1+b^\top {\mathbf{w}}\\0_d\end{bmatrix}
 =\zeta_Q$. 
 
Therefore, the attention outputs for all candidate-action queries and
for the parameter query are 
 $\bigl[({\mathbf{V}}_1\mathbf{H}^{(0)})\bar{\mathbf{A}}\bigr]_{:,C_{i,a}}={\mathbf{V}}_1p$, 
 $\bigl[({\mathbf{V}}_1\mathbf{H}^{(0)})\bar{\mathbf{A}}\bigr]_{:,P_o}={\mathbf{V}}_1p$. 
Using the residual connection once again,
\begin{align*}
 c_{i,a}^{(1)}
  =c_{i,a}+{\mathbf{V}}_1p =\begin{bmatrix}0_d\\0_d\\0\\\phi^+_{i,a}\\0_d\\0_{d+1}\end{bmatrix} +\begin{bmatrix}0_d\\0_d\\0\\0_d\\0_d\\\zeta_Q\end{bmatrix} =\begin{bmatrix}0_d\\0_d\\0\\\phi^+_{i,a}\\0_d\\\zeta_Q\end{bmatrix}, \quad p^{(1)}
  =p+{\mathbf{V}}_1p =\begin{bmatrix}0_d\\0_d\\0\\0_d\\0_d\\\widetilde{\mathbf{w}}\end{bmatrix} +\begin{bmatrix}0_d\\0_d\\0\\0_d\\0_d\\\zeta_Q\end{bmatrix} =\begin{bmatrix}0_d\\0_d\\0\\0_d\\0_d\\\widetilde {\mathbf{w}}_Q\end{bmatrix},
\end{align*}
where $\widetilde{\mathbf{w}}+\zeta_Q
 =\bigl(I_{d+1}+{\mathbf{Q}}-I_{d+1}\bigr)\widetilde{\mathbf{w}}
 ={\mathbf{Q}}\widetilde{\mathbf{w}}
 =\widetilde{\mathbf{w}}_Q$. 

Thus, $\mathbf{H}^{(1)}= \bigl[t_0^{(1)},\mathbf{C}_0^{(1)},t_1^{(1)},\mathbf{C}_1^{(1)},\ldots, t_{n-1}^{(1)},\mathbf{C}_{n-1}^{(1)},p^{(1)}\bigr]$,
with all three types of tokens as stated. In particular, the
candidate-action tokens and the parameter token remain unchanged for
every ${\mathbf{w}}$ precisely when $\rho=1$ and $b=0$, that is, when ${\mathbf{Q}}=I_{d+1}$.
\end{proof}

With the above lemma established, we proceed to prove
Theorem~\ref{Th:Qlearning1}.
\begin{proof}[Proof of Theorem~\ref{Th:Qlearning1}]
\textup{(i)} By Lemma~\ref{Lem:Q1}, after the first
attention layer, each transition token takes the form
\[
 t_i^{(1)}=
 \begin{bmatrix}
  \phi_i+{\mathbf{D}}_\Phi\bar\phi_i^{+,\tau}\\
  {\mathbf{L}}\bar\phi_i^{+,\tau}\\
  r_{i+1}+{\mathbf{D}}_R\bar\phi_i^{+,\tau}\\
  0_d\\{\mathbf{w}}+{\mathbf{D}}_U\bar\phi_i^{+,\tau}\\0_{d+1}
 \end{bmatrix},
 \qquad
 \bar\phi_i^{+,\tau}=\sum_{a\in\A}\pi_{i,a}^\tau\phi^+_{i,a}.
\]
The candidate-action tokens and the parameter token have the forms
 $c_{i,a}^{(1)}=[0_d;0_d;0;\phi^+_{i,a};0_d;\zeta_Q]$,
 $p^{(1)}=[0_d;0_d;0;0_d;0_d;\widetilde {\mathbf{w}}_Q]$.
In particular, the $\Phi$ and $G$ blocks of these latter tokens are
still zero, and the only possibly nonzero block of $p^{(1)}$ is its
$W$ block.

We now consider the second linear-attention layer. Write
$\mathbf{H}^{(1)}=[h_1^{(1)},\ldots,h_N^{(1)}]$. From 
$\mathbf{H}^{(2)}=\mathbf{H}^{(1)}+\frac1n({\mathbf{V}}_2\mathbf{H}^{(1)})
\bigl((\mathbf{H}^{(1)})^\top \mathbf{P}_2\mathbf{H}^{(1)}\bigr)$, the update to the parameter token is
\begin{equation}
 \Delta p:=p^{(2)}-p^{(1)}
 =\frac1n\sum_{j=1}^N
   {\mathbf{V}}_2h_j^{(1)}\bigl((h_j^{(1)})^\top \mathbf{P}_2p^{(1)}\bigr).
 \label{eq:query-increment}
\end{equation}
Indeed, this is simply the column of the matrix product corresponding
to the query token $p^{(1)}$: the $j$th entry of the corresponding score
column is $(h_j^{(1)})^\top \mathbf{P}_2p^{(1)}$, and the $j$th value column is
${\mathbf{V}}_2h_j^{(1)}$.

We only need the last $d$ coordinates of $\Delta p$, since the readout
extracts these coordinates of the parameter token. Recall 
 $\Pi_{\out}=[0_{d\times(4d+2)},I_d]$, 
 ${\mathbf{C}}=\frac\alpha c I_d$, ${\mathbf{C}}_G=-{\mathbf{C}}{\mathbf{D}}_\Phi {\mathbf{L}}^{-1}$. 
By the block structure of ${\mathbf{V}}_2$, for any token $h$,
\begin{equation}
 \Pi_{\out}{\mathbf{V}}_2h={\mathbf{C}}h_\Phi+{\mathbf{C}}_Gh_G,
 \label{eq:projected-value-general}
\end{equation}
where $h_\Phi$ and $h_G$ denote its $\Phi$ and $G$ blocks. The arbitrary
entries in the first $4d+2$ rows of ${\mathbf{V}}_2$ do not appear in this expression.
Therefore,
\begin{align}
 \Pi_{\out}{\mathbf{V}}_2t_i^{(1)}
 &={\mathbf{C}}\bigl(\phi_i+{\mathbf{D}}_\Phi\bar\phi_i^{+,\tau}\bigr)
       +{\mathbf{C}}_G {\mathbf{L}}\bar\phi_i^{+,\tau}\nonumber\\
 &={\mathbf{C}}\phi_i+{\mathbf{C}}{\mathbf{D}}_\Phi\bar\phi_i^{+,\tau}
       -{\mathbf{C}}{\mathbf{D}}_\Phi {\mathbf{L}}^{-1}{\mathbf{L}}\bar\phi_i^{+,\tau}\nonumber\\
 &={\mathbf{C}}\phi_i+{\mathbf{C}}{\mathbf{D}}_\Phi\bar\phi_i^{+,\tau}
       -{\mathbf{C}}{\mathbf{D}}_\Phi\bar\phi_i^{+,\tau}={\mathbf{C}}\phi_i=\frac\alpha c\phi_i.
 \label{eq:projected-transition-value}
\end{align}
Here the second equality uses ${\mathbf{C}}_G=-{\mathbf{C}}{\mathbf{D}}_\Phi {\mathbf{L}}^{-1}$, and the third
uses ${\mathbf{L}}^{-1}{\mathbf{L}}=I_d$.

On the other hand, both $c_{i,a}^{(1)}$ and $p^{(1)}$ have zero $\Phi$
and $G$ blocks. Applying \eqref{eq:projected-value-general} gives
\begin{equation}
 \Pi_{\out}{\mathbf{V}}_2c_{i,a}^{(1)}={\mathbf{C}}0_d+{\mathbf{C}}_G0_d=0_d,
 \qquad
 \Pi_{\out}{\mathbf{V}}_2p^{(1)}={\mathbf{C}}0_d+{\mathbf{C}}_G0_d=0_d.
 \label{eq:zero-projected-values}
\end{equation}
Thus, only the transition tokens contribute to the last $d$ coordinates
of the parameter update. The additional $W$ block $\zeta_Q$ in each
candidate token does not affect this conclusion, since the last $d$ rows
of the $W$ block column of ${\mathbf{V}}_2$ are zero. No assumption that the
candidate or parameter scores vanish is needed: their projected value
vectors already vanish.

It remains to compute the scalar score $(t_i^{(1)})^\top \mathbf{P}_2p^{(1)}$.
Since 
 $p^{(1)}=[0_d;0_d;0;0_d;0_d;\widetilde {\mathbf{w}}_Q]$, 
 $\widetilde {\mathbf{w}}_Q={\mathbf{Q}}\widetilde {\mathbf{w}}={\mathbf{Q}}\begin{bmatrix}1\\{\mathbf{w}}\end{bmatrix}$, 
only the last block column of $\mathbf{P}_2$ contributes to $\mathbf{P}_2p^{(1)}$.
Because ${\mathbf{Q}}$ is invertible, 
 ${\mathbf{Q}}^{-1}\widetilde {\mathbf{w}}_Q={\mathbf{Q}}^{-1}{\mathbf{Q}}\widetilde {\mathbf{w}}
 =\widetilde {\mathbf{w}}=\begin{bmatrix}1\\{\mathbf{w}}\end{bmatrix}$. 
From the definition of the specified blocks, we obtain
\begin{align}
 {\mathbf{B}}_-\widetilde {\mathbf{w}}_Q
 &=-c\mathbf{J}{\mathbf{Q}}^{-1}{\mathbf{Q}}\widetilde {\mathbf{w}}
  =-c[0_{d\times1},I_d]\begin{bmatrix}1\\{\mathbf{w}}\end{bmatrix}
  =-c{\mathbf{w}},\label{eq:Bminus-query}\\
 {\mathbf{B}}_+\widetilde {\mathbf{w}}_Q
 &=c{\mathbf{L}}^{-\top}\bigl[-{\mathbf{D}}_R^\top,
       \gamma I_d+{\mathbf{D}}_\Phi^\top-{\mathbf{D}}_U^\top\Omega\bigr]
       {\mathbf{Q}}^{-1}{\mathbf{Q}}\widetilde {\mathbf{w}}\nonumber\\
 &=c{\mathbf{L}}^{-\top}\bigl[-{\mathbf{D}}_R^\top,
       \gamma I_d+{\mathbf{D}}_\Phi^\top-{\mathbf{D}}_U^\top\Omega\bigr]
       \begin{bmatrix}1\\{\mathbf{w}}\end{bmatrix}\nonumber\\
 &=c{\mathbf{L}}^{-\top}\left[-{\mathbf{D}}_R^\top+
       (\gamma I_d+{\mathbf{D}}_\Phi^\top-{\mathbf{D}}_U^\top\Omega){\mathbf{w}}\right],
       \label{eq:Bplus-query}\\
 {\mathbf{B}}_r\widetilde {\mathbf{w}}_Q
 &=ce_0^\top {\mathbf{Q}}^{-1}{\mathbf{Q}}\widetilde {\mathbf{w}}
  =c[1,0_{1\times d}]\begin{bmatrix}1\\{\mathbf{w}}\end{bmatrix}
  =c,\label{eq:Br-query}\\
 {\mathbf{B}}_U\widetilde {\mathbf{w}}_Q
 &=c\Omega J{\mathbf{Q}}^{-1}{\mathbf{Q}}\widetilde {\mathbf{w}}
  =c\Omega[0_{d\times1},I_d]\begin{bmatrix}1\\{\mathbf{w}}\end{bmatrix}
  =c\Omega {\mathbf{w}}.
  \label{eq:BU-query}
\end{align}
The fourth and sixth block rows of $\mathbf{P}_2p^{(1)}$ may contain arbitrary
terms arising from the corresponding $\star$ blocks. However, the
$\Psi$ and $W$ blocks of $t_i^{(1)}$ are both zero, so these terms do
not contribute to the bilinear score. 
Consequently,
\begin{align}
 (t_i^{(1)})^\top \mathbf{P}_2p^{(1)}
 =&\bigl(\phi_i+{\mathbf{D}}_\Phi\bar\phi_i^{+,\tau}\bigr)^\top
       {\mathbf{B}}_-\widetilde {\mathbf{w}}_Q+\bigl({\mathbf{L}}\bar\phi_i^{+,\tau}\bigr)^\top
       {\mathbf{B}}_+\widetilde {\mathbf{w}}_Q+\bigl(r_{i+1}+{\mathbf{D}}_R\bar\phi_i^{+,\tau}\bigr)
       {\mathbf{B}}_r\widetilde {\mathbf{w}}_Q\nonumber\\
&+\bigl({\mathbf{w}}+{\mathbf{D}}_U\bar\phi_i^{+,\tau}\bigr)^\top {\mathbf{B}}_U\widetilde {\mathbf{w}}_Q\nonumber\\
 =&-c\bigl(\phi_i+{\mathbf{D}}_\Phi\bar\phi_i^{+,\tau}\bigr)^\top {\mathbf{w}}+c\bigl({\mathbf{L}}\bar\phi_i^{+,\tau}\bigr)^\top {\mathbf{L}}^{-\top}\left[-{\mathbf{D}}_R^\top+
(\gamma I_d+{\mathbf{D}}_\Phi^\top-{\mathbf{D}}_U^\top\Omega){\mathbf{w}}\right]\nonumber\\
&+c\bigl(r_{i+1}+{\mathbf{D}}_R\bar\phi_i^{+,\tau}\bigr)
       +c\bigl({\mathbf{w}}+{\mathbf{D}}_U\bar\phi_i^{+,\tau}\bigr)^\top\Omega {\mathbf{w}}.
 \label{eq:transition-score-unexpanded}
\end{align}
We simplify these terms one by one. First, 
 $-c\bigl(\phi_i+{\mathbf{D}}_\Phi\bar\phi_i^{+,\tau}\bigr)^\top {\mathbf{w}}
 =-c\phi_i^\top {\mathbf{w}}-c(\bar\phi_i^{+,\tau})^\top {\mathbf{D}}_\Phi^\top {\mathbf{w}}$. 
Next, since ${\mathbf{L}}^\top {\mathbf{L}}^{-\top}=I_d$,
\begin{align*}
 &c\bigl({\mathbf{L}}\bar\phi_i^{+,\tau}\bigr)^\top {\mathbf{L}}^{-\top}
       \left[-{\mathbf{D}}_R^\top+
       (\gamma I_d+{\mathbf{D}}_\Phi^\top-{\mathbf{D}}_U^\top\Omega){\mathbf{w}}\right]\\
 &\quad=c(\bar\phi_i^{+,\tau})^\top {\mathbf{L}}^\top {\mathbf{L}}^{-\top}
       \left[-{\mathbf{D}}_R^\top+
       (\gamma I_d+{\mathbf{D}}_\Phi^\top-{\mathbf{D}}_U^\top\Omega){\mathbf{w}}\right]\\
 &\quad=-c(\bar\phi_i^{+,\tau})^\top {\mathbf{D}}_R^\top
       +c\gamma(\bar\phi_i^{+,\tau})^\top {\mathbf{w}}
       +c(\bar\phi_i^{+,\tau})^\top {\mathbf{D}}_\Phi^\top {\mathbf{w}}
       -c(\bar\phi_i^{+,\tau})^\top {\mathbf{D}}_U^\top\Omega {\mathbf{w}}.
\end{align*}
The remaining two terms satisfy
\begin{align*}
 c\bigl(r_{i+1}+{\mathbf{D}}_R\bar\phi_i^{+,\tau}\bigr)
 &=cr_{i+1}+c{\mathbf{D}}_R\bar\phi_i^{+,\tau},\\
 c\bigl({\mathbf{w}}+{\mathbf{D}}_U\bar\phi_i^{+,\tau}\bigr)^\top\Omega {\mathbf{w}}
 &=c{\mathbf{w}}^\top\Omega {\mathbf{w}}
   +c(\bar\phi_i^{+,\tau})^\top {\mathbf{D}}_U^\top\Omega {\mathbf{w}}.
\end{align*}
The terms involving ${\mathbf{D}}_\Phi$, ${\mathbf{D}}_R$, and ${\mathbf{D}}_U$ therefore cancel in
pairs:
\begin{align*}
 &-c(\bar\phi_i^{+,\tau})^\top {\mathbf{D}}_\Phi^\top {\mathbf{w}}
   +c(\bar\phi_i^{+,\tau})^\top {\mathbf{D}}_\Phi^\top {\mathbf{w}}=0,\\
 &-c(\bar\phi_i^{+,\tau})^\top {\mathbf{D}}_R^\top
   +c{\mathbf{D}}_R\bar\phi_i^{+,\tau}=0,\\
 &-c(\bar\phi_i^{+,\tau})^\top {\mathbf{D}}_U^\top\Omega {\mathbf{w}}
   +c(\bar\phi_i^{+,\tau})^\top {\mathbf{D}}_U^\top\Omega {\mathbf{w}}=0.
\end{align*}
For the second cancellation, we used the fact that
${\mathbf{D}}_R\bar\phi_i^{+,\tau}$ is a scalar, so it equals its transpose
$(\bar\phi_i^{+,\tau})^\top {\mathbf{D}}_R^\top$.
Moreover, since $\Omega^\top=-\Omega$,
\[
 {\mathbf{w}}^\top\Omega {\mathbf{w}}=({\mathbf{w}}^\top\Omega {\mathbf{w}})^\top
 ={\mathbf{w}}^\top\Omega^\top {\mathbf{w}}=-{\mathbf{w}}^\top\Omega {\mathbf{w}},
\]
and therefore ${\mathbf{w}}^\top\Omega {\mathbf{w}}=0$. Substituting these identities into
\eqref{eq:transition-score-unexpanded} yields
\begin{align}
 (t_i^{(1)})^\top \mathbf{P}_2p^{(1)}
 &=-c\phi_i^\top {\mathbf{w}}+c\gamma(\bar\phi_i^{+,\tau})^\top {\mathbf{w}}+cr_{i+1}
       \nonumber\\
 &=c\bigl(r_{i+1}+\gamma {\mathbf{w}}^\top\bar\phi_i^{+,\tau}
                    -{\mathbf{w}}^\top\phi_i\bigr).
 \label{eq:transition-score}
\end{align}
Define $\delta_i^{(\tau)}
 :=r_{i+1}+\gamma {\mathbf{w}}^\top\bar\phi_i^{+,\tau}-{\mathbf{w}}^\top\phi_i$. 
Then $(t_i^{(1)})^\top \mathbf{P}_2p^{(1)}=c\delta_i^{(\tau)}$.

Therefore, the last $d$ coordinates of the second-layer update are
\begin{align}
 \Pi_{\out}\Delta p
 &=\frac1n\sum_{j=1}^N
     \Pi_{\out}{\mathbf{V}}_2h_j^{(1)}
        \bigl((h_j^{(1)})^\top \mathbf{P}_2p^{(1)}\bigr) =\frac1n\sum_{i=0}^{n-1}
     \Pi_{\out}{\mathbf{V}}_2t_i^{(1)}
        \bigl((t_i^{(1)})^\top \mathbf{P}_2p^{(1)}\bigr)\nonumber\\
 &=\frac1n\sum_{i=0}^{n-1}
       \left(\frac\alpha c\phi_i\right)c\delta_i^{(\tau)}
  =\frac\alpha n\sum_{i=0}^{n-1}\phi_i\delta_i^{(\tau)}.
 \label{eq:projected-second-layer-update}
\end{align}
The second equality follows from \eqref{eq:zero-projected-values}, and
the third uses \eqref{eq:projected-transition-value} and
\eqref{eq:transition-score}.

Since the $W$ block of the first-layer parameter token is
$\widetilde {\mathbf{w}}_Q=[\rho+b^\top {\mathbf{w}};{\mathbf{w}}]$, its last $d$ coordinates are
\[
 \Pi_{\out}p^{(1)}
 =\mathbf{J}\widetilde {\mathbf{w}}_Q=\mathbf{J}{\mathbf{Q}}\widetilde {\mathbf{w}}
 =\mathbf{J}\widetilde {\mathbf{w}}
 =[0_{d\times1},I_d]\begin{bmatrix}1\\{\mathbf{w}}\end{bmatrix}
 ={\mathbf{w}}.
\]
Here $\mathbf{J}{\mathbf{Q}}=\mathbf{J}$ follows directly from the prescribed form of ${\mathbf{Q}}$.
The residual connection thus gives
\begin{align*}
 \TF_\tau(\mathbf{H}^{(0)})
 &=\Pi_{\out}p^{(2)}
  =\Pi_{\out}p^{(1)}+\Pi_{\out}\Delta p ={\mathbf{w}}+\frac\alpha n\sum_{i=0}^{n-1}\phi_i\delta_i^{(\tau)}\\
 &={\mathbf{w}}+\frac\alpha n\sum_{i=0}^{n-1}
\phi_i\left[r_{i+1}+\gamma {\mathbf{w}}^\top\bar\phi_i^{+,\tau}-{\mathbf{w}}^\top\phi_i\right].
\end{align*}
The readout extracts precisely these last $d$ coordinates of $p^{(2)}$.
Hence we have proved part~\textup{(i)}. Notice that this argument does
not prescribe the first coordinate of the $W$ block of $p^{(2)}$;
that coordinate is outside the readout.
 
\noindent\textup{(ii)} Recall that
$q_{i,a}:={\mathbf{w}}^\top\phi^+_{i,a}$, $q_i^\star:=\max_{a\in\A}q_{i,a}$, $\A_i^\star:=\arg\max_{a\in\A}q_{i,a}$,  $K_i:=|\A_i^\star|$. 
The soft backup value is $\bar q_i^{(\tau)}
 :={\mathbf{w}}^\top\bar\phi_i^{+,\tau}
 =\sum_{a\in\A}\pi_{i,a}^\tau q_{i,a}$. Since $\sum_{a\in\A}\pi_{i,a}^\tau=1$ and $q_{i,a}=q_i^\star$ for
every $a\in\A_i^\star$,
\begin{align}
 q_i^\star-\bar q_i^{(\tau)}
 &=\sum_{a\in\A}\pi_{i,a}^\tau q_i^\star
    -\sum_{a\in\A}\pi_{i,a}^\tau q_{i,a} =\sum_{a\in\A}\pi_{i,a}^\tau(q_i^\star-q_{i,a}) =\sum_{a\notin\A_i^\star}
       \pi_{i,a}^\tau(q_i^\star-q_{i,a}).
 \label{eq:backup-gap-identity}
\end{align}
In particular, $q_i^\star-\bar q_i^{(\tau)}\ge0$. Since
$|q_{i,a}|\le B_q$ and $|q_i^\star|\le B_q$, we have 
 $0\le q_i^\star-q_{i,a}
 \le |q_i^\star|+|q_{i,a}|\le2B_q$.
Therefore,  $0\le q_i^\star-\bar q_i^{(\tau)}\le2B_q\sum_{a\notin\A_i^\star}\pi_{i,a}^\tau$. 
If $\A_i^\star\ne\A$, Lemma~\ref{lem:Zero} gives
\begin{equation}
 \sum_{a\notin\A_i^\star}\pi_{i,a}^\tau
 =1-\sum_{a\in\A_i^\star}\pi_{i,a}^\tau
 \le\frac{m-K_i}{K_i}e^{-\Delta_i/\tau}.
 \label{eq:suboptimal-mass}
\end{equation}
Hence
\begin{equation}
 0\le q_i^\star-\bar q_i^{(\tau)}
 \le2B_q\frac{m-K_i}{K_i}e^{-\Delta_i/\tau}.
 \label{eq:backup-gap}
\end{equation}
If $\A_i^\star=\A$, then all actions have the same value and
$q_i^\star-\bar q_i^{(\tau)}=0$ for every $\tau>0$.

The difference between the hard Q-learning update and the Transformer
update is therefore
\begin{align*}
 {\mathbf{w}}_Q-{\mathbf{w}}_{\mathrm{TF}}^{(\tau)}
 &=\frac\alpha n\sum_{i=0}^{n-1}\phi_i
   \left[(r_{i+1}+\gamma q_i^\star-{\mathbf{w}}^\top\phi_i) -(r_{i+1}+\gamma\bar q_i^{(\tau)}-{\mathbf{w}}^\top\phi_i)\right] =\frac{\alpha\gamma}{n}\sum_{i=0}^{n-1}\phi_i\bigl(q_i^\star-\bar q_i^{(\tau)}\bigr).
\end{align*}
By the triangle inequality and $\|\phi_i\|_2\le B_\phi$,
\begin{align}
 \bigl\|{\mathbf{w}}_Q-{\mathbf{w}}_{\mathrm{TF}}^{(\tau)}\bigr\|_2
 &\le\frac{\alpha\gamma}{n}\sum_{i=0}^{n-1}
      \|\phi_i\|_2\bigl|q_i^\star-\bar q_i^{(\tau)}\bigr| \le\frac{\alpha\gamma B_\phi}{n}
      \sum_{i:\,\A_i^\star\ne\A}
      \bigl(q_i^\star-\bar q_i^{(\tau)}\bigr)\nonumber\\
 &\le\frac{2\alpha\gamma B_qB_\phi}{n}
      \sum_{i:\,\A_i^\star\ne\A}
      \frac{m-K_i}{K_i}e^{-\Delta_i/\tau}.
 \label{eq:hard-bound-proof}
\end{align}
The second inequality uses the nonnegativity of the backup gap and the
fact that this gap is zero whenever $\A_i^\star=\A$; the last inequality
uses \eqref{eq:backup-gap}.

Since the number of transitions is finite and each $\Delta_i>0$ in
the displayed sum, the right-hand side converges to zero as
$\tau\downarrow0$. Thus $\lim_{\tau\downarrow0}{\mathbf{w}}_{\mathrm{TF}}^{(\tau)}={\mathbf{w}}_Q$. 
In particular, if $m\ge2$ and every maximizer is unique, then $K_i=1$
and $1-\pi_{i,a_i^\star}^{\tau}\le(m-1)e^{-\Delta_i/\tau}$. 
Since there are finitely many transitions,
$\Delta_{\min}:=\min_{0\le i\le n-1}\Delta_i>0$. Consequently,
$e^{-\Delta_i/\tau}\le e^{-\Delta_{\min}/\tau}$, and hence
\begin{align*}
 \bigl\|{\mathbf{w}}_Q-{\mathbf{w}}_{\mathrm{TF}}^{(\tau)}\bigr\|_2
 &\le\frac{2\alpha\gamma B_qB_\phi(m-1)}{n}
          \sum_{i=0}^{n-1}e^{-\Delta_i/\tau} \le\frac{2\alpha\gamma B_qB_\phi(m-1)}{n}
          n e^{-\Delta_{\min}/\tau}\\
 &=2\alpha\gamma B_qB_\phi(m-1)e^{-\Delta_{\min}/\tau}.
\end{align*}
If $m=1$, all backup gaps are zero and
${\mathbf{w}}_{\mathrm{TF}}^{(\tau)}={\mathbf{w}}_Q$ for every $\tau>0$.
This proves the claimed exponential approximation bound.
\end{proof}

We next consider whether emergence of tabular Q‑learning algorithms can be achieved within Transformers. Let $e_{(s,a)}$ denote the one‑hot basis vector for state‑action pair $(s,a)$.

\begin{corollary}[Transformer implements tabular Q-learning]\label{Cor:Q1}
Under the conditions of the preceding theorem, let $\mathcal{S}$ and $\mathcal{A}$ be finite, and consider $\phi(s,a) = e_{(s,a)} \in \mathbb{R}^{|\mathcal{S}||\mathcal{A}|}$, $\mathbf{w}_{(s,a)} = Q(s,a)$.
The two‑layer Transformer constructed in the preceding theorem implements the batch tabular Q‑learning update up to softmax approximation error. 
\end{corollary}
\begin{proof}
According to Theorem \ref{Th:Qlearning1}, the batch‑update formula using $n$ samples is
\[
\text{TF}_\tau(\mathbf{H}^{(0)};\mathbf{P}_1,\mathbf{V}_1,\mathbf{P}_2,\mathbf{V}_2)
= \mathbf{w} + \frac{\alpha}{n}\sum_{i=0}^{n-1} e_{(s_i,a_i)}
\Big[ r_{i+1} + \gamma\sum_{b\in\mathcal{A}} \pi_{i,b}^{\tau} \mathbf{w}_{(s_{i+1},b)}-\mathbf{w}_{(s_i,a_i)} \Big],
\]
where $\pi_{i,a}^{\tau} \coloneqq
\exp\bigl(Q(s_{i+1},a)/\tau\bigr)/\sum_{b\in\mathcal{A}} \exp\bigl(Q(s_{i+1},b)/\tau\bigr)$.
Only the weight components corresponding to visited $(s_i,a_i)$ samples are updated; all remaining coordinates are multiplied by zero and thus remain unchanged. 

The equivalent update rule for the Q-values reads
$$
[\text{TF}_\tau(\mathbf{H}^{(0)};\mathbf{P}_1,\mathbf{V}_1,\mathbf{P}_2,\mathbf{V}_2)]_{(s,a)}
=Q(s,a)+\frac{\alpha}{n}\sum_{i:(s_i,a_i)=(s,a)}\big[r_{i+1}+\gamma\sum_{b\in\mathcal{A}} \pi_{i,b}^{\tau} Q(s_{i+1},b)-Q(s_i,a_i)\big].
$$
Notably, only state‑action pairs $(s,a)$ appearing within the mini‑batch receive updates; all other entries stay fixed.
\end{proof}

\section{Training Algorithm}
\label{app:algorithm}

\subsection{Independent-context gradient variance}\label{app:Algorithm}
\label{app:variance}
Let $\ell({\boldsymbol{\theta}};\mathbf{z})
:=
\frac{1}{2} \| \operatorname{TF}_{\theta} (\mathbf{H}(\mathbf{z})) - \mathbf{w}^{(\tau)}_{\text{soft}}(\mathbf{z}) \|_2^2$, and  $\mathcal L({\boldsymbol{\theta}})
:=
\mathbb E_{\mathbf{z}\sim\mathcal D}[\ell({\boldsymbol{\theta}};\mathbf{z})]$.
Fix ${\boldsymbol{\theta}}$, and let
$\mathbf{z}_1,\ldots,\mathbf{z}_B\overset{\mathrm{i.i.d.}}{\sim}\mathcal D$
be independent of ${\boldsymbol{\theta}}$. Define
$g({\boldsymbol{\theta}};\mathbf{z})=\nabla_{\boldsymbol{\theta}} \ell({\boldsymbol{\theta}};\mathbf{z})$, $\hat g_B({\boldsymbol{\theta}})
=\frac1B\sum_{b=1}^B g({\boldsymbol{\theta}};\mathbf{z}_b)$. 

The following proposition shows that independent-context minibatching preserves the population-gradient
direction in expectation while reducing context-specific gradient
fluctuations.
\begin{proposition}[Independent-context minibatch gradient]
\label{prop:minibatch-gradient}

Assume $\mathbb E\|g({\boldsymbol{\theta}};\mathbf{z})\|_2^2<\infty$ and that differentiation and
expectation can be interchanged. Then
\begin{align}
\mathbb E[\hat g_B({\boldsymbol{\theta}})\mid{\boldsymbol{\theta}}]
&=
\nabla\mathcal L({\boldsymbol{\theta}}),
\label{eq:minibatch-unbiased}\\
\operatorname{Cov}(\hat g_B({\boldsymbol{\theta}})\mid{\boldsymbol{\theta}})
&=
\frac1B
\operatorname{Cov}(g({\boldsymbol{\theta}};\mathbf{z})\mid{\boldsymbol{\theta}}),
\label{eq:minibatch-cov}\\
\mathbb E\!\left[
\left\|
\hat g_B({\boldsymbol{\theta}})-\nabla\mathcal L({\boldsymbol{\theta}})
\right\|_2^2
\middle|{\boldsymbol{\theta}}
\right]
&=
\frac1B
\mathbb E\!\left[
\left\|
g({\boldsymbol{\theta}};\mathbf{z})-\nabla\mathcal L({\boldsymbol{\theta}})
\right\|_2^2
\middle|{\boldsymbol{\theta}}
\right].
\label{eq:minibatch-variance}
\end{align}
Equivalently, for every deterministic direction $v$, $\operatorname{Var}(v^\top\hat g_B({\boldsymbol{\theta}})\mid{\boldsymbol{\theta}})=\frac1B\operatorname{Var}(v^\top g({\boldsymbol{\theta}};\mathbf{z})\mid{\boldsymbol{\theta}})$. 
\end{proposition}

\begin{proof}
(i) Let $\mu({\boldsymbol{\theta}})
:=
\mathbb E_{\mathbf{z}\sim\mathcal D}[g({\boldsymbol{\theta}};\mathbf{z})\mid{\boldsymbol{\theta}}]$.
By the interchange of differentiation and expectation,
\[
\mu({\boldsymbol{\theta}})
=
\mathbb E_{\mathbf{z}\sim\mathcal D}
[\nabla_{\boldsymbol{\theta}}\ell({\boldsymbol{\theta}};\mathbf{z})]
=
\nabla_{\boldsymbol{\theta}}
\mathbb E_{\mathbf{z}\sim\mathcal D}[\ell({\boldsymbol{\theta}};\mathbf{z})]
=
\nabla\mathcal L({\boldsymbol{\theta}}).
\]
Therefore, 
$\mathbb E[\hat g_B({\boldsymbol{\theta}})\mid{\boldsymbol{\theta}}]
=
\frac1B\sum_{b=1}^B
\mathbb E[g({\boldsymbol{\theta}};\mathbf{z}_b)\mid{\boldsymbol{\theta}}]
=
\mu({\boldsymbol{\theta}})
=
\nabla\mathcal L({\boldsymbol{\theta}})$, 
which proves~\eqref{eq:minibatch-unbiased}.

(ii) Now define the centered gradients $\xi_b
:=g({\boldsymbol{\theta}};\mathbf{z}_b)-\mu({\boldsymbol{\theta}})$.
Conditioned on ${\boldsymbol{\theta}}$, the random vectors
$\xi_1,\ldots,\xi_B$ are independent and identically distributed, with $\mathbb E[\xi_b\mid{\boldsymbol{\theta}}]=0$.
Moreover, $\hat g_B({\boldsymbol{\theta}})-\nabla\mathcal L({\boldsymbol{\theta}})
=
\frac1B\sum_{b=1}^B\xi_b$. 
Hence
\begin{align*}
\operatorname{Cov}(\hat g_B\mid{\boldsymbol{\theta}})
 &=
\mathbb E\!\left[
\left(
\frac1B\sum_{b=1}^B\xi_b
\right)
\left(
\frac1B\sum_{b'=1}^B\xi_{b'}
\right)^\top
\middle|{\boldsymbol{\theta}}
\right] \\
&=
\frac1{B^2}
\sum_{b=1}^B
\mathbb E[\xi_b\xi_b^\top\mid{\boldsymbol{\theta}}]
+
\frac1{B^2}
\sum_{b\neq b'}
\mathbb E[\xi_b\xi_{b'}^\top\mid{\boldsymbol{\theta}}].
\end{align*}
For $b\neq b'$, conditional independence and centering give 
$\mathbb E[\xi_b\xi_{b'}^\top\mid{\boldsymbol{\theta}}]=\mathbb E[\xi_b\mid{\boldsymbol{\theta}}]
\mathbb E[\xi_{b'}^\top\mid{\boldsymbol{\theta}}]=0$. 
Thus $\operatorname{Cov}(\hat g_B\mid{\boldsymbol{\theta}})
=
\frac{B}{B^2}
\operatorname{Cov}(g({\boldsymbol{\theta}};\mathbf{z})\mid{\boldsymbol{\theta}})
=
\frac1B
\operatorname{Cov}(g({\boldsymbol{\theta}};\mathbf{z})\mid{\boldsymbol{\theta}})$, 
which proves~\eqref{eq:minibatch-cov}.

(iii) Taking the trace of the covariance identity yields 
$$\mathbb E \left[
 \|
\hat g_B-\nabla\mathcal L({\boldsymbol{\theta}})
 \|_2^2
\middle|{\boldsymbol{\theta}}
\right]
=
\frac1B
\mathbb E \left[
\left\|
g({\boldsymbol{\theta}};\mathbf{z})-\nabla\mathcal L({\boldsymbol{\theta}})
\right\|_2^2
\middle|{\boldsymbol{\theta}}
\right],$$ 
which is~\eqref{eq:minibatch-variance}.
Finally, for any deterministic $v$,
\[
\operatorname{Var}(v^\top\hat g_B\mid{\boldsymbol{\theta}})
=
v^\top
\operatorname{Cov}(\hat g_B\mid{\boldsymbol{\theta}})
v
=
\frac1B
v^\top
\operatorname{Cov}(g({\boldsymbol{\theta}};\mathbf{z})\mid{\boldsymbol{\theta}})
v,
\]
giving the directional variance identity.
\end{proof}

\subsection{Invariant parameter groups}\label{app:effective}

\begin{proposition}[Invariant parameter groups]
\label{prop:invariant-groups}
Assume that the context distribution $\mathcal D_n$ is independent of ${\boldsymbol{\theta}}$ and satisfies the boundedness condition in Assumption \ref{assume} (i). 

Let ${\boldsymbol{\theta}}_{\mathrm{out}}$ collect all entries of
${\boldsymbol{\theta}}=(\mathbf{P}_1,\mathbf{V}_1,\mathbf{P}_2,\mathbf{V}_2)$ outside the following $16$ blocks: $$
 \mathbf{P}_{1,\Psi\Tset},
 \mathbf{V}_{1,\Iset\Psi}, \mathbf{V}_{1,WW},
 \mathbf{P}_{2,\Iset W},\mathbf{V}_{2,W_{\mathrm{out}}\Iset}.$$ 
Under the update ${\boldsymbol{\theta}}^{k+1}={\boldsymbol{\theta}}^k-\eta_k\nabla\widehat{\mathcal L}_B
({\boldsymbol{\theta}}^k)$, if
${\boldsymbol{\theta}}_{\mathrm{out}}^0=0$, then
${\boldsymbol{\theta}}_{\mathrm{out}}^k=0$ for all $k>0$.
\end{proposition}

\begin{proof}
Let $\Ssixteen=\{{\boldsymbol{\theta}}:{\boldsymbol{\theta}}_{\mathrm{out}}=0\}$. 
We will prove the pointwise identity
\begin{equation}\label{eq:p3-pointwise-goal}
 \nabla_{{\boldsymbol{\theta}}_{\mathrm{out}}}\ell({\boldsymbol{\theta}};\mathbf{z})=0,
 \qquad\text{for every context }\mathbf{z}\text{ and every }{\boldsymbol{\theta}}\in\Ssixteen,
\end{equation}
where $\ell({\boldsymbol{\theta}};\mathbf{z})=\frac12\|f_{\boldsymbol{\theta}}(\mathbf{z})-y_\tau(\mathbf{z})\|_2^2$, $f_{\boldsymbol{\theta}}(\mathbf{z})=\operatorname{TF}_{\theta} (\mathbf{H}(\mathbf{z}))$, and $y_\tau(\mathbf{z})={\mathbf{w}}_{\mathrm{soft}}^{(\tau)}(\mathbf{z})$.

Here the context and teacher target are fixed when differentiating with
respect to ${\boldsymbol{\theta}}$. Throughout the proof, derivatives are taken in the
full space of the four matrices. The condition ${\boldsymbol{\theta}}\in\Ssixteen$ is
imposed only when evaluating these derivatives. 

Fix $\mathbf{z}$ and ${\boldsymbol{\theta}}\in\Ssixteen$. The input tokens $t_i$, $c_{i,a}$, and $p$ are supported in $\Tset$, $\Psi$, and
$W$, respectively. 
The first-layer mask allows $t_i$ to attend only to
$\{c_{i,a}:a\in\mathcal A\}$. Write
\begin{equation}\label{eq:p3-attention}
 u_{i,a}=c_{i,a}^{\top}\mathbf{P}_1t_i,\qquad
 \pi_{i,a}=\frac{e^{u_{i,a}}}{\sum_{b\in\mathcal A}e^{u_{i,b}}},
 \qquad
 \bar c_i=\sum_{a\in\mathcal A}\pi_{i,a}c_{i,a}.
\end{equation}
The probabilities $\pi_{i,a}$ are the student's current attention
weights; no teacher parameter value is imposed on $\mathbf{P}_1$.
Since each $c_{i,a}$ is supported in $\Psi$, so is $\bar c_i$.
Every candidate query and the parameter query have exactly one allowed
key, namely $p$. Their attention weights are identically one, for
arbitrary full matrices $\mathbf{P}_1$ and $\mathbf{V}_1$.
Consequently, the first-layer tokens satisfy
\begin{equation}\label{eq:p3-first-layer}
 h_i=t_i+\mathbf{V}_1\bar c_i,\qquad
 \widehat c_{i,a}=c_{i,a}+\mathbf{V}_1p,\qquad
 q=p+\mathbf{V}_1p.
\end{equation}

On $\Ssixteen$, the only possibly nonzero blocks of $\mathbf{V}_1$ are
$\mathbf{V}_{1,\Iset\Psi}$ and $\mathbf{V}_{1,WW}$. Therefore $\supp(\mathbf{V}_1\bar c_i)\subseteq\Iset$, 
 $\supp(\mathbf{V}_1p)\subseteq W$. 
Because $\Tset\subseteq\Iset$ and the coordinate sets
$\Iset,\Psi,W$ are pairwise disjoint, \eqref{eq:p3-first-layer} gives
\begin{equation}\label{eq:p3-supports}
 \supp(h_i)\subseteq\Iset,\qquad
 \supp(\widehat c_{i,a})\subseteq\Psi\cup W,\qquad
 \supp(q)\subseteq W.
\end{equation}
Neither $q=p$ nor $\widehat c_{i,a}=c_{i,a}$ is required. In particular,
$\mathbf{V}_{1,WW}$ can be any $(d+1)\times(d+1)$ matrix.

Let $\Pi=[0_{d\times(4d+2)},I_d]$, $\widetilde{\mathbf{V}}_2=\Pi {\mathbf{V}}_2$.
Write $h_1,\ldots,h_N$ for all first-layer tokens, in their original
order, with $h_N=q$. For arbitrary full matrices, the second-layer
readout is
\begin{equation}\label{eq:p3-readout}
 f_{\boldsymbol{\theta}}(\mathbf{z})=\Pi q+\frac1n\sum_{j=1}^N(\widetilde {\mathbf{V}}_2h_j)(h_j^\top \mathbf{P}_2q).
\end{equation}
To compute its full differential, define
$v_j=\widetilde {\mathbf{V}}_2h_j$ and $\sigma_j=h_j^\top \mathbf{P}_2q$.
The product rule gives
\begin{align}
 \dif (v_j\sigma_j) =\sigma_j\bigl(\Pi\,\dif {\mathbf{V}}_2\,h_j+\widetilde {\mathbf{V}}_2\,\dif h_j\bigr) +v_j\bigl((\dif h_j)^\top \mathbf{P}_2q+h_j^\top(\dif \mathbf{P}_2)q+h_j^\top \mathbf{P}_2\,\dif q\bigr).
 \label{eq:p3-full-product-rule}
\end{align}
We now evaluate this differential at ${\boldsymbol{\theta}}\in\Ssixteen$.
Since ${\mathbf{V}}_2$ has nonzero entries only in
$(W_{\mathrm{out}},\Iset)$, $\widetilde {\mathbf{V}}_2$ has nonzero columns only in $\Iset$. Similarly, $\mathbf{P}_2$ has nonzero entries only in $(\Iset,W)$, and
$q$ is supported in $W$, so $\supp(\mathbf{P}_2q)\subseteq\Iset$.
For every candidate token and for the parameter token considered as a
key/value token, \eqref{eq:p3-supports} therefore implies
\begin{equation}\label{eq:p3-double-zero}
 v_j=\widetilde {\mathbf{V}}_2h_j=0_d,
 \qquad
 \sigma_j=h_j^\top \mathbf{P}_2q=0.
\end{equation}
Substituting \eqref{eq:p3-double-zero} into
\eqref{eq:p3-full-product-rule}, every term vanishes:
$$
 \dif (v_j\sigma_j)=0
 \quad\text{for every full-matrix perturbation }\dif {\boldsymbol{\theta}}.
$$
For a transition token, write $
 a_i=\widetilde {\mathbf{V}}_2h_i$, $s_i=h_i^\top \mathbf{P}_2q$, and $e=f_{\boldsymbol{\theta}}(\mathbf{z})-y_\tau(\mathbf{z})$.
The full differential evaluated on $\Ssixteen$ is
\begin{align}
 \dif f_{\boldsymbol{\theta}}(\mathbf{z})
 &=\Pi\,\dif q+\frac1n\sum_{i=0}^{n-1} s_i\bigl(\Pi\,\dif {\mathbf{V}}_2\,h_i+\widetilde {\mathbf{V}}_2\,\dif h_i\bigr)
       \nonumber\\
 &\quad+\frac1n\sum_{i=0}^{n-1}a_i
       \bigl((\dif h_i)^\top \mathbf{P}_2q+h_i^\top(\dif \mathbf{P}_2)q+h_i^\top \mathbf{P}_2\,\dif q\bigr).
 \label{eq:p3-full-differential}
\end{align}
The sample-loss differential is $\dif\ell=e^\top \dif f_{\boldsymbol{\theta}}(\mathbf{z})$.

\textbf{Gradient of ${\mathbf{P}}_2$ and ${\mathbf{V}}_2$.} Holding the first-layer tokens fixed, the terms involving 
$\dif \mathbf{P}_2$ in
\eqref{eq:p3-full-differential} give
$$
\dif_{\mathbf{P}_2}\ell=\frac1n\sum_{i=0}^{n-1}(e^\top a_i)h_i^\top(\dif \mathbf{P}_2)q.
$$ 
Since $\langle A,B\rangle_F=\operatorname{tr}(A^\top B)$ and $h_i^\top(\dif \mathbf{P}_2)q=\langle h_iq^\top,\dif \mathbf{P}_2\rangle_F$, we obtain $\nabla_{\mathbf{P}_2}\ell({\boldsymbol{\theta}};\mathbf{z}) =\frac1n\sum_{i=0}^{n-1}(e^\top a_i)h_iq^\top$. 
Each $h_i$ is supported in $\Iset$ and $q$ is supported in $W$.
Thus $h_iq^\top$ has nonzero entries only in rows $\Iset$ and columns $W$, and
$\supp(\nabla_{\mathbf{P}_2}\ell)\subseteq(\Iset,W)$.
Entrywise, for a row coordinate $r$ and column coordinate $s$, $(\nabla_{\mathbf{P}_2}\ell)_{rs} =\frac1n\sum_i(e^\top a_i)(h_i)_r q_s=0$ if $r\notin\Iset\text{ or }s\notin W$. 

Likewise, the terms involving $\dif{\mathbf{V}}_2$ in
\eqref{eq:p3-full-differential} give
$$
 \dif_{{\mathbf{V}}_2}\ell
 =\frac1n\sum_{i=0}^{n-1}s_i e^\top\Pi(\dif{\mathbf{V}}_2)h_i
 =\left\langle
     \frac1n\sum_{i=0}^{n-1}s_i\Pi^\top e\,h_i^\top,\dif{\mathbf{V}}_2
   \right\rangle_F.
$$
Hence $\nabla_{{\mathbf{V}}_2}\ell({\boldsymbol{\theta}};\mathbf{z})
 =\frac1n\sum_{i=0}^{n-1}s_i\Pi^\top e\,h_i^\top$. 
The vector $\Pi^\top e$ is supported in $W_{\mathrm{out}}$, while
$h_i$ is supported in $\Iset$. Therefore $\supp(\nabla_{{\mathbf{V}}_2}\ell)\subseteq(W_{\mathrm{out}},\Iset)$. 
Every row outside $W_{\mathrm{out}}$, including the first
row of the $W$ block, has zero gradient; within $W_{\mathrm{out}}$,
the columns in $\Psi$ and $W$ also have zero gradient.

To propagate the loss gradient through the first layer, collect the
$\dif h_i$ and $\dif q$ terms in \eqref{eq:p3-full-differential}.
For each transition token, 
$$\dif_{h_i}\ell
 =\frac1n [s_i e^\top\widetilde {\mathbf{V}}_2\,\dif h_i
             +(e^\top a_i)(\dif h_i)^\top \mathbf{P}_2q ]
 =:\lambda_i^\top \dif h_i,$$ 
Both $\widetilde {\mathbf{V}}_2^\top e$ and $\mathbf{P}_2q$ are supported in $\Iset$, so
$\supp(\lambda_i)\subseteq\Iset$. 
For the parameter token, its residual contribution is
$e^\top\Pi\,\dif q$, and its query contribution is
$n^{-1}\sum_i(e^\top a_i)h_i^\top \mathbf{P}_2\,\dif q$. Thus
\begin{equation*} 
\dif_q\ell=[\Pi^\top e+
        \frac1n\sum_{i=0}^{n-1}(e^\top a_i)\mathbf{P}_2^\top h_i]^\top \dif q=:\lambda_q^\top \dif q.
\end{equation*}
There is no additional key/value contribution from $h_N=q$, because
its summand has zero differential by \eqref{eq:p3-double-zero}.
Since $\mathbf{P}_2$ is supported in $(\Iset,W)$, its transpose maps $h_i$ into
$W$. Also $\Pi^\top e$ is supported in $W_{\mathrm{out}}\subseteq W$.
Consequently, $\supp(\lambda_q)\subseteq W$.
Each candidate token appears in the second layer only as a key/value
token in its own summand. That summand has zero differential, so
\begin{equation}\label{eq:p3-candidate-adjoint}
 \nabla_{\widehat c_{i,a}}\ell=0.
\end{equation}
Here $q$ and the candidate tokens are treated as separate intermediate
variables; their shared dependence on ${\mathbf{V}}_1p$ is included in the next step.

\textbf{Gradient of ${\mathbf{V}}_1$.}
When varying ${\mathbf{V}}_1$ and keeping $\mathbf{P}_1$ fixed, the attention weights and
$\bar c_i$ are fixed. Differentiating \eqref{eq:p3-first-layer} yields 
 $\dif h_i=(\dif {\mathbf{V}}_1)\bar c_i$, $\dif \widehat c_{i,a}=(\dif {\mathbf{V}}_1)p$, $\dif q=(\dif {\mathbf{V}}_1)p$. By the chain rule, including all three token types,
\begin{align*}
 \dif_{{\mathbf{V}}_1}\ell &=\sum_{i=0}^{n-1}\lambda_i^\top(\dif {\mathbf{V}}_1)\bar c_i
   +\sum_{i,a}(\nabla_{\widehat c_{i,a}}\ell)^\top(\dif {\mathbf{V}}_1)p
   +\lambda_q^\top(\dif {\mathbf{V}}_1)p\\
   &=\sum_{i=0}^{n-1}\lambda_i^\top(\dif {\mathbf{V}}_1)\bar c_i
   +\lambda_q^\top(\dif {\mathbf{V}}_1)p,
\end{align*}
where the candidate terms vanish by \eqref{eq:p3-candidate-adjoint}.
It follows that 
 $\nabla_{{\mathbf{V}}_1}\ell({\boldsymbol{\theta}};\mathbf{z})
 =\sum_{i=0}^{n-1}\lambda_i\bar c_i^\top+\lambda_qp^\top$. 
For the first sum, $\lambda_i$ is supported in $\Iset$ and $\bar c_i$
in $\Psi$, so its support is contained in $(\Iset,\Psi)$.
For the second term, both $\lambda_q$ and $p$ are supported in $W$,
so its support is contained in $(W,W)$. Hence
\begin{equation*} 
 \supp(\nabla_{{\mathbf{V}}_1}\ell)
 \subseteq(\Iset,\Psi)\cup(W,W).
\end{equation*}
Equivalently, the only possibly nonzero gradient blocks of ${\mathbf{V}}_1$ are
$(\Phi,\Psi)$, $(G,\Psi)$, $(R,\Psi)$, $(U,\Psi)$, $(W,W)$. 
Every entry outside these five blocks therefore has zero sample gradient.

\textbf{Gradient of $\mathbf{P}_1$ (including the softmax derivative).}
Only transition queries have nonconstant first-layer attention weights.
For a candidate query or the parameter query, the unique allowed-key
weight is $e^u/e^u=1$, whose derivative with respect to every entry of
$\mathbf{P}_1$ is zero.

For a transition query, differentiating \eqref{eq:p3-attention} gives
\begin{align*}
&\dif \pi_{i,a} =\pi_{i,a} (\dif u_{i,a} -\sum_{b\in\mathcal A}\pi_{i,b}\,\dif u_{i,b} ),\\
&\dif \bar c_i =\sum_{a\in\mathcal A}c_{i,a}\,\dif \pi_{i,a}=\sum_{a\in\mathcal A}\pi_{i,a}c_{i,a}\,\dif u_{i,a} -\bar c_i\sum_{b\in\mathcal A}\pi_{i,b}\,\dif u_{i,b}=\sum_{a\in\mathcal A} \pi_{i,a}(c_{i,a}-\bar c_i)\,\dif u_{i,a}.
\end{align*}
When only $\mathbf{P}_1$ varies, $\dif h_i={\mathbf{V}}_1\,\dif\bar c_i$, while $\dif q=0$ and
$\dif\widehat c_{i,a}=0$. Therefore
\begin{align*}
 \dif_{\mathbf{P}_1}\ell
 =\sum_{i=0}^{n-1}\lambda_i^\top {\mathbf{V}}_1\,\dif\bar c_i=\sum_{i=0}^{n-1}\sum_{a\in\mathcal A}
       \pi_{i,a}\lambda_i^\top {\mathbf{V}}_1(c_{i,a}-\bar c_i)\,\dif u_{i,a}.
\end{align*}
Define the scalar $\kappa_{i,a}:=\pi_{i,a}\lambda_i^\top {\mathbf{V}}_1(c_{i,a}-\bar c_i)$.
Since the input tokens are fixed, $\dif u_{i,a}=c_{i,a}^\top(\dif\mathbf{P}_1)t_i =\langle c_{i,a}t_i^\top,\dif\mathbf{P}_1\rangle_F$. Consequently, $ \nabla_{\mathbf{P}_1}\ell({\boldsymbol{\theta}};\mathbf{z})
 =\sum_{i=0}^{n-1}\sum_{a\in\mathcal A}
\kappa_{i,a}c_{i,a}t_i^\top$. 
Each $c_{i,a}$ is supported in $\Psi$ and each input transition token
$t_i$ is supported in $\Tset=(\Phi,R,U)$. Hence
\begin{equation*}
\supp(\nabla_{\mathbf{P}_1}\ell)\subseteq(\Psi,\Tset).
\end{equation*}
Explicitly, only $\mathbf{P}_{1,\Psi\Phi}$, $\mathbf{P}_{1,\Psi R}$, and
$\mathbf{P}_{1,\Psi U}$ can have nonzero gradients. 
 
\textbf{Population gradients and zero coordinates.}
The calculations above hold for every fixed context $z$. Taking
expectations gives the full-matrix population gradient formulas,
evaluated at any ${\boldsymbol{\theta}}\in\Ssixteen$:
\begin{align*}
\nabla_{\mathbf{P}_1}\mathcal{L}({\boldsymbol{\theta}}) &=\E_\mathbf{z} \left[\sum_{i,a}\kappa_{i,a}c_{i,a}t_i^\top\right], \quad\nabla_{{\mathbf{V}}_1}\mathcal L({\boldsymbol{\theta}}) =\E_\mathbf{z} \left[\sum_i\lambda_i\bar c_i^\top+\lambda_qp^\top\right],\\
 \nabla_{\mathbf{P}_2}\mathcal L({\boldsymbol{\theta}})
&=\frac1n\E_\mathbf{z} \left[\sum_i(e^\top a_i)h_iq^\top\right], \quad    
 \nabla_{{\mathbf{V}}_2}\mathcal L({\boldsymbol{\theta}})
  =\frac1n\E_\mathbf{z}\!\left[\sum_i s_i\Pi^\top e\,h_i^\top\right].
\end{align*}
Their supports are therefore contained in
\begin{equation}\label{eq:p3-all-supports}
 \begin{aligned}
  \supp(\nabla_{\mathbf{P}_1}\mathcal L)&\subseteq(\Psi,\Tset), 
  \quad\supp(\nabla_{{\mathbf{V}}_1}\mathcal L) \subseteq(\Iset,\Psi)\cup(W,W),\\
  \supp(\nabla_{\mathbf{P}_2}\mathcal L)&\subseteq(\Iset,W), \quad
  \supp(\nabla_{{\mathbf{V}}_2}\mathcal L) \subseteq(W_{\mathrm{out}},\Iset).
 \end{aligned}
\end{equation}
These are exactly the $16$ retained blocks. Thus
\begin{equation}\label{eq:p3-normal-gradient-zero}
\nabla_{{\boldsymbol{\theta}}_{\mathrm{out}}}\mathcal L({\boldsymbol{\theta}})=0
 \quad\text{whenever }{\boldsymbol{\theta}}_{\mathrm{out}}=0.
\end{equation}

The boundedness condition in Assumption \ref{assume} (i) justifies differentiation under the expectation: on each compact
parameter set, the fixed-mask softmax and the two attention layers
have uniformly bounded derivatives of every fixed order on these
bounded inputs. The teacher target is bounded as well. Thus the sample
loss and its first two parameter derivatives admit an integrable bound
there, so $\mathcal L$ is $C^2$ and the displayed expectation formulas
are valid. 

It remains to show that \eqref{eq:p3-normal-gradient-zero} holds along
the entire training trajectory, rather than only at initialization.
Write ${\boldsymbol{\theta}}=(x,y)$, where $x$ collects the retained coordinates and
$y={\boldsymbol{\theta}}_{\mathrm{out}}$. Then the full gradient-flow system is
\begin{equation}\label{eq:p3-full-flow}
 \dot x=-\nabla_x\mathcal L(x,y),\qquad
 \dot y=-\nabla_y\mathcal L(x,y).
\end{equation}
By \eqref{eq:p3-normal-gradient-zero}, $\nabla_y\mathcal L(x,0)=0$
for every $x$, not merely for the initial value $x(0)$.
Consider the restricted system $\dot{\bar x}=-\nabla_x\mathcal L(\bar x,0)$, 
 $\bar x(0)=x(0)$, $\bar y(t)\equiv0$. 
It satisfies 
 $\dot{\bar y}(t)=0=-\nabla_y\mathcal L(\bar x(t),0)$, 
so $(\bar x(t),\bar y(t))$ solves the full system
\eqref{eq:p3-full-flow} with the same initial condition as $(x(t),y(t))$.
Since $\mathcal L$ is $C^2$, its gradient is locally Lipschitz, and
the initial-value problem has a unique local solution. The full solution
therefore coincides with the restricted solution. Repeating this
uniqueness argument along the trajectory gives $y(t)=0$
throughout its maximal interval of existence.

For completeness, this interval contains every $t\ge0$. Along a
gradient-flow solution, using the product Frobenius norm, 
 $\frac{d}{dt}\mathcal L({\boldsymbol{\theta}}(t))
 =\langle\nabla\mathcal L({\boldsymbol{\theta}}(t)),\dot{\boldsymbol{\theta}}(t)\rangle
 =-\|\nabla\mathcal L({\boldsymbol{\theta}}(t))\|^2
 =-\|\dot{\boldsymbol{\theta}}(t)\|^2$.
Integrating and using $\mathcal L\ge0$ yields $\int_0^t\|\dot{\boldsymbol{\theta}}(s)\|^2\,ds
 =\mathcal L({\boldsymbol{\theta}}(0))-\mathcal L({\boldsymbol{\theta}}(t))
 \le\mathcal L({\boldsymbol{\theta}}(0))$.
By the triangle inequality and Cauchy--Schwarz,
\begin{align*}
 \|{\boldsymbol{\theta}}(t)-{\boldsymbol{\theta}}(0)\|\le\int_0^t\|\dot{\boldsymbol{\theta}}(s)\|\,ds\le\sqrt t\left(\int_0^t\|\dot{\boldsymbol{\theta}}(s)\|^2\,ds\right)^{1/2}
 \le\sqrt{t\,\mathcal L({\boldsymbol{\theta}}(0))}.
\end{align*}
If the maximal existence time were finite, this bound would keep the
trajectory inside a compact subset of the full parameter space up to
that time. The locally Lipschitz vector field then permits continuation,
contradicting maximality. Thus the solution is global, and
${\boldsymbol{\theta}}_{\mathrm{out}}(t)=0$ for all $t\ge0$.

For any batch
$\mathbf{z}_1,\ldots,\mathbf{z}_B$, at every ${\boldsymbol{\theta}}\in\Ssixteen$, $\nabla_{{\boldsymbol{\theta}}_{\mathrm{out}}}\widehat{\mathcal L}_B({\boldsymbol{\theta}})
 =\frac1B\sum_{b=1}^B
       \nabla_{{\boldsymbol{\theta}}_{\mathrm{out}}}\ell({\boldsymbol{\theta}};\mathbf{z}_b)=0$.
Hence an update ${\boldsymbol{\theta}}^{k+1}={\boldsymbol{\theta}}^k-\eta_k\nabla\widehat{\mathcal L}_B
({\boldsymbol{\theta}}^k)$ gives
${\boldsymbol{\theta}}_{\mathrm{out}}^{k+1}
 ={\boldsymbol{\theta}}_{\mathrm{out}}^k-\eta_k\,0=0$ whenever  ${\boldsymbol{\theta}}_{\mathrm{out}}^k=0$. 
Induction proves invariance at every discrete iteration, without
projecting the gradients or explicitly freezing these coordinates.
\end{proof}
\subsection{Support parameter blocks}
Denote the six blocks by 
\begin{align*}
\mathbf{A}=(\mathbf{P}_1)_{\Psi,U},\quad &\mathbf{L}=({\mathbf{V}}_1)_{G,\Psi}, \quad\mathbf{B}_-=(\mathbf{P}_2)_{\Phi,W},\\
\mathbf{B}_+=(\mathbf{P}_2)_{G,W},\quad &\mathbf{B}_r=(\mathbf{P}_2)_{R,W},\quad \mathbf{C}=({\mathbf{V}}_2)_{W_{\rm out},\Phi}.
\end{align*}
Let $\mathcal S_6=\{\mathbf{A},\mathbf{L},\mathbf{B}_-,\mathbf{B}_+,\mathbf{B}_r,\mathbf{C}\}$. For
$\mathcal S\subseteq\mathcal S_6$, let $\Theta(\mathcal S)$ be the
parameter space in which the blocks in $\mathcal S$ are unrestricted
and every other entry of the four matrices is zero. Set
$\Theta_{\mathrm{sup}}=\Theta(\mathcal S_6)$.
Define
\begin{align}
F_\tau(\mathbf{H})=\mathbf{w}+\frac{\alpha}{n}\sum_{i=1}^n\phi_i
 \bigl(r_i+\gamma(\mu_i^\tau)^\top \mathbf{w}-\phi_i^\top \mathbf{w}\bigr).
 \label{eq:target}
\end{align}
where $\mu_i(\mathbf{A})=\sum_{a=1}^m \frac{\exp((\phi_{i,a}^{+})^\top\mathbf{A}\mathbf{w})}
{\sum_{b=1}^m\exp((\phi_{i,b}^{+})^\top\mathbf{A}\mathbf{w})}\phi^+_{i,a}$, and $\mu_i^\tau=\mu_i(\tau^{-1}I_d)$.

The following Proposition shows that the six-block support realizes the target update $F_\tau(\mathbf{H})$ exactly, whereas none of its proper subsets does so on all original prompts.
\begin{proposition}[Sufficiency and inclusion-minimality of the six-block support]
\label{prop:six-block-minimal-support}
Let $\mathcal H$ denote the class of prompts of the form \eqref{eq:Hzero}, $d,n\ge1$, $m\ge2$, $\alpha>0$, $0<\gamma<1$, and
$0<\tau<\infty$. Then

(i) There exists ${\boldsymbol{\theta}}\in\Theta_{\mathrm{sup}}$, such that for all $\mathbf{H}\in\mathcal H$, $\TF_{\theta}(\mathbf{H})=F_\tau(\mathbf{H})$.

(ii) For all $\mathcal S\subsetneq\mathcal S_6$, and ${\boldsymbol{\theta}}\in\Theta(\mathcal S)$, there exists $\mathbf{H}\in\mathcal H$, such that $\TF_{\theta}(\mathbf{H})\ne F_\tau(\mathbf{H})$.
\end{proposition}

\begin{proof}
(i) For arbitrary parameters in $\Theta_{\mathrm{sup}}$, the first-layer transition token is $$t_i^{(1)}=\begin{bmatrix}\phi_i;\mathbf{L}\mu_i(\mathbf{A});r_i;0_d;\mathbf{w};0_{d+1}\end{bmatrix}.$$
Indeed, the mask restricts transition $i$ to its own successor candidates,
whose scores are $\psi_{ia}^{\top}\mathbf{A}\mathbf{w}$. The candidates and final query
are unchanged, since each attends only to the final query $p$ and
${\mathbf{V}}_1p=0$. In the second layer, $\mathbf{P}_2p$ is supported on $\Phi,G,R$.
Consequently, the candidate and final-query scores against $p$ are zero;
only the transition tokens contribute. Their scores are 
 $\sigma_i=\phi_i^\top \mathbf{B}_-\tilde{\mathbf{w}}+ \mu_i(\mathbf{A})^\top \mathbf{L}^\top \mathbf{B}_+\tilde{\mathbf{w}}+r_i\mathbf{B}_r\tilde{\mathbf{w}}$, 
and their value vectors in the output coordinates are $\mathbf{C}\phi_i$.
Including the residual readout gives
\begin{equation}
 \TF_{\boldsymbol{\theta}}(\mathbf{H})-\mathbf{w}
 =\frac1n\sum_{i=1}^n \mathbf{C}\phi_i
 \left(\phi_i^\top \mathbf{B}_-\tilde{\mathbf{w}}+
       \mu_i(\mathbf{A})^\top \mathbf{L}^\top \mathbf{B}_+\tilde{\mathbf{w}}+r_i\mathbf{B}_r\tilde{\mathbf{w}}\right).
 \label{eq:forward}
\end{equation}
For any $c\ne0$ and any $\mathbf{L}\in\GL(d)$, choose 
 $\mathbf{A}=\tau^{-1}I_d$, $\mathbf{B}_-=-c\mathbf{J}$, $\mathbf{B}_r=ce_0^\top$, $\mathbf{B}_+=\gamma c\mathbf{L}^{-\top}\mathbf{J}$, $\mathbf{C}=\frac\alpha c I_d$. 
Then $\mathbf{J}\tilde{\mathbf{w}}=\mathbf{w}$, $e_0^\top \tilde{\mathbf{w}}=1$, and $\mathbf{L}^\top \mathbf{B}_+=\gamma c\mathbf{J}$.
Thus $\sigma_i=c\bigl(r_i+\gamma(\mu_i^\tau)^\top \mathbf{w}-\phi_i^\top \mathbf{w}\bigr)$,  $\mathbf{C}\phi_i=\frac\alpha c\phi_i$. 
Equation~\eqref{eq:forward} becomes~\eqref{eq:target}.

(ii) Fix a unit vector $e\in\R^d$. In every counterexample below, take
$\phi_1=e$ and $\phi_i=0$ for $i\ge2$. Accordingly only transition 1
contributes to either update, regardless of the fixed context length
$n$. These inputs can be realized by a trajectory that moves from an
initial state to a second state and subsequently selects a zero-feature
action there. The successor feature lists used below are either all
zero or $(e,0,\ldots,0)$, so such an action exists because $m\ge2$.
Set the later rewards to zero.  

If $\mathbf{C}=0$, \eqref{eq:forward} gives $\TF_{\boldsymbol{\theta}}(\mathbf{H})=\mathbf{w}$
for every input. Choose $\mathbf{w}=0$, $r_1=1$, and all first-transition
successor features zero. The target increment is $\alpha e/n\ne0$,
which cannot be reproduced.

If $\mathbf{B}_r=0$, the output in~\eqref{eq:forward} does not depend on any
reward. Fix $\mathbf{w}=0$ and the first-transition successor features at zero,
and compare otherwise identical inputs with $r_1=0$ and $r_1=1$.
Their target outputs differ by $\alpha e/n$, while the model outputs
coincide. No fixed choice of the other blocks can fit both inputs;
Set $\mathbf{w}=e$, $r_1=0$, and all first-transition successor features to zero.
Then $\mu_1(\mathbf{A})=0$ for every $\mathbf{A}$. 

If $\mathbf{B}_-=0$, the model increment is zero, but the target increment is $-\alpha e/n\ne0$.

If $\mathbf{L}=0$ or $\mathbf{B}_+=0$, then $\mathbf{L}^\top \mathbf{B}_+=0$, so the output is
independent of the successor features. Fix $\mathbf{w}=e$ and $r_1=0$,
and compare successor lists $(0,\ldots,0)$ and $(e,0,\ldots,0)$.
The target outputs differ by 
 $\frac{\alpha\gamma}{n}
 \frac{\exp(1/\tau)}{\exp(1/\tau)+m-1}\,e\ne0$. 
Again the model outputs coincide, so they cannot both equal the targets.

If $\mathbf{A}=0$, attention weights are uniform. Holding the features and
rewards fixed, $\mu_i(0)$ is independent of $\mathbf{w}$, and the increment
in~\eqref{eq:forward} is an affine function of $\mathbf{w}$.
Choose the first-transition successor list $(e,0,\ldots,0)$,
$r_1=0$, and $\mathbf{w}=te$, with $t$ varying in an interval about zero.
The target increment is $g(t)=\frac\alpha n e\bigl[\gamma t p(t)-t\bigr]$, $p(t)=\frac{\exp(t/\tau)}{\exp(t/\tau)+m-1}$. 
Since $p'(0)=(m-1)/(\tau m^2)$, 
 $g''(0)=\frac{2\alpha\gamma(m-1)}{n\tau m^2}\,e\ne0$. 
Thus the target increment is not affine in $t$, whereas the model
increment is affine. They cannot agree for all these inputs.

Every proper subset of $\mathcal S_6$ omits at least one of these
blocks. The corresponding case applies irrespective of the values of
the remaining blocks.
\end{proof}


\section{Experimental reproducibility and implementation details}\label{sec:Experimental}
We evaluate parameter recovery and cross-task generalization in a structured two-layer Q-learning Transformer trained by output-only supervision. All models use cosine learning-rate decay from the first Adam update.  Batch-size and context-length comparisons further reveal distinct effects on parameter recovery and functional accuracy under fixed task budgets. Frozen models transfer to 512 unseen tabular tasks, achieving returns close to their soft teachers. 
\label{app:reproducibility}
\begin{table}[H]
\caption{Notation. A task is one independently sampled MDP/feature realization;
a context is a valid trajectory window generated in that task. In the
training, one task supplies one context.}
\label{tab:notation}\centering\small
\begin{tabular}{lp{0.74\linewidth}}
\toprule
Symbol & Meaning and default \\
\midrule
$S,m,d$ & State count $S=9$, action count $m=4$, Q-feature dimension $d=36$.\\
$\mathbf{w},{\boldsymbol{\theta}}$ & Q parameter $\mathbf{w}\in\R^d$; distinct Transformer parameter vector ${\boldsymbol{\theta}}$.\\
$\tau$ & Teacher temperature.\\
$\alpha,\gamma,\epsilon,\eta$ & Q-update step $0.2$, discount $0.5$, behavior exploration $0.1$, Adam learning rate.\\
$B,U$ & Independent contexts per optimizer update; number of optimizer updates.\\
$n$ & Transitions within one context, default $20$; not minibatch size $B$.\\
$\Ntask,\Nobs$ & Tasks/contexts seen: $\Ntask=BU$; supplied transitions: $\Nobs=nBU$.\\
$D,N_{\rm tok}$ & Token width $D=5d+2=182$; tokens per prompt $N_{\rm tok}=n(m+1)+1$.\\
\bottomrule
\end{tabular}
\end{table}

\paragraph{Task budget and optimization.}
Each independently sampled task contributes one context. Let $B$ denote the number of contexts per optimizer update, $n$ the number of transitions within a context, and $U_B$ the number of optimizer updates.
At update $k=0,\ldots,U_B-1$, let $m_k=kB$ be the number of tasks already consumed. The learning rate is 
$\eta(m_k)=10^{-6}+\frac{10^{-3}-10^{-6}}{2}
\left[1+\cos (\pi\frac{m_k}{N_{\rm task}} )\right]$. 
All blocks follow this schedule from the first update at $10^{-3}$, without warm-up, restarts, or an additional phase. Its minimum is $10^{-6}$ at the budget boundary; the last actual update occurs at $m_k=N_{\rm task}-B$ and uses a slightly larger value. Adam uses $(\beta_1,\beta_2)=(0.9,0.999)$, $\epsilon_{\rm Adam}=10^{-8}$, no weight decay, and float32 parameters and moments.  

The study uses five paired seeds for each configuration. For experiments on valid parameter blocks, the cosine schedule terminates at \(20.48\times10^6\) tasks. Training continues at a learning rate of \(10^{-6}\) to \(25.60\times10^6\) tasks, then \(10^{-4}\), with the rate lowered to \(10^{-5}\) beyond \(37.76\times10^6\) tasks. The rate is finally switched to \(10^{-6}\), ending at \(46.40\times10^6\) tasks per seed. For experiments on structured parameter blocks, the cosine schedule terminates at \(20.48\times10^6\) tasks.
\subsection{Experiments on Valid Parameter Blocks}\label{ex:Valid}

\paragraph{Evaluating the accuracy of approximate Q updates.} We adopt three metrics, namely Soft error $\text{err}_{\text{soft}}$, Hard error $\text{err}_{\text{hard}}$, and Greedy agreement, to quantify how closely the Transformer proposed in this work approximates the Q-learning update rule. Specifically,  
\begin{equation} \text{err}_{\text{soft}}=\sqrt{\frac{\sum_{\mathbf{z}\in\mathcal Z}\|{\hat{u}(\mathbf{z})-u_\tau(\mathbf{z})}\|_2^2} {\sum_{\mathbf{z}\in\mathcal Z}\|{u_\tau(\mathbf{z})}\|_2^2}},\qquad \text{err}_{\text{hard}}=\sqrt{\frac{\sum_\mathbf{z}\|{\hat{u}(\mathbf{z})-u_{\max}(\mathbf{z})}\|_2^2}{\sum_\mathbf{z}\|{u_{\max}(\mathbf{z})}\|_2^2}}.\label{eq:nrmse} \end{equation}
where $u_\tau(\mathbf{z}) = \frac{\alpha}{n}\sum_{i=0}^{n-1}\phi_i[r_{i+1}+\gamma V_{\tau,i}-\mathbf{w}^\top\phi_i]$,  $u_{\mathrm{max}}(\mathbf{z}) = \frac{\alpha}{n}\sum_{i=0}^{n-1}\phi_i[r_{i+1}+\gamma\max_a q_{i,a}^+-\mathbf{w}^\top\phi_i]$, and $q_{i,a}^+ = \mathbf{w}^\top \phi_{i,a}$, $V_{\tau,i} = \sum_{a}\pi_i q_{i,a}^+$.
The two greedy metrics quantify action agreement at distinct granularities: \(\text{Greedy}_{\rm tr}\) reports the fraction of individual transitions for which the model’s attention-selected action matches the teacher’s greedy action, i.e., the action with the maximum Q-value; \(\text{Greedy}_{\rm ctx}\) is a stricter metric, defined as the fraction of contexts where such a match holds across all 20 transitions within the context. The Teacher bias is obtained by replacing the numerator of \(\text{err}_{\text{hard}}\) with \(\sum_\mathbf{z}\|u_\tau-u_{\max}\|^2\).   Teacher bias is deterministic on the common corpus and is not assigned an artificial seed confidence interval. 

Update errors and greedy agreement are reported as percentages. The identity/zero-update predictor yields a 100\% update error: the residual $\mathbf{w}$ cannot mask an unlearned update. $\pm$ denotes a 95\% Student‑t interval across trained seeds. 

Table \ref{tab:dense} (dense features) demonstrates close reproduction of the soft-Q update by the model, yielding 0.01256\% NRMSE and 99.990\% transition-level greedy agreement; by contrast, the one-hot feature setting (Table \ref{tab:one-hot}) exhibits noticeably weaker performance, attaining 0.07950\% NRMSE and 99.989\% transition-level greedy agreement.

\begin{table}[H]
\centering
\caption{Dense held-out contexts: update errors and greedy agreement}
\label{tab:dense}
\small
\begin{tabular}{rrrrr}
\toprule
Tasks ($10^6$) & Soft error (\%) & Hard error (\%) & Greedy$_{\rm tr}$ (\%) & Greedy$_{\rm ctx}$ (\%) \\
\midrule
0.32 & $9.0157 \pm 0.6508$ & $9.0465 \pm 0.6716$ & $95.314 \pm 0.102$ & $56.587 \pm 0.776$ \\
5.12 & $7.1287 \pm 0.2225$ & $7.1792 \pm 0.2118$ & $96.708 \pm 0.157$ & $67.205 \pm 0.895$ \\
10.24 & $4.4694 \pm 0.6952$ & $4.5328 \pm 0.6750$ & $98.246 \pm 0.087$ & $80.769 \pm 1.045$ \\
20.48 & $0.0657 \pm 0.0002$ & $0.6655 \pm 0.0019$ & $99.920 \pm 0.005$ & $98.840 \pm 0.082$ \\
\midrule
25.60 & $0.0625 \pm 0.0002$ & $0.6633 \pm 0.0020$ & $99.927 \pm 0.003$ & $98.936 \pm 0.035$ \\
\midrule
37.76 & $0.7393 \pm 0.0557$ & $0.9834 \pm 0.0444$ & $99.654 \pm 0.012$ & $96.003 \pm 0.197$ \\
\midrule
45.76 & $0.0846 \pm 0.0107$ & $0.6468 \pm 0.0044$ & $99.971 \pm 0.002$ & $99.636 \pm 0.046$ \\
\midrule
46.40 & $0.0126 \pm 0.0005$ & $0.6394 \pm 0.0008$ & $99.990 \pm 0.004$ & $99.861 \pm 0.045$ \\
\bottomrule
\end{tabular}
\end{table}

\begin{table}[H]
\centering
\caption{One-hot held-out contexts: update errors and greedy agreement }
\label{tab:one-hot}
\small
\begin{tabular}{rrrrr}
\toprule
Tasks ($10^6$) & Soft error (\%) & Hard error (\%) & Greedy$_{\rm tr}$ (\%) & Greedy$_{\rm ctx}$ (\%) \\
\midrule
0.32 & $55.7678 \pm 4.5493$ & $57.4641 \pm 4.8244$ & $95.752 \pm 0.151$ & $66.636 \pm 0.899$ \\
5.12 & $44.0258 \pm 1.9607$ & $45.6216 \pm 1.9971$ & $97.076 \pm 0.085$ & $75.928 \pm 0.833$ \\
10.24 & $27.8956 \pm 4.6804$ & $29.2407 \pm 4.7820$ & $98.406 \pm 0.067$ & $86.035 \pm 0.483$ \\
20.48 & $0.4427 \pm 0.0066$ & $5.1372 \pm 0.0071$ & $99.985 \pm 0.002$ & $99.834 \pm 0.023$ \\
\midrule
25.60 & $0.4181 \pm 0.0073$ & $5.1137 \pm 0.0097$ & $99.987 \pm 0.003$ & $99.875 \pm 0.022$ \\
\midrule
37.76 & $4.5953 \pm 0.3353$ & $6.7774 \pm 0.2200$ & $99.681 \pm 0.012$ & $97.070 \pm 0.077$ \\
\midrule
45.76 & $0.5257 \pm 0.0639$ & $4.8381 \pm 0.0171$ & $99.978 \pm 0.004$ & $99.783 \pm 0.020$ \\
\midrule
46.40 & $0.0795 \pm 0.0037$ & $4.8091 \pm 0.0029$ & $99.989 \pm 0.004$ & $99.890 \pm 0.039$ \\
\bottomrule
\end{tabular}
\end{table}

\paragraph{Emergence of Parameters.}
To distinguish scalar calibration from nonscalar distortion, let $\beta_{att}=\operatorname{tr}(\mathbf{A}_{{\theta}})/d$. Frobenius orthogonality of $\mathbf{A}_{{\theta}}-\beta_{att} I$ and $I$ gives $$E_A^2=\underbrace{(\tau\beta_{att}-1)^2}_{E_{A,\mathrm{scale}}^2}
+\underbrace{\frac{\tau^2}{d}\|\mathbf{A}_{{\theta}}-\beta_{att} I\|_F^2}_{E_{A,\mathrm{shape}}^2}.$$
When $\beta_{att}>0$, $1/\beta_{att}$ summarizes the effective temperature of the scalar projection; it is not an exact temperature when nonscalar components remain. A visible diagonal in a heatmap is therefore insufficient evidence of inverse-temperature recovery. As can be seen in Table \ref{tab:attention}, both residual components decrease substantially toward the end of training, demonstrating the emergence of an action-selection mechanism. Figures \ref{fig:sixteen_blocks} presents heatmaps of the key matrices under $\tau=0.2$, revealing that the training algorithm successfully induces the constructed parameters to emerge.

\begin{table}[H]
\centering
\caption{Attention residual and scale across training. }\label{tab:attention}
\small
\begin{tabular}{rrrr}
\toprule
Tasks ($10^6$) & $E_A$ & $\beta_{\rm att}$ & $E_{A,\rm shape}$ \\
\midrule
0.32 & $0.729464 \pm 0.000924$ & $1.353730 \pm 0.004596$ & $0.017515 \pm 0.000474$ \\
5.12 & $0.435038 \pm 0.001273$ & $2.828350 \pm 0.006350$ & $0.024805 \pm 0.000613$ \\
10.24 & $0.303830 \pm 0.000436$ & $3.483203 \pm 0.002233$ & $0.016895 \pm 0.000379$ \\
20.48 & $0.081884 \pm 0.000192$ & $4.590583 \pm 0.000960$ & $0.000274 \pm 0.000009$ \\
\midrule
25.60 & $0.077982 \pm 0.000189$ & $4.610089 \pm 0.000944$ & $0.000148 \pm 0.000007$ \\
\midrule
37.76 & $0.023377 \pm 0.000188$ & $4.885535 \pm 0.001009$ & $0.004731 \pm 0.000092$ \\
\midrule
45.76 & $0.009809 \pm 0.000093$ & $4.951000 \pm 0.000462$ & $0.000428 \pm 0.000016$ \\
\midrule
46.40 & $0.009781 \pm 0.000091$ & $4.951105 \pm 0.000457$ & $0.000185 \pm 0.000008$ \\
\bottomrule
\end{tabular}
\end{table}

\begin{figure}[t]
\centering\includegraphics[width=1\linewidth]{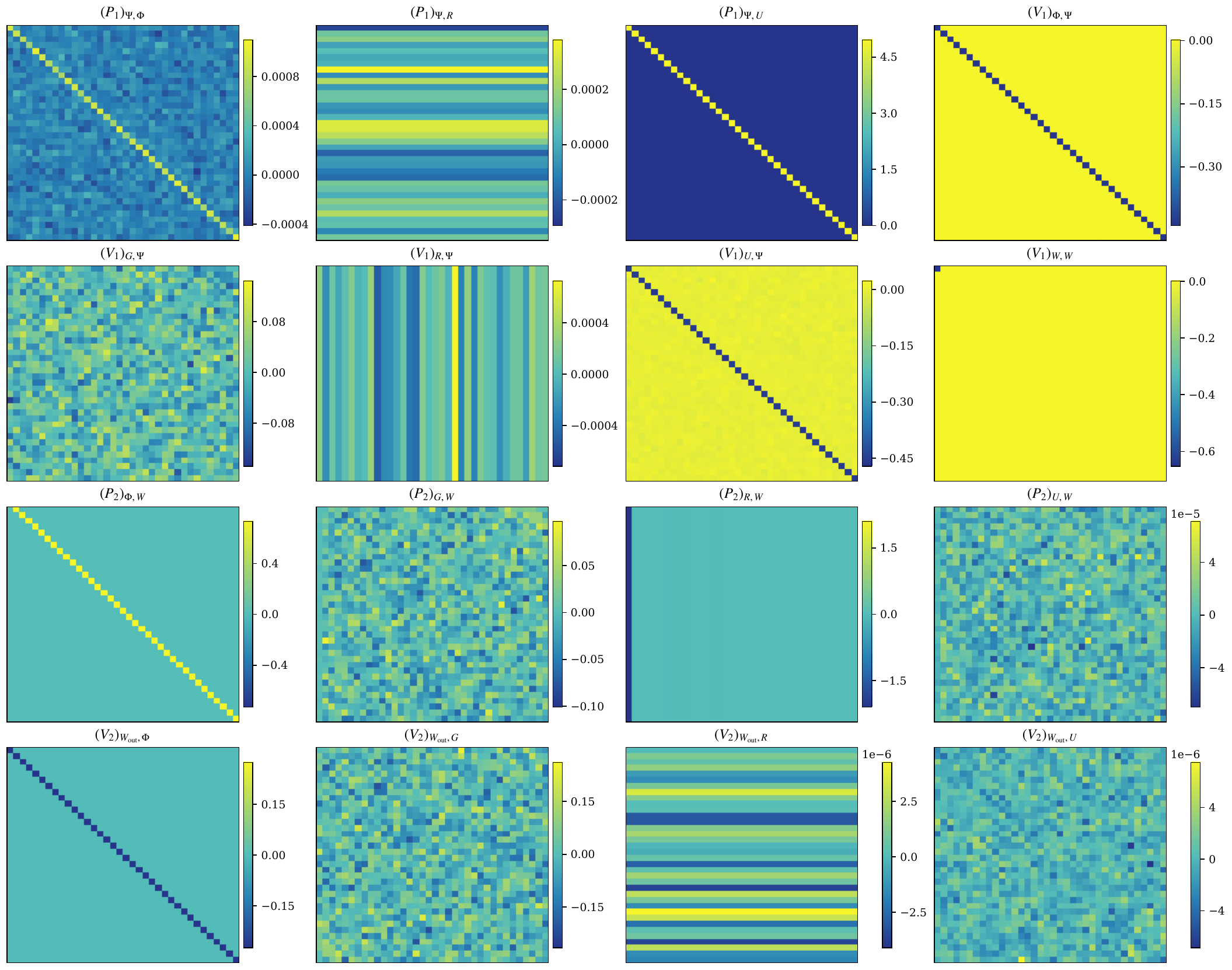}
\caption{Terminal detail of all 16 initialized blocks for seed 0. Each block has its own color range, shared across seeds for that block. This view complements the full matrices; it does not replace structural zeros by cropped or normalized values in the full-matrix figures.}\label{fig:sixteen_blocks}
\end{figure}

\paragraph{Generalization to unseen tabular tasks} 
We evaluate frozen Transformers on a common bank of 512 independent unseen MDPs. Deployment features are one-hot, $\phi(s,a)=e_{4s+a}$, whereas training uses dense features. These experiments assess dense-to-tabular transfer, not direct tabular training.
Each learner recursively updates its own Q parameter $\mathbf{w}$ and generates its own $\epsilon$-greedy actions and successor states. The Transformer parameters ${\boldsymbol{\theta}}$ remain fixed, with no teacher parameters, actions, or states injected into its rollout. The primary comparison uses $\mathbf{w}_0=0$, $n_{\rm test}=20$, and 20,000 environment interactions per learner.

Every temperature comparison includes the Transformer, hard and soft window-Q teachers, Oracle, and Random. The window teachers, or their hard-max counterpart, maintain a fixed $\mathbf{w}$ for each trajectory segment. Standard online Q-learning instead updates after every transition; we report both a constant step size of $0.2$ and a visit-count step size of $1/N_{sa}$. Oracle uses the true transition and reward model to compute an optimal policy and is therefore a privileged upper reference. Random uses a uniform action policy rather than a greedy policy induced by a random Q table.

We evaluate the exact discounted return
$J(\pi)=P_0^\top(I-\gamma P_\pi)^{-1}r_\pi$, 
$r(s,a)=\sum_{s'}P(s'\mid s,a)R(a,s')$,
where $P_\pi$ and $r_\pi$ are the policy-induced transition matrix and expected reward vector. For learned methods and Q-learning baselines, $\pi$ is the current greedy policy; Oracle and Random use their specified policies. Thus, $J$ is a policy-evaluation metric. 

\begin{table}[H]
\centering
\caption{Exact greedy-policy return after 20,000 interactions on 512 unseen MDPs. }\label{tab:returns}
\small
\setlength{\tabcolsep}{5pt}
\begin{tabular}{lrr}
\toprule
Method & Mean $J(\pi)$ & 95\% interval\\
\midrule
Transformer & 0.479225 & $[0.461365, 0.497690]$ \\
Soft window-Q & 0.480136 & $[0.462090, 0.498943]$ \\
Hard window-Q & 0.465381 & $[0.447874, 0.483684]$ \\
Online Q, constant & 0.449844 & $[0.426708, 0.473394]$ \\
Online Q, visit count & 0.522897 & $[0.502016, 0.544255]$ \\
Oracle & 0.529490 & $[0.508685, 0.550915]$ \\
Random & -0.002155 & $[-0.019109, 0.015196]$ \\
\bottomrule
\end{tabular}
\end{table}
Figure \ref{fig:tasksixteen_blocks} deploys frozen dense-feature-trained Transformers on unseen tabular MDPs and measures exact policy returns under a common interaction budget. 
Comparisons include hard and soft window Q-learning, a model-informed
Oracle, and uniform Random.
The Transformer achieves returns close to those of the soft window teacher and a higher sample-mean return than the hard window-Q baseline using the same fixed step size.

\begin{figure}[H]
\centering\includegraphics[width=1\linewidth]{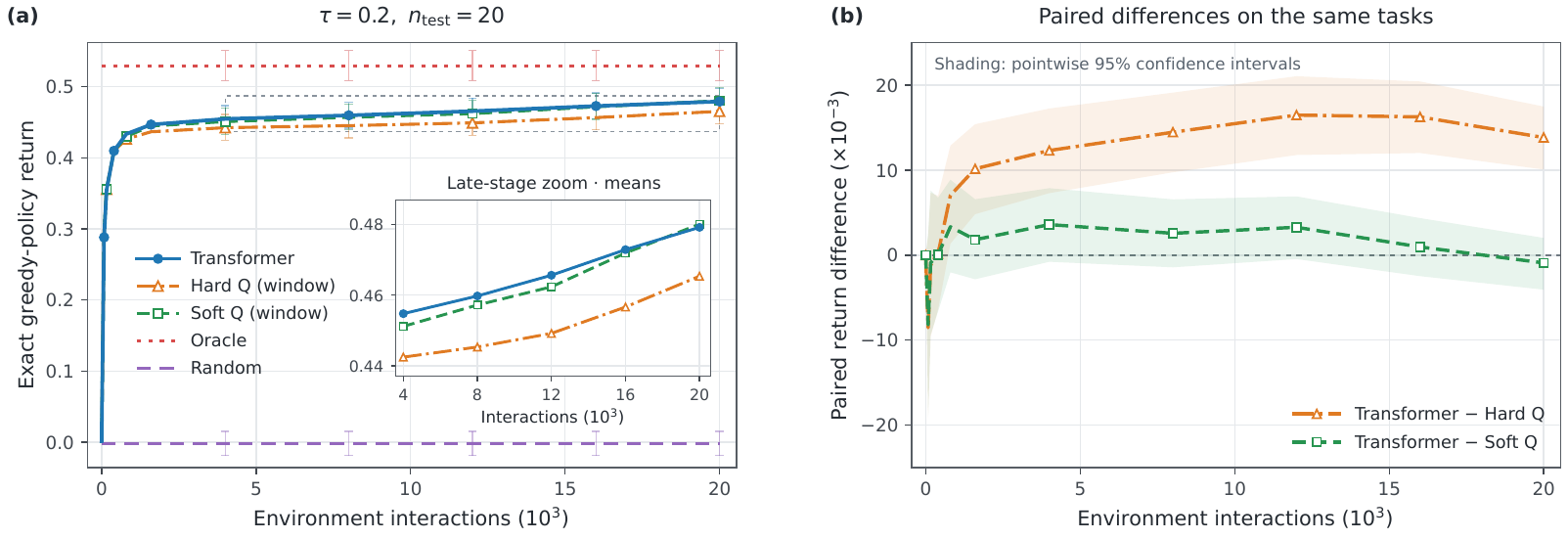}
\caption{Frozen dense-feature-trained Transformers on unseen tabular MDPs and measures exact policy returns under a common interaction budget.}\label{fig:tasksixteen_blocks}
\end{figure}

\subsection{Experiments on Structured Parameter Blocks}\label{app:ApQ}
\paragraph{Evaluating the accuracy of approximate Q updates.} 

Learning error across different \textbf{temperatures} \(\tau\).  Table~\ref{tab:cos-functional} gives an analysis of the impacts of $\tau$. The $\tau=0.1$ model has a soft-update error of $0.07575\%$ despite its appreciable parameter residual; the $\tau=1$ model closely matches its soft teacher but has a large hard-update error, reflecting the finite-temperature teacher discrepancy. Here, Greedy agreement quantifies the fraction of contexts where the argmax action from the attention module matches the optimal maximizing next action of the Q-learning target.

\begin{table}[H]
\centering\small
\setlength{\tabcolsep}{3pt}
\caption{Update accuracy on 8,192 common dense-feature contexts. }
\label{tab:cos-functional}
\begin{tabular}{rrrrr}
\toprule
$\tau$ & Soft error (\%) & Hard error (\%) & Teacher bias (\%) & Greedy (\%) \\
\midrule
1 & $0.00283\pm0.00107$ & $10.03718\pm0.00020$ & 10.0370 & 99.9985 \\
0.8 & $0.00261\pm0.00071$ & $7.114048\pm0.000049$ & 7.1140 & 99.9974 \\
0.5 & $0.00276\pm0.00078$ & $3.21972\pm0.00013$ & 3.2198 & 99.9962 \\
0.2 & $0.00393\pm0.00019$ & $0.63950\pm0.00014$ & 0.6385 & 99.9990 \\
0.1 & $0.07575\pm0.00025$ & $0.25509\pm0.00040$ & 0.1960 & 99.9707 \\
\bottomrule
\end{tabular}
\end{table}

Table \ref{tab:functional} illustrates how \textbf{different batch sizes} $B$ affect the approximation error. The soft error decreases first and then rises with increasing batch size, minimized at $B=64$. Greedy agreement stays near 100\%. The fitted effective temperature $\widehat\tau_{\mathrm{eff}}=1/\beta_{att}$, with $\beta_{att}=\operatorname{tr}(\mathbf{A}_{\boldsymbol{\theta}})/d>0$ acting as a scalar projection diagnostic, remains around 0.2.
\begin{table}[H]
\centering\footnotesize
\caption{Final checkpoints on 8,192 new shared contexts. }
\label{tab:functional}
{\small
\begin{tabular}{rrrrr}
\toprule
$B$ & Soft error (\%) & Hard error (\%) & Greedy ($\%$) & $\widehat\tau_{\rm eff}$ \\
\midrule
1 & $0.0202\pm0.0008$ & $0.6348\pm0.0007$ & 99.994 & 0.200113 \\
2 & $0.0135\pm0.0013$ & $0.6347\pm0.0004$ & 99.995 & 0.200128 \\
4 & $0.0110\pm0.0014$ & $0.6349\pm0.0004$ & 99.995 & 0.200161 \\
8 & $0.0077\pm0.0009$ & $0.6351\pm0.0002$ & 99.999 & 0.200223 \\
16 & $0.0052\pm0.0005$ & $0.6352\pm0.0003$ & 99.999 & 0.200299 \\
32 & $0.0041\pm0.0003$ & $0.6355\pm0.0001$ & 100.000 & 0.200373 \\
64 & $0.0036\pm0.0001$ & $0.6357\pm0.0001$ & 100.000 & 0.200459 \\
128 & $0.0045\pm0.0001$ & $0.6367\pm0.0001$ & 100.000 & 0.200784 \\
256 & $0.0123\pm0.0001$ & $0.6424\pm0.0003$ & 99.998 & 0.202525 \\
\bottomrule
\end{tabular}}

\end{table}

To compare the effects of different optimizations on experimental results, we also evaluated the following methods. 1. Constant Adam: All six blocks retain learning rate $10^{-3}$, with moments $(0.9,0.999)$ and numerical constant $10^{-8}$; 2. Uniform cosine decay: separately labeled optimizer change. After the common 100,000-step prefix, all block learning rates decrease from $10^{-3}$ to $10^{-5}$. For an update indexed by $j\in\{0,\ldots,219999\}$ after the fork, the rate for block $k$ is $\eta_k(j)=10^{-5}+\frac{\eta_{k,0}-10^{-5}}{2}[1+\cos(\frac{\pi j}{219999})]$; 3. Low-rate extension: a larger, adaptively chosen budget. After observing the completed uniform-cosine runs and the aggressive schedule's training/validation overscaling, we continue all five uniform-cosine endpoints to 20.48 million tasks. Every block uses  $\eta(j)=10^{-6}+\frac{10^{-5}-10^{-6}}{2} [1+\cos (\frac{\pi j}{319999} ) ]$, $j=0,\ldots,319999$. Table \ref{tab:functional1} reports  three metrics under $\tau=0.2$ with above learning rate schedules. Each dense-feature experiment uses $B=32$ and $n=20$. 
\begin{table}[H]
\centering\footnotesize
\setlength{\tabcolsep}{4pt}
\caption{Results of Soft error, Hard error, and Greedy agreement with three learning rate schedules.}
\label{tab:functional1}
\begin{tabular}{lrrrr}
\toprule
Condition & $N_{task}$ & Soft error (\%) & Hard error (\%) & Greedy (\%)\\
\midrule
Constant Adam &  10,240,000& $3.242\pm0.732$ & $3.332\pm0.717$ & $98.100\pm0.142$ \\
Uniform cosine & 10,240,000& $0.04028\pm0.00419$ & $0.64309\pm0.00076$ & $99.980\pm0.006$ \\
Low-rate extension & 20,480,000& $0.00415\pm0.00023$ & $0.63428\pm0.00012$ & $100.000\pm0.000$ \\
\bottomrule
\end{tabular}
\end{table}

\paragraph{Emergence of Parameters.} 
Figure \ref{em:emergence} presents the final learned matrices $\mathbf{P}_1$, $\mathbf{V}_1$, $\mathbf{P}_2$, and $\mathbf{V}_2$ obtained after training with $\tau=0.2$. Remarkably, the learned matrices recover the parameter structure predicted in Theorem \ref{Th:Qlearning1}. In particular, the diagonal entries of $\mathbf{A}_\theta$ in $\mathbf{P}_1$ approach $1/\tau$, indicating that the Transformer has learned the internal action-selection computation required by the soft Q-learning update. This provides direct evidence of parameter emergence.

\begin{figure}[H]
  \centering
  \begin{subfigure}{0.24\textwidth}
    \centering
    \includegraphics[width=\linewidth]{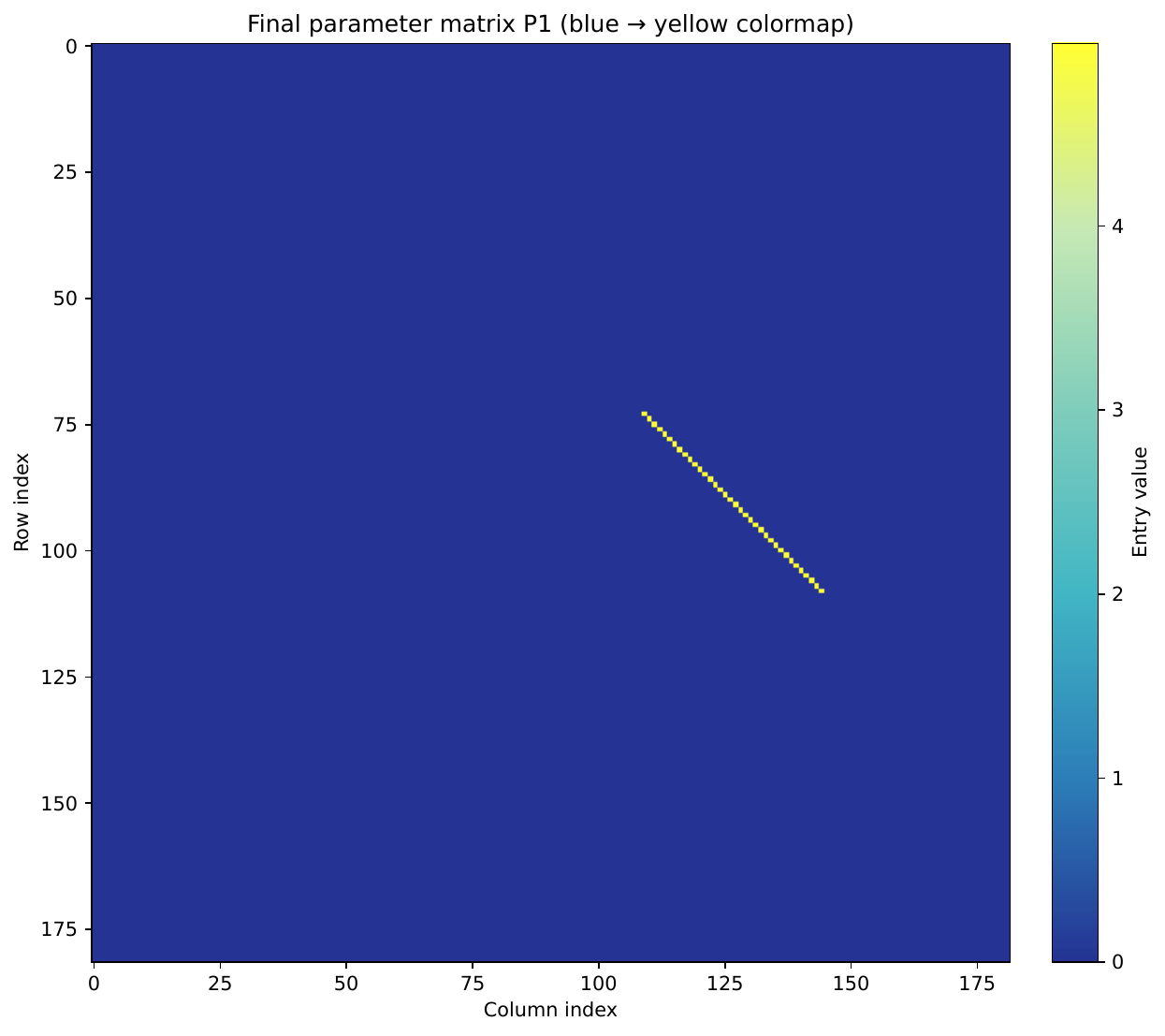}
    \caption{Learned $\mathbf{P}_1$}
  \end{subfigure}
  \hfill
  \begin{subfigure}{0.24\textwidth}
    \centering
    \includegraphics[width=\linewidth]{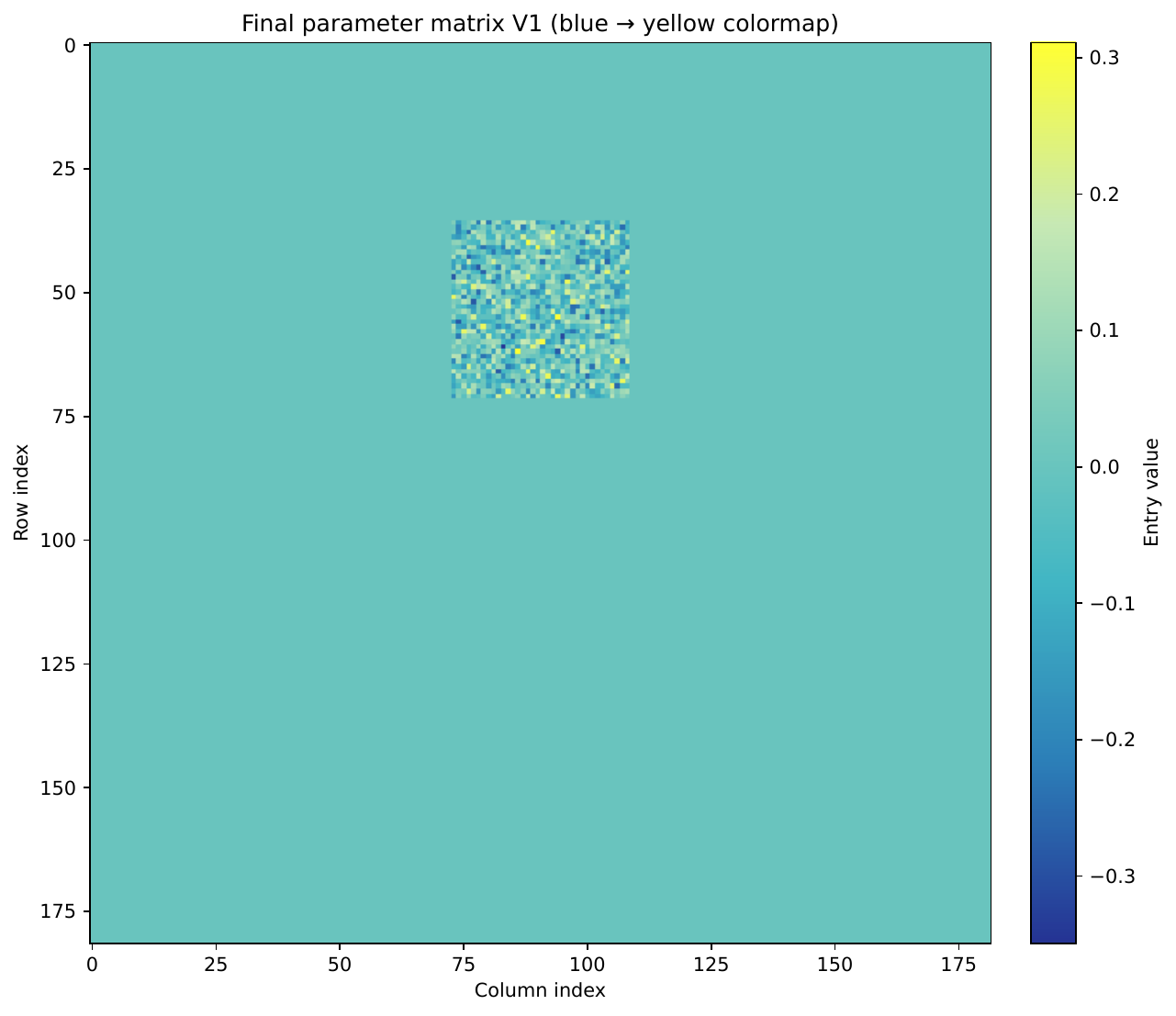}
    \caption{Learned $\mathbf{V}_1$}
  \end{subfigure}
  \hfill
  \begin{subfigure}{0.24\textwidth}
    \centering
    \includegraphics[width=\linewidth]{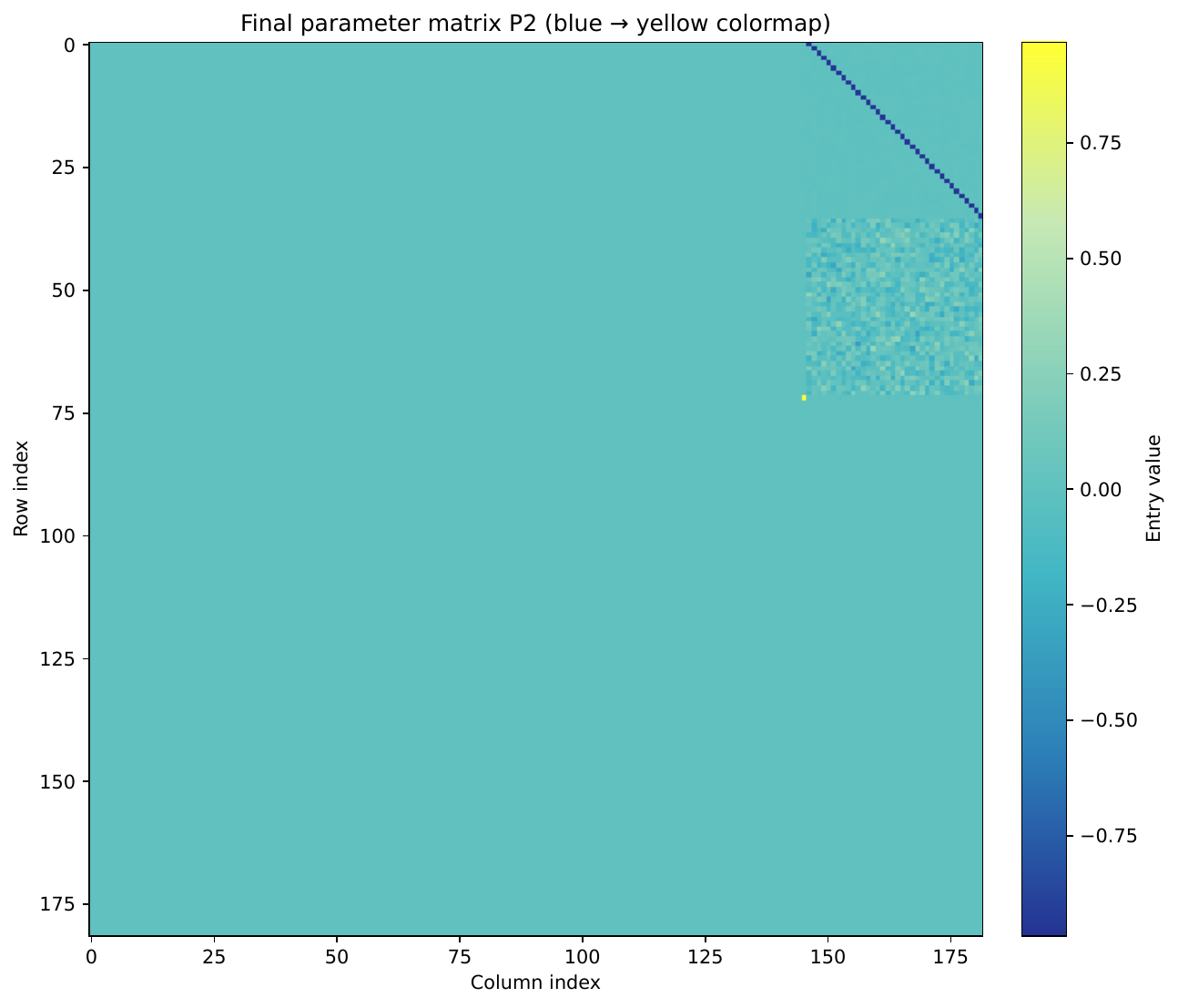}
    \caption{Learned $\mathbf{P}_2$}
  \end{subfigure}
  \hfill
  \begin{subfigure}{0.24\textwidth}
    \centering
    \includegraphics[width=\linewidth]{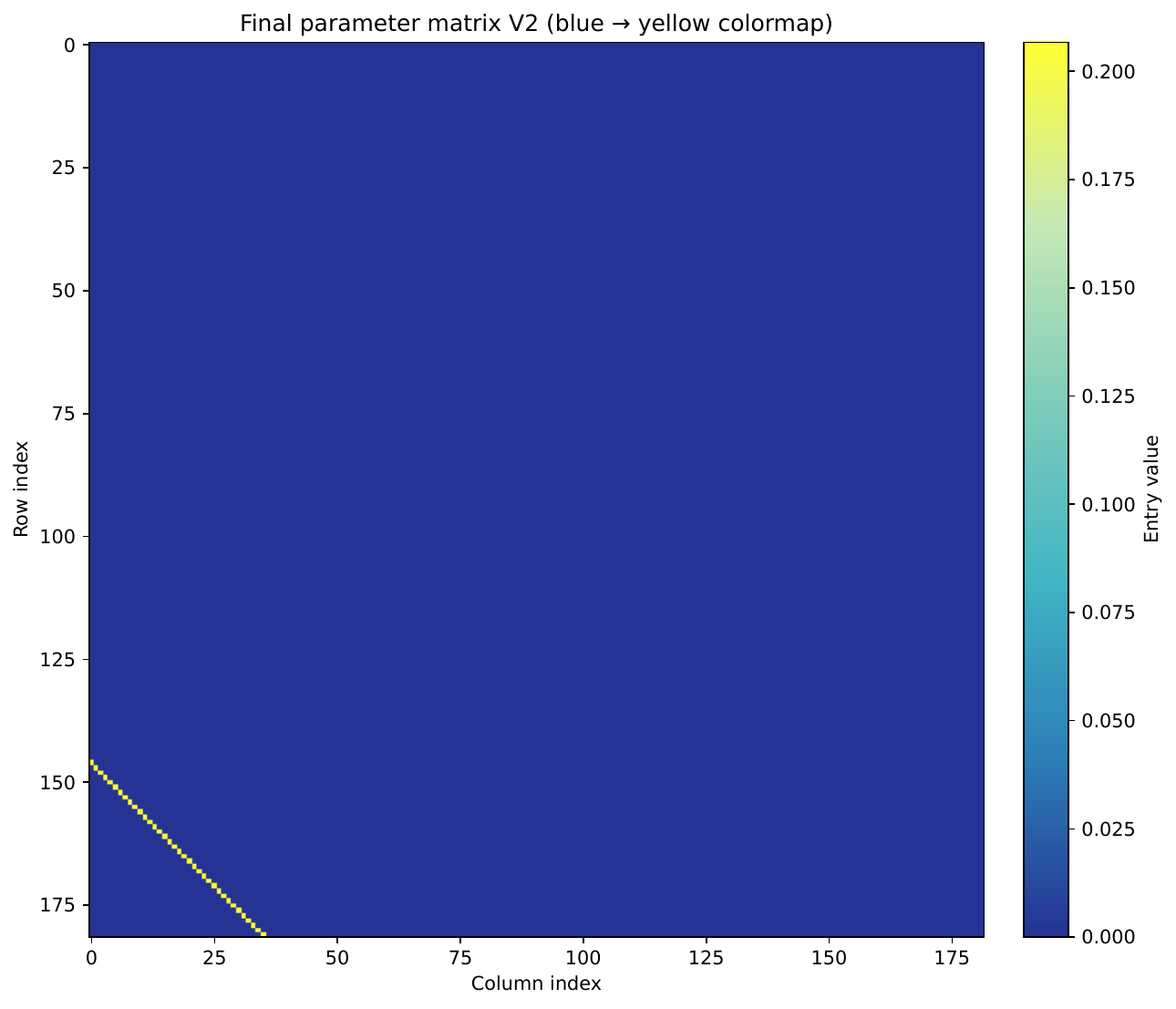}
    \caption{Learned $\mathbf{V}_2$}
  \end{subfigure}
  \caption{Final learned $\mathbf{P}_1$, $\mathbf{V}_1$, $\mathbf{P}_2$ and $\mathbf{V}_2$ matrices after training at $\tau=0.2$, $B=32$, and $n=20$. the non-zero structure matches the theoretical constructions $(\mathbf{P}_1, \mathbf{V}_1, \mathbf{P}_2, \mathbf{V}_2)$ in \eqref{eq:layer1} and \eqref{eq:layer2}.}  
  \label{em:emergence}
\end{figure}

We define some evaluation metrics to demonstrate whether the final parameters emerge.
Write $\mathbf{B}_-=[b_-,K_-]$ and $\mathbf{B}_+=[b_+,K_+]$. The analytical construction defines a sufficient family satisfying
 $\mathbf{A}_{\theta} =\tau^{-1}I$, $K_- =-cI$, $\mathbf{B}_r =c[1,0]$, 
 $\mathbf{L}^\top K_+ =\gamma cI$, $c\mathbf{C} =\alpha I$, $b_- =0$, $\mathbf{L}^\top b_+=0$, 
where $c\ne0$ may be positive or negative. For each trained model, we estimate $\widehat c=((\mathbf{B}_r)_0-\operatorname{tr}K_-)/(d+1)$ separately for each
model, allowing either sign. The product $\mathbf{L}^\top K_+$ removes the arbitrary invertible representation of the
intermediate greedy feature. Likewise $\mathbf{L}^\top b_+$, rather than $b_+$ alone,
is the appropriate product entering the output relation. Low residuals indicate
agreement with these sufficient relations.
The complete diagnostics used in every table are
\begin{align}
 E_{\mathbf{A}_{\theta}}&=\frac{\|\tau \mathbf{A}_{\theta} -I\|_F}{\sqrt d},
 &E_-&=\frac{\|K_-+\widehat cI\|_F}{|\widehat c|\sqrt d},\\
 E_r&=\frac{\|\mathbf{B}_r-\widehat c[1,0]\|_F}{|\widehat c|},
 &E_+&=\frac{\|\mathbf{L}^\top K_+-\gamma\widehat cI\|_F}{\gamma|\widehat c|\sqrt d},\\
 E_{C}&=\frac{\|\widehat c \mathbf{C}-\alpha I\|_F}{\alpha\sqrt d},
 &E_{-,0}&=\frac{\|b_-\|_2}{|\widehat c|\sqrt d},\\
 E_{+,0}&=\frac{\| \mathbf{L}^\top b_+\|_2}{\gamma|\widehat c|\sqrt d},
 &E_{\rm TD}^2&=E_-^2+E_r^2+E_+^2+E_C^2+E_{-,0}^2+E_{+,0}^2.
\end{align}
and $E_{\rm all}=\sqrt{E_{\mathbf{A}_{\theta}}^2+E_{\rm TD}^2}$, $\widehat\tau_{\rm eff}=1/\beta_{att}$ with
$\beta_{att}=\operatorname{tr}(\mathbf{A}_{\boldsymbol{\theta}})/d>0$ is a scalar projection diagnostic, not an
exact temperature when $\mathbf{A}_{\boldsymbol{\theta}}$ has nonscalar components. The scalar estimate is allowed to be negative.
The full completed models have nondegenerate $\widehat c$; a normalization near
zero is not interpreted as a zero residual. For arbitrary active blocks,
$E_{\rm all}$ is not an established Euclidean distance to all minimizers.
Figures \ref{FinalP1} and \ref{FinalP2} present heatmaps of the key matrices and their relations under $\tau=0.2$, respectively, revealing that the training algorithm successfully induces the constructed parameters to emerge.

\clearpage
\begin{figure}[H]
\centering
\includegraphics[width=.48\linewidth]{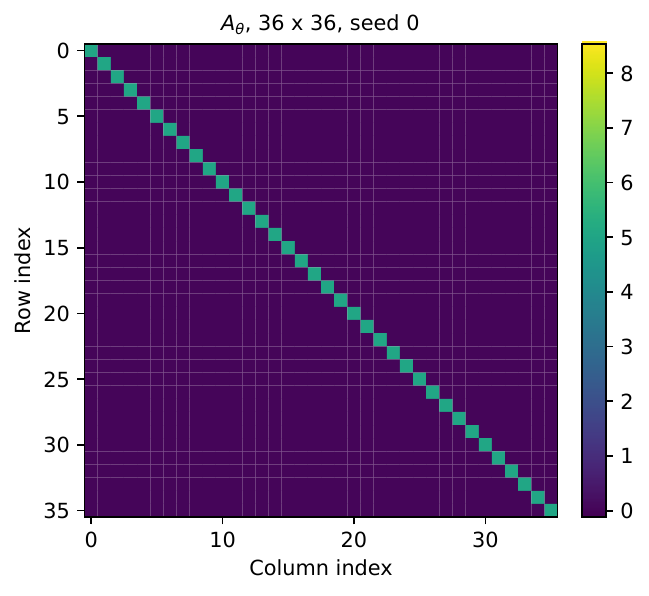}\hfill
\includegraphics[width=.48\linewidth]{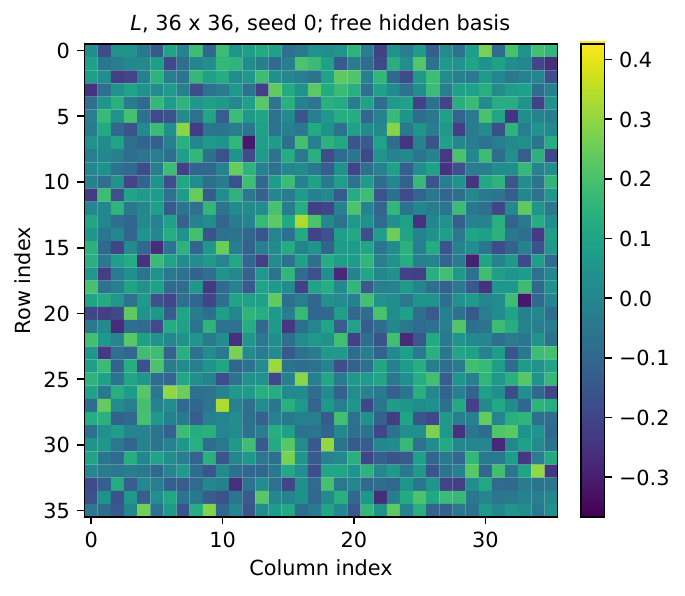}\par\medskip
\includegraphics[width=.48\linewidth]{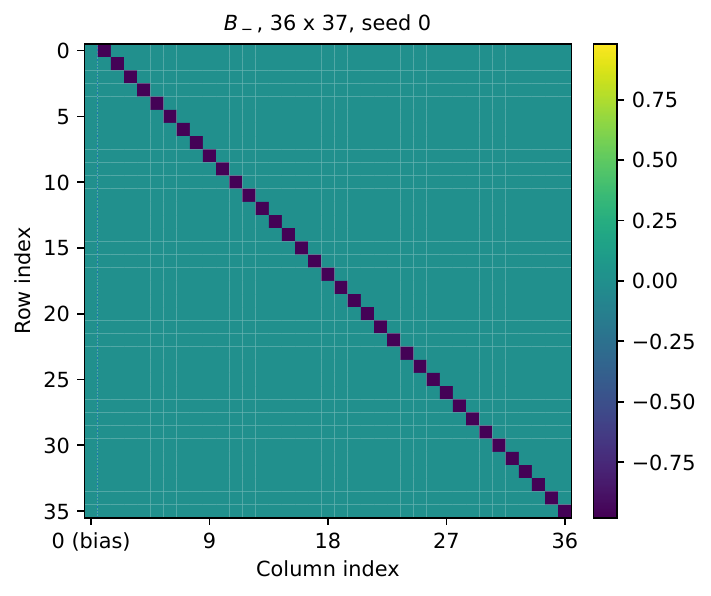}\hfill
\includegraphics[width=.48\linewidth]{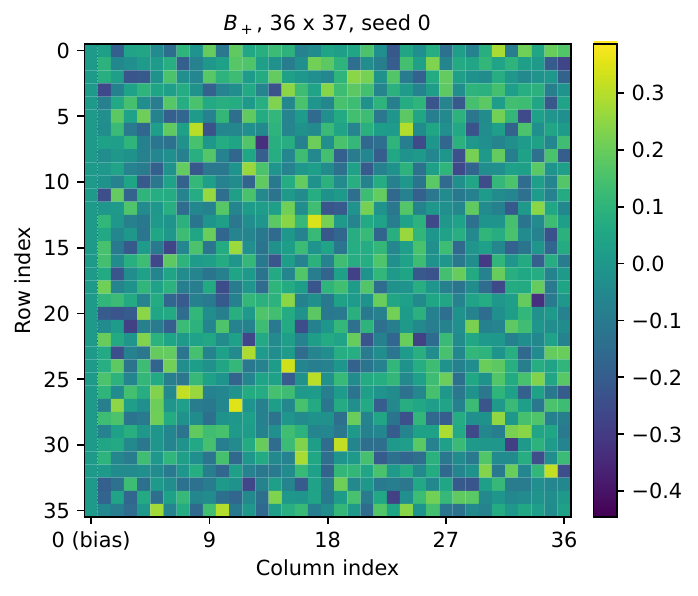}\par\medskip
\includegraphics[width=.48\linewidth]{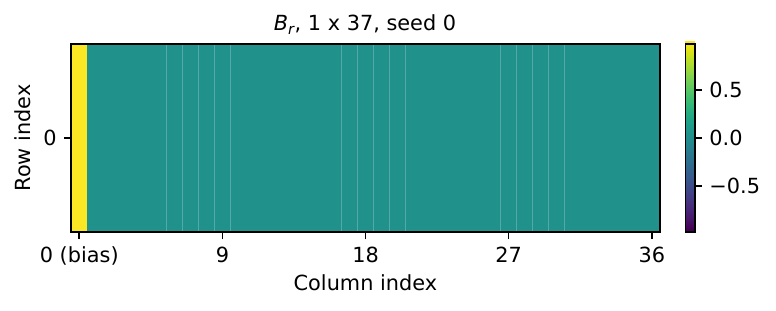}\hfill
\includegraphics[width=.48\linewidth]{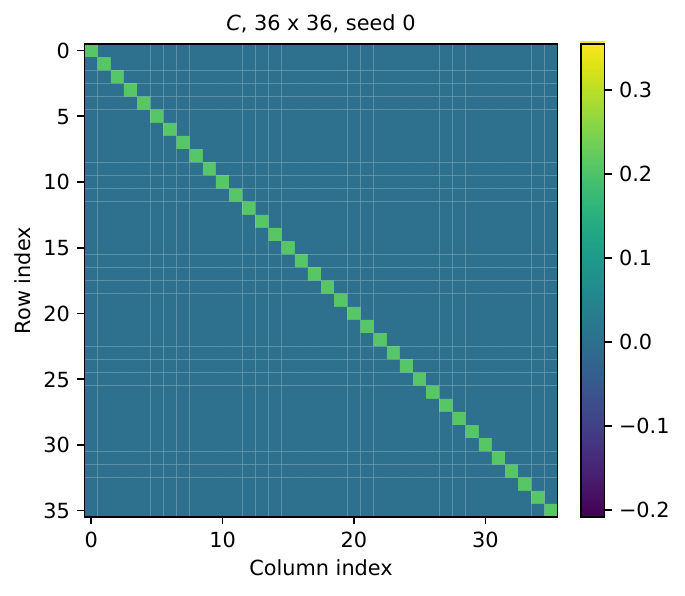}\par\medskip
\caption{All six actual trained blocks, seed 0. $\mathbf{A}_{\boldsymbol{\theta}},\mathbf{L},\mathbf{C}$ are $36\times36$; $\mathbf{B}_-,\mathbf{B}_+$ are $36\times37$; $\mathbf{B}_r$ is $1\times37$. The fitted scalar is $\widehat c=0.967893$. Dense $\mathbf{L}$ and $\mathbf{B}_+$ are compatible with a nontrivial invertible gauge and are not replaced by identity-shaped pictures.}
\label{FinalP1}
\end{figure}
\clearpage
\begin{figure}[H]
\centering
\includegraphics[width=.48\linewidth]{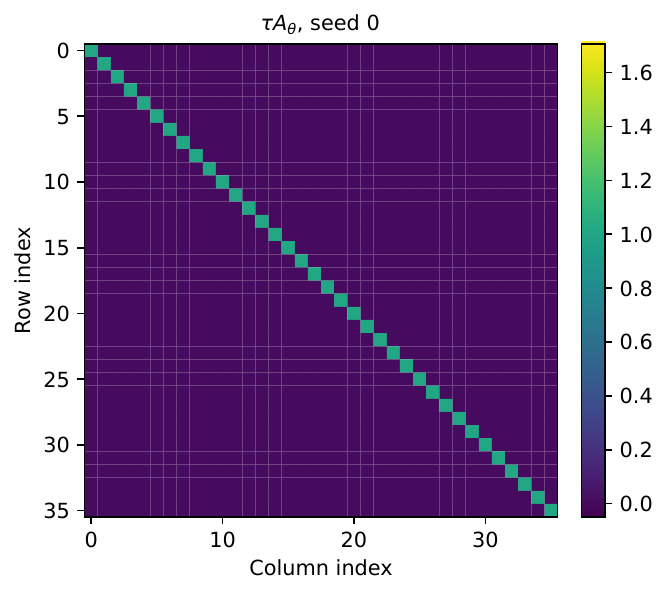}\hfill
\includegraphics[width=.48\linewidth]{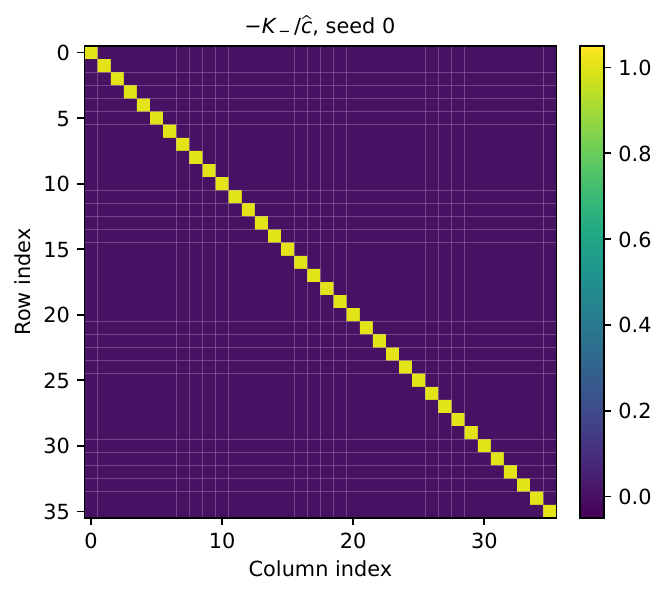}\par\medskip
\includegraphics[width=.48\linewidth]{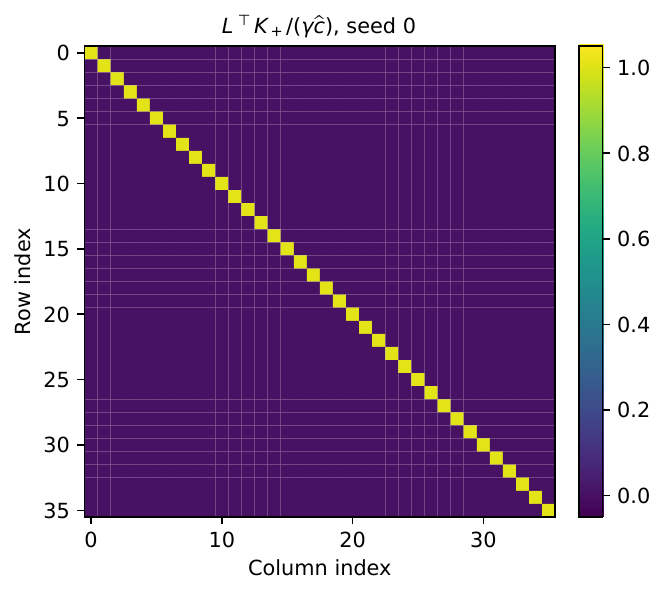}\hfill
\includegraphics[width=.48\linewidth]{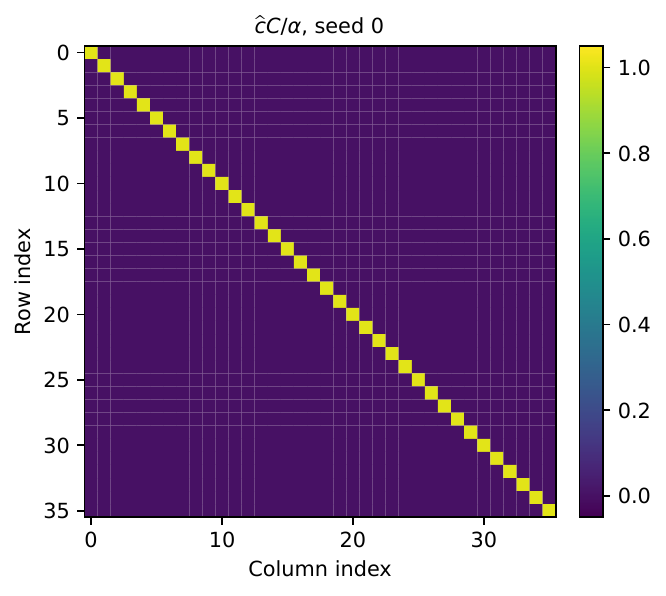}\par\medskip
\includegraphics[width=.48\linewidth]{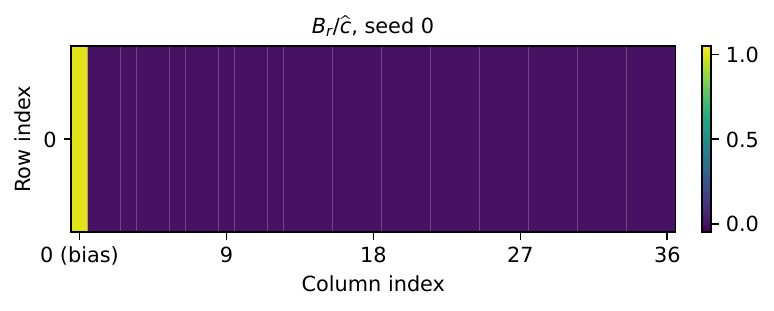}\hfill
\includegraphics[width=.48\linewidth]{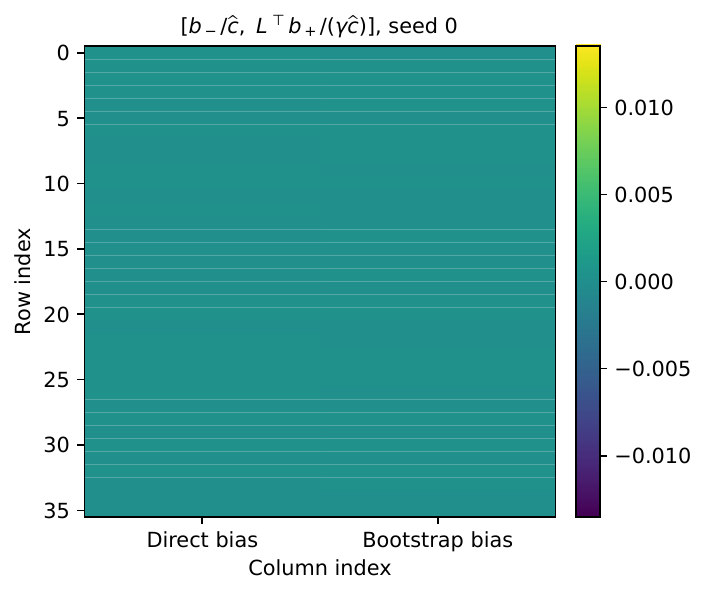}\par\medskip
\caption{Normalized relations for the same model. The first four targets are $I$, the reward-row target is $[1,0,\ldots,0]$, and the two bias targets are zero. These are display-only diagnostics, not raw weights. Common limits are used for each named relation across seeds and methods, including the failed control.}
\label{FinalP2}
\end{figure}
For the relationship between the emergence of parameters and cumulative \textbf{task exposure}.  Figure~\ref{fig:cos-task-curves} tracks recovery along one training trajectory under growing cumulative task exposure. Reduced final TD residual shows attention-scale calibration dominates aggregate residual. Since task exposure and optimizer updates rise together, the curves cannot isolate task diversity’s fixed-compute effect.

\begin{figure}[H]
\centering
\includegraphics[width=.48\linewidth]{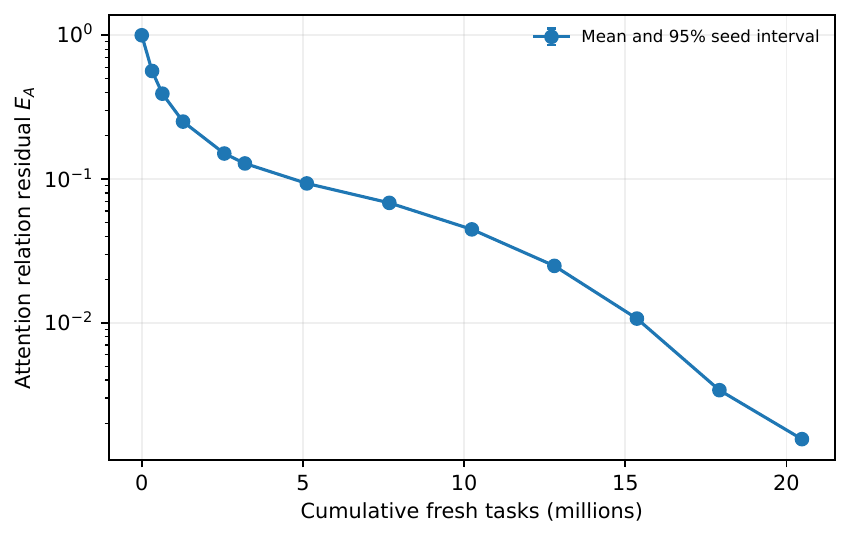}\hfill
\includegraphics[width=.48\linewidth]{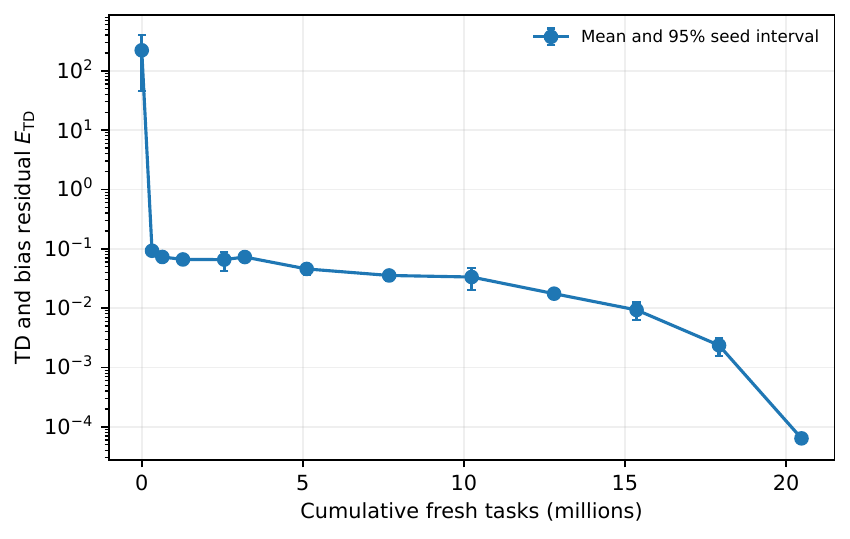}
\caption{Parameter recovery versus cumulative fresh tasks at $B=32$, $n=20$, and $\tau=0.2$. Left: attention residual $E_A$. Right: the aggregate of the six remaining TD and bias residuals, $E_{\rm TD}$. }\label{fig:cos-task-curves}
\end{figure}

Table~\ref{tab:cos-tau} shows close agreement with the analytical relations at \textbf{temperature} $\tau=1,0.8,0.5$, and $0.2$, while recovery remains incomplete at $\tau=0.1$. At $\tau=0.2$, the learned scalar projection is $\beta_{att}=4.992236$ against a target of $5$. At $\tau=0.1$, it is $8.174396$ against a target of $10$, and $E_{\rm all}=0.1826$. Figure~\ref{fig:cos-temperature} makes this scale discrepancy explicit rather than inferring recovery from matrix appearance alone. No temperature-specific tuning is performed; the tested values do not establish a critical temperature or an impossibility result for lower-temperature recovery.

\begin{table}[H]
\centering\small
\setlength{\tabcolsep}{3pt}
\caption{Final parameter recovery at 20.48 million tasks, $B=32$, and $n=20$. Entries are five-seed means; $E_{\rm all}$ includes a pointwise 95\% interval half-width. Residuals are dimensionless. The scalar projection $\beta_{att}$ is compared with $1/\tau$ without removing nonscalar distortion.}
\label{tab:cos-tau}
\begin{tabular}{rrrrrr}
\toprule
$\tau$ & $1/\tau$ & $\beta_{att}$ & $E_A$ & $E_{\rm TD}$ & $E_{\rm all}$ \\
\midrule
1 & 1 & 1.000000 & $3.808\times10^{-5}$ & $5.410\times10^{-5}$ & $(6.624\pm2.102)\times10^{-5}$ \\
0.8 & 1.25 & 1.250000 & $3.150\times10^{-5}$ & $4.741\times10^{-5}$ & $(5.715\pm1.049)\times10^{-5}$ \\
0.5 & 2 & 2.000002 & $2.424\times10^{-5}$ & $4.882\times10^{-5}$ & $(5.466\pm1.323)\times10^{-5}$ \\
0.2 & 5 & 4.992236 & $1.553\times10^{-3}$ & $6.384\times10^{-5}$ & $(1.554\pm0.017)\times10^{-3}$ \\
0.1 & 10 & 8.174396 & $0.1826$ & $6.590\times10^{-4}$ & $0.18256\pm0.00020$ \\
\bottomrule
\end{tabular}
\end{table}

\begin{figure}[H]
\centering
\includegraphics[width=.485\linewidth]{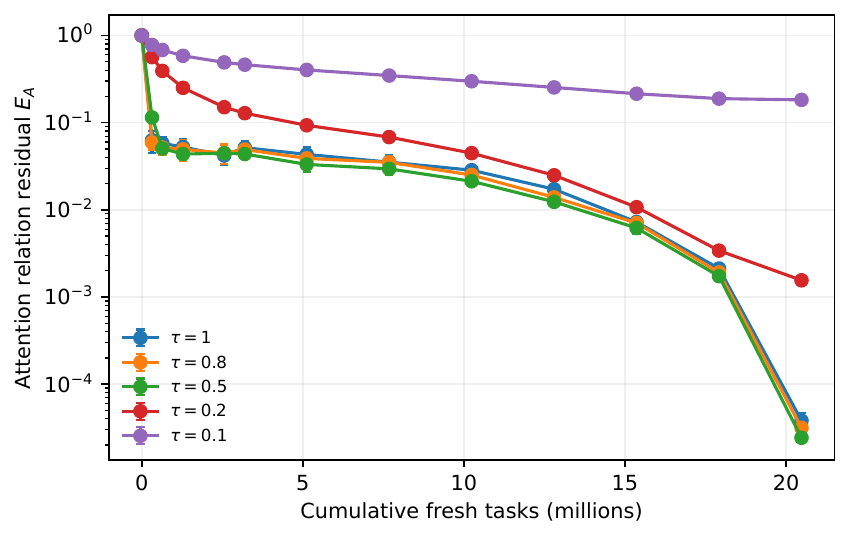}\hfill
\includegraphics[width=.485\linewidth]{figures/temperature_scale1.pdf}
\caption{Temperature-dependent recovery at $B=32$ and $n=20$. Left: attention residual versus cumulative tasks, with five paired seeds per temperature. Right: the learned scalar projection $\beta_{att}$ and its target $1/\tau$. The target is shown only for evaluation; no matrix target is supplied during training.}\label{fig:cos-temperature}
\end{figure}

For the relationship between the emergence of parameters and \textbf{batch size}. We report scale calibration and matrix shape in Figure~\ref{fig:cos-batch}. The lowest endpoint mean of $E_{\rm all}$ occurs at $B=128$, with value $7.204\times10^{-4}$. By contrast, mean soft-update error continues to decrease up to $B=64$ in Table \ref{tab:functional}. Thus, the best observed parameter recovery and the lowest observed functional error do not coincide in this sweep.
At fixed $N_{\rm task}$, reducing $B$ increases the number of Adam updates, whereas increasing $B$ averages more tasks in each update. Adam's moment memories are defined in update units, so their span in task units also changes. The observed ranking is consequently specific to the shared schedule, numerical precision, architecture, and task budget; it does not establish a universally optimal batch size.

\begin{figure}[H]
\centering
\includegraphics[width=.485\linewidth]{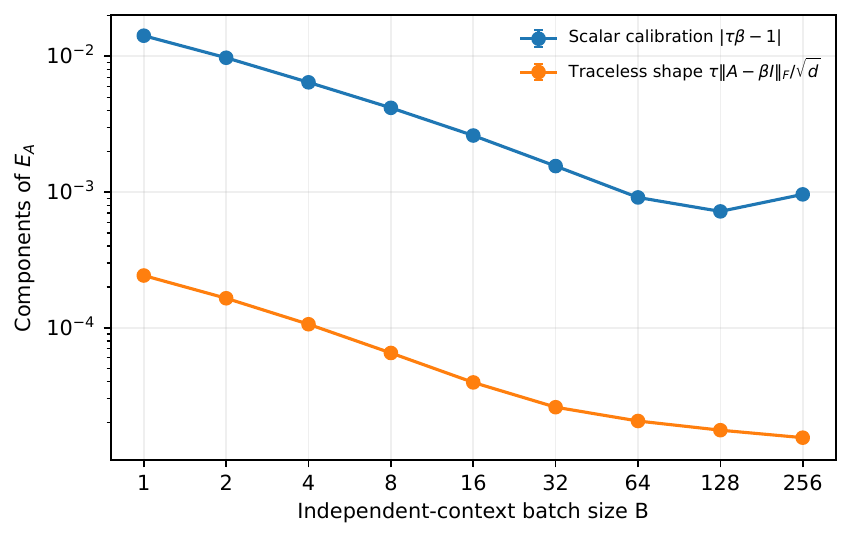}\hfill
\includegraphics[width=.485\linewidth]{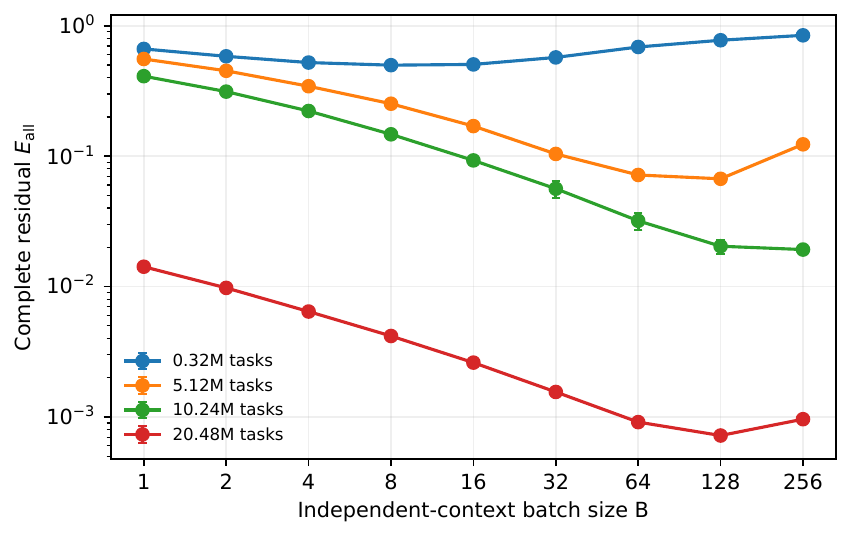}
\caption{Batch-size effects under a fixed task budget. Left: attention scale and traceless-shape residuals, whose squares sum to $E_A^2$ for each model. Right: aggregate residuals at four checkpoints of the same training runs. Intermediate curves are not independently retrained short-budget experiments.}\label{fig:cos-batch}
\end{figure}

For the relationship between the emergence of parameters and \textbf{Training context length}.
For each training length, only the first $n_{\rm train}$ transitions of the common 80-step trajectory enter the loss. The teacher increment is recomputed from that prefix with normalization $\alpha/n_{\rm train}$. The $n_{\rm train}=20$ default uses the same paired generator and is shared across the three sweeps.

\begin{table}[H]
\centering\small
\setlength{\tabcolsep}{3pt}
\caption{Training-length comparison at $\tau=0.2$ and $B=32$. Every model receives 20.48 million tasks and 640,000 updates; supplied transitions differ. Residuals are five-seed means, with the 95\% interval half-width for $E_{\rm all}$ in the final column. The default is shared across sweeps.}
\label{tab:cos-length}
\begin{tabular}{rrrrrr}
\toprule
$n_{\rm train}$ & Transitions (M) & $E_A$ & $E_{\rm TD}$ & $E_{\rm all}$ & \shortstack{95\% interval\\half-width} \\
\midrule
1 & 20.48 & $0.0128$ & $2.033\times10^{-4}$ & $0.0128$ & $1.969\times10^{-4}$ \\
5 & 102.4 & $2.155\times10^{-3}$ & $8.059\times10^{-5}$ & $2.156\times10^{-3}$ & $8.661\times10^{-5}$ \\
10 & 204.8 & $1.699\times10^{-3}$ & $6.440\times10^{-5}$ & $1.700\times10^{-3}$ & $5.381\times10^{-5}$ \\
20 & 409.6 & $1.553\times10^{-3}$ & $6.384\times10^{-5}$ & $1.554\times10^{-3}$ & $1.728\times10^{-5}$ \\
40 & 819.2 & $1.739\times10^{-3}$ & $6.393\times10^{-5}$ & $1.740\times10^{-3}$ & $5.290\times10^{-5}$ \\
80 & 1638.4 & $2.019\times10^{-3}$ & $7.050\times10^{-5}$ & $2.020\times10^{-3}$ & $3.846\times10^{-5}$ \\
\bottomrule
\end{tabular}
\end{table}

\begin{table}[H]
\centering\small
\setlength{\tabcolsep}{3pt}
\caption{Length-specific models on matched-length and common-$n_{\rm test}=20$ dense tests, and matched-length one-hot tests. Entries are update NRMSE percentages with pointwise 95\% seed-interval half-widths. The tabular column measures feature transfer without one-hot training.}
\label{tab:cos-length-fit}
\begin{tabular}{rrrr}
\toprule
$n_{\rm train}$ & Dense matched (\%) & Dense common 20 (\%) & Tabular matched (\%) \\
\midrule
1 & $0.01893\pm0.00032$ & $0.01573\pm0.00025$ & $0.08467\pm0.00095$ \\
5 & $0.00561\pm0.00039$ & $0.00525\pm0.00038$ & $0.01558\pm0.00077$ \\
10 & $0.00431\pm0.00043$ & $0.00417\pm0.00042$ & $0.01284\pm0.00049$ \\
20 & $0.00393\pm0.00019$ & $0.00393\pm0.00019$ & $0.01222\pm0.00036$ \\
40 & $0.00367\pm0.00020$ & $0.00384\pm0.00021$ & $0.01413\pm0.00060$ \\
80 & $0.00380\pm0.00012$ & $0.00411\pm0.00012$ & $0.01679\pm0.00044$ \\
\bottomrule
\end{tabular}
\end{table}

Among the tested lengths, $n_{\rm train}=20$ yields the lowest sample mean of $E_{\rm all}$, $1.554\times10^{-3}$ (Table~\ref{tab:cos-length}). This is a fixed-task, fixed-update comparison, not a fixed-transition or fixed-compute comparison: $n=80$ supplies 80 times as many transitions per task as $n=1$. Context length also changes the distribution and variance of the window-averaged target.
Table~\ref{tab:cos-length-fit} therefore reports both matched-length tests and tests at the common length $n_{\rm test}=20$. The common-length evaluation holds the test window fixed when comparing models trained with different lengths. The one-hot results measure dense-to-tabular feature transfer, not performance after one-hot training.
\begin{figure}[H]
\centering
\includegraphics[width=.485\linewidth]{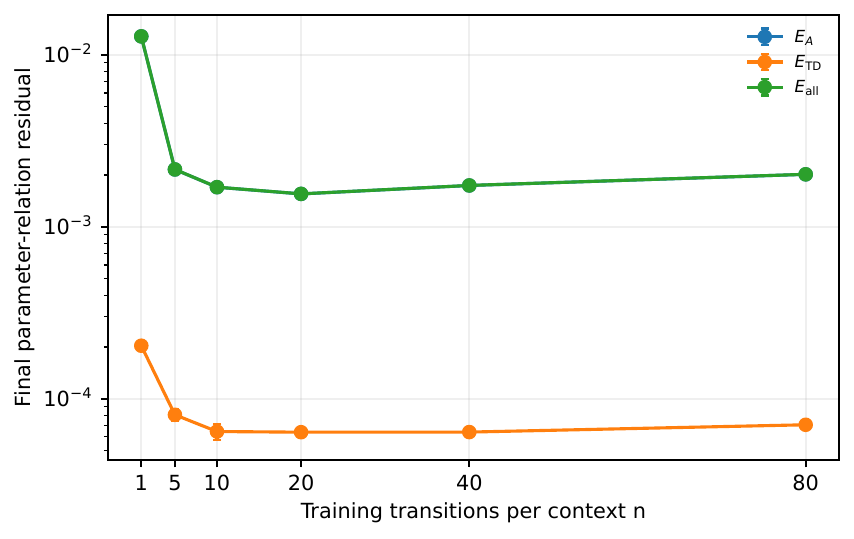}\hfill
\includegraphics[width=.485\linewidth]{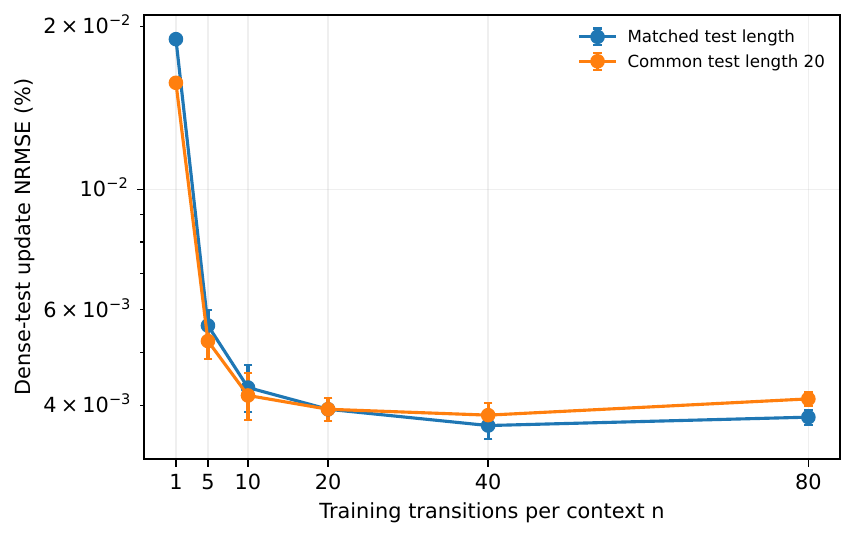}
\caption{Training context length at $\tau=0.2$ and $B=32$, with 20.48 million tasks per model. Left: final parameter residuals. Right: update errors on matched-length and common-length dense-feature tests. }\label{fig:cos-length}
\end{figure}

\paragraph{Generalization to unseen tabular tasks.}
At $\tau=0.2$, the Transformer achieves mean return $0.4880$, compared with $0.4861$ for its soft teacher and $0.5374$ for Oracle. The paired Transformer-minus-soft difference is $0.00188$. This comparison does not resolve a return advantage and does not establish equivalence or noninferiority. The visit-count online Q-learning baseline reaches $0.5318$, exceeding the Transformer means in the primary comparison. The evidence therefore supports useful self-fed transfer (Figure \ref{fig:cos-online}). 
Deployment intervals use 5,000 percentile-bootstrap replicates with independent resampling of the training-seed and test-task axes, preserving paired methods and complete trajectories. 
\begin{figure}[H]
\centering\includegraphics[width=.6\linewidth]{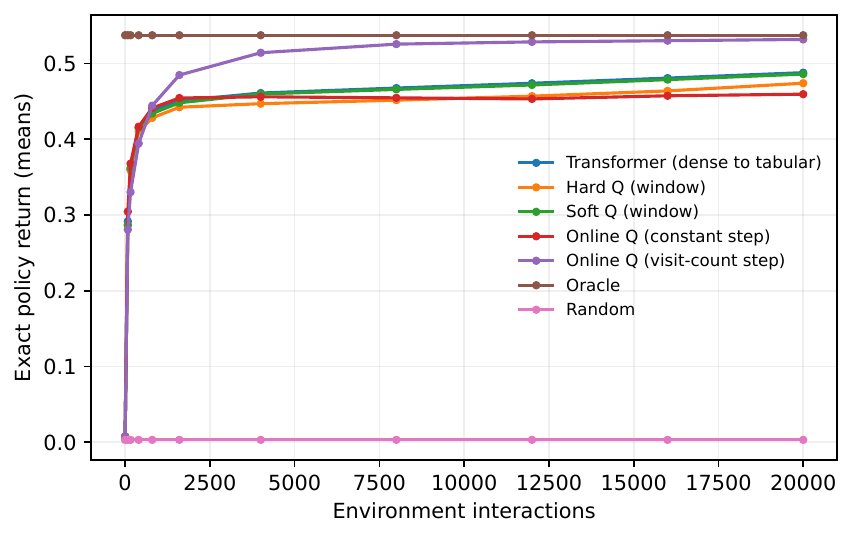}
\caption{Online Q-learning is included in the comparisons; the variant with decaying visit-count learning rate yields the best performance.}\label{fig:cos-online}
\end{figure}

Table \ref{tab:cos-control} and Figure \ref{fig:cos-control-all} complete the five-temperature comparison; $\tau=0.2$ appears in the main text. All curves use the same task bank, interaction axis, and evaluation protocol, and include Transformer, hard window-Q, soft window-Q, Oracle, and Random.

\begin{table}[H]
\centering\small
\setlength{\tabcolsep}{3pt}
\caption{Final returns on 512 unseen tabular tasks: $n_{\rm test}=20$, zero Q initialization.}
\label{tab:cos-control}
\begin{tabular}{rrrrrr}
\toprule
$\tau$ & Transformer [95\% CI] & \shortstack{Hard Q\\(window)} & \shortstack{Soft Q\\(window)} & Oracle & Random \\
\midrule
1 & 0.5036 [0.4835, 0.5237] & 0.4740 & 0.5036 & 0.5374 & 0.0033 \\
0.8 & 0.5033 [0.4829, 0.5227] & 0.4740 & 0.5025 & 0.5374 & 0.0033 \\
0.5 & 0.5003 [0.4804, 0.5198] & 0.4740 & 0.4996 & 0.5374 & 0.0033 \\
0.2 & 0.4880 [0.4686, 0.5074] & 0.4740 & 0.4861 & 0.5374 & 0.0033 \\
0.1 & 0.4833 [0.4644, 0.5022] & 0.4740 & 0.4795 & 0.5374 & 0.0033 \\
\bottomrule
\end{tabular}
\end{table}

\begin{figure}[htbp]
\centering
\includegraphics[width=.485\linewidth]{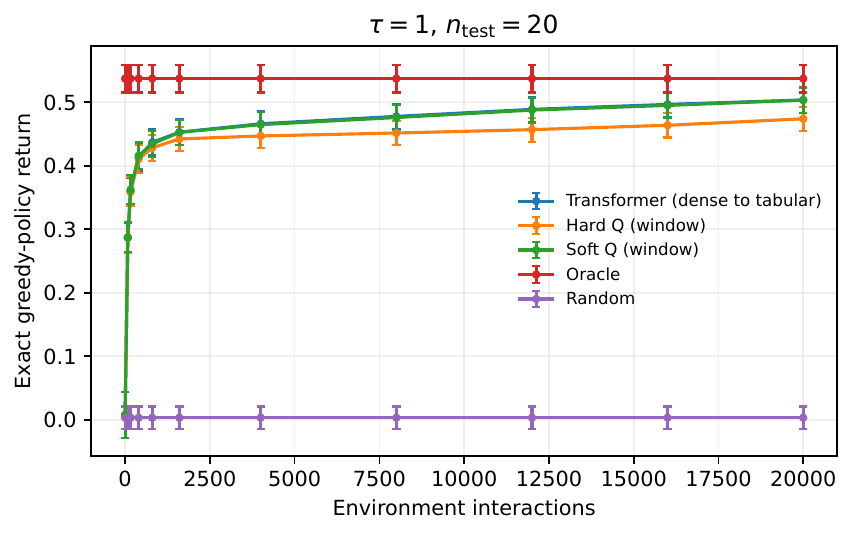}\hfill
\includegraphics[width=.485\linewidth]{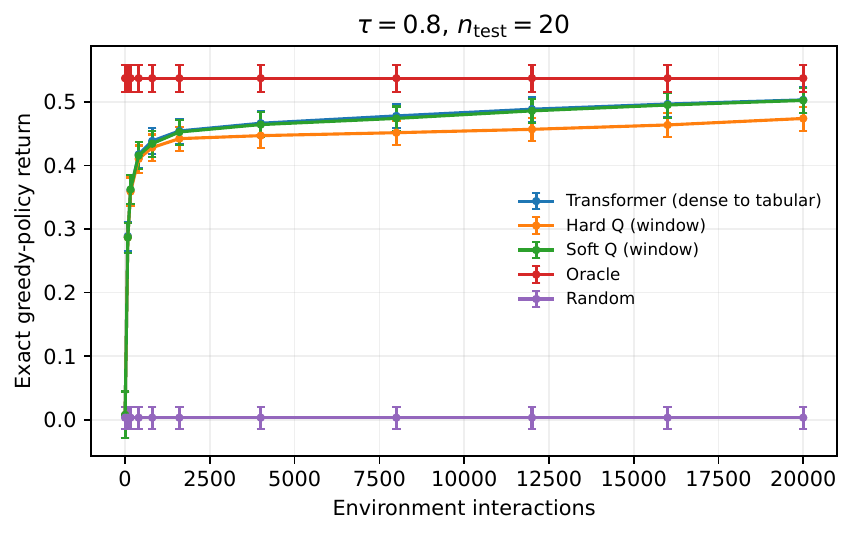}

\vspace{0.5em}
\includegraphics[width=.485\linewidth]{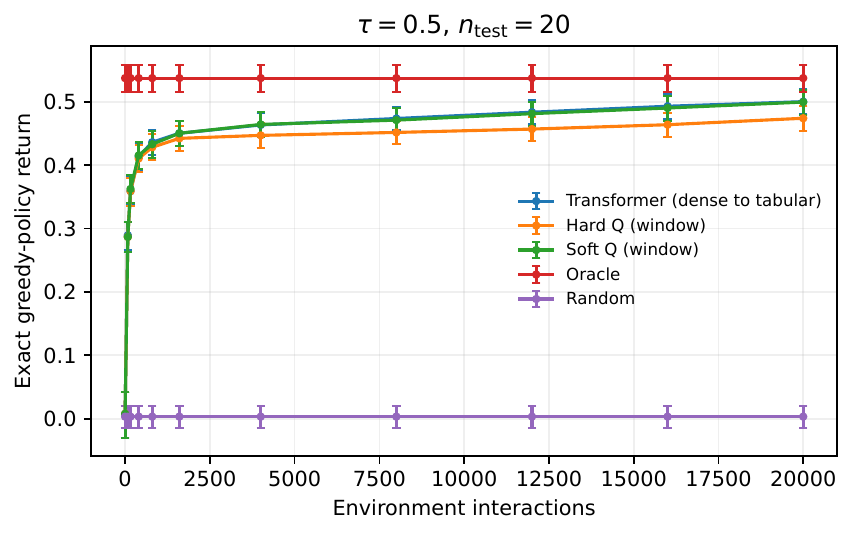}\hfill
\includegraphics[width=.485\linewidth]{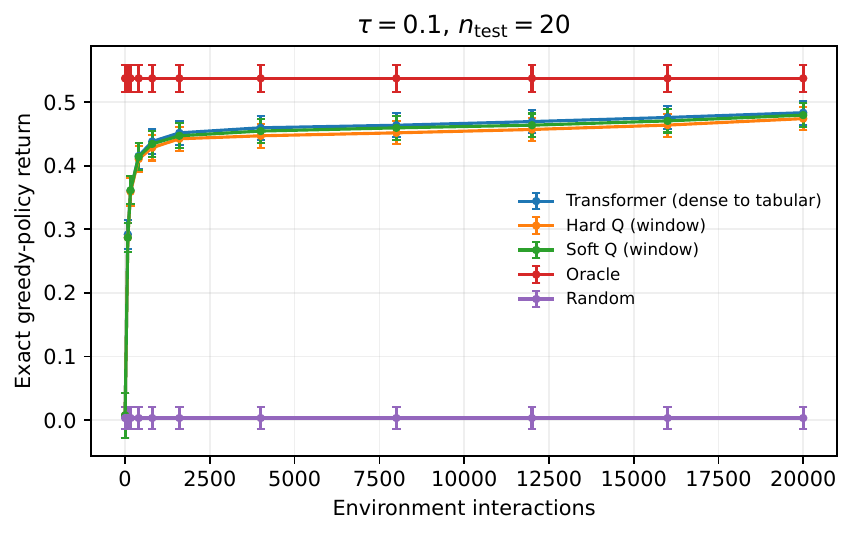}

\caption{Tabular transfer at $\tau=1$, $\tau=0.8$, $\tau=0.5$, and $\tau=0.1$, with $n_{\rm test}=20$, zero Q initialization, and 20,000 interactions. Error bars denote pointwise 95\% intervals. Transformer weights are frozen throughout. All experiments use the same 512 tasks and 20,000-interaction budget.}\label{fig:cos-control-all}
\end{figure}

For training and deployment context lengths. We distinguish three evaluations. The frozen-length test varies $n_{\rm test}$ for models trained with $n_{\rm train}=20$. The matched-length test deploys each length-specific model at its training length. The common-length test deploys all length-specific models at $n_{\rm test}=20$.
Every learner receives 20,000 interactions, so a window-based learner performs $20,000/n_{\rm test}$ updates. Varying the deployment length therefore changes both the window-averaging coefficient and the update frequency.

In the frozen-length test, $n_{\rm test}=5$ yields the highest observed Transformer mean return, $0.5194$, whereas $n_{\rm test}=20$ yields $0.4880$. The window teachers exhibit a similar dependence on deployment length. With common $n_{\rm test}=20$, the length-specific models instead have mean returns between $0.4868$ and $0.4884$. Thus, training-context effects should be distinguished from deployment update frequency; longer windows do not uniformly improve finite-budget control.
\begin{figure}[H]
\centering
\includegraphics[width=.485\linewidth]{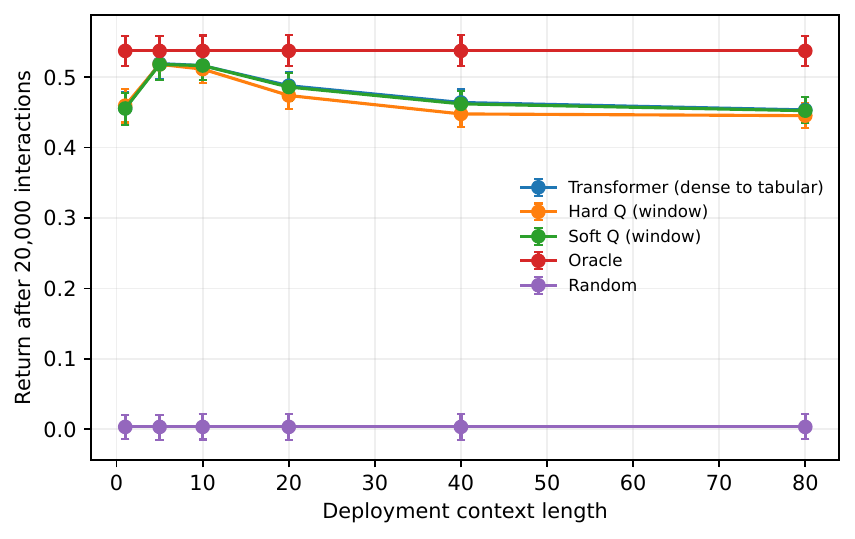}\hfill
\includegraphics[width=.485\linewidth]{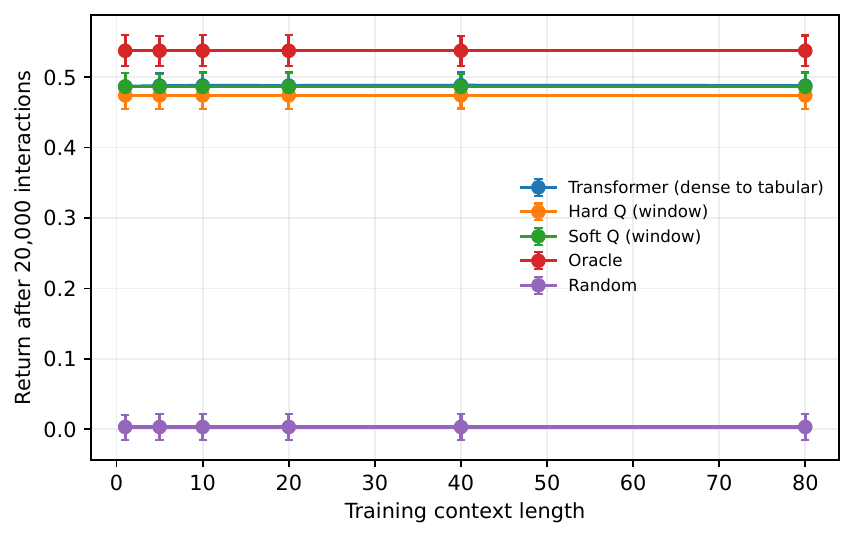}
\caption{Separating training-length and deployment-length effects at $\tau=0.2$. Left: models trained at $n_{\rm train}=20$ and deployed with different windows. Right: models trained at different lengths and all deployed at $n_{\rm test}=20$. Both evaluations use the same task bank and 20,000-interaction budget.}\label{fig:cos-control-length}
\end{figure}

\section{Theoretical analysis}\label{app:Theoretical}
Fix integers $d\ge1$, $m\ge2$, $n\ge1$ and real constants $\alpha>0$, $0<\gamma<1$,  $\tau>0$. The meanings of these fixed quantities are given below. The temperature $\tau$ is fixed during optimization, let
\begin{equation*}
\mathbf{z}=\bigl(\mathbf{w},\{\phi_i,r_{i+1},\psi_{i,a}:0\le i<n,\ a\in[m]\}\bigr)\sim\mathcal D_n.
\end{equation*}
Here $\mathbf{w}\in\R^d$ is an input vector, held fixed within a context. The vector $\phi_i\in\R^d$ is the current state--action feature, $r_{i+1}\in\R$ is the observed reward, and $\psi_{i,a}\in\R^d$ is the feature of candidate action $a$ at the successor state. These are allowed to be statistically dependent.   The symbol $\mathcal D_n$ denotes their joint probability law, let
\begin{align}
\pi_{i,a}^{\tau}(\mathbf{z})&=\frac{\exp(\mathbf{w}^\top\psi_{i,a}/\tau)}{\sum_{b=1}^{m}\exp(\mathbf{w}^\top\psi_{i,b}/\tau)},
&\bar\psi_i^\tau(\mathbf{z})&=\sum_{a=1}^m\pi_{i,a}^\tau(\mathbf{z})\psi_{i,a},\label{eq:teacher-attention}\\
\delta_i^\tau(\mathbf{z})&=r_{i+1}-\mathbf{w}^\top\phi_i+\gamma \mathbf{w}^\top\bar\psi_i^\tau(\mathbf{z}),
&F_{\tau,n}(\mathbf{z})&=\frac\alpha n\sum_{i=0}^{n-1}\phi_i\delta_i^\tau(\mathbf{z}).
\label{eq:teacher}
\end{align}
The probabilities $\pi_{i,a}^\tau$ average successor features; they are not the behavior policy that generated the trajectory. The scalar $\delta_i^\tau$ is the specified temporal-difference target minus the current score. The vector $F_{\tau,n}\in\R^d$ is the update increment to imitate. Thus the teacher's full prediction is $\mathbf{w}+F_{\tau,n}$.  

In this appendix, we explicitly restrict optimization to the six-block support of Proposition \ref{prop:six-block-minimal-support}, with every other block fixed at zero. The six-block support
${\boldsymbol{\theta}}_{\mathrm{sup}}
=
(\mathbf{A}_{\boldsymbol{\theta}},\mathbf{L}, \mathbf{B}_-, \mathbf{B}_+, \mathbf{B}_r,\mathbf{C})
\in\Theta_{\mathrm{sup}}$ can be broken down into following blocks
\begin{equation}
{\boldsymbol{\theta}}=(A,L,b_-,K_-,b_+,K_+,\beta,q,C)\in\Theta,
\qquad \dim\Theta=5d^2+3d+1.
\label{eq:parameters}
\end{equation}
where $A=\mathbf{A}_\theta$, $L=\mathbf{L}$, $\mathbf{B}_{-} = [b_{-}, K_{-}]$, $\mathbf{B}_{+} = [b_{+}, K_{+}]$, $\mathbf{B}_r = [\beta, q^\top]$.
The trainable matrix $L$ is not assumed invertible; invertibility is imposed only on the exact-realization manifold.
\begin{align}
&p_{i,a}(A;\mathbf{z})=\frac{\exp(\psi_{i,a}^\top A\mathbf{w})}{\sum_{b=1}^m\exp(\psi_{i,b}^\top A\mathbf{w})}, \quad \bar\psi_i(A;\mathbf{z})=\sum_a p_{i,a}(A;\mathbf{z})\psi_{i,a},\label{eq:student-attention}\\
&s_i({\boldsymbol{\theta}};\mathbf{z})=\phi_i^\top(b_-+K_-\mathbf{w})+(L\bar\psi_i(A;\mathbf{z}))^\top(b_++K_+\mathbf{w})+r_{i+1}(\beta+q^\top \mathbf{w}),\label{eq:score}\\
&F_{{\boldsymbol{\theta}},n}(\mathbf{z})=\frac1n\sum_{i=0}^{n-1}C\phi_i s_i({\boldsymbol{\theta}};\mathbf{z}).
\label{eq:loss}
\end{align}
Both full outputs add the same $\mathbf{w}$, so their difference is the increment error
$e_{\boldsymbol{\theta}}(\mathbf{z}):=F_{{\boldsymbol{\theta}},n}(\mathbf{z})-F_{\tau,n}(\mathbf{z})$. For the fixed context length \(n\), we write the loss as $L_n(\theta)=\mathcal L(\theta)$, i.e.  
\begin{align}\label{eq:lossf}
L_n({\boldsymbol{\theta}})=\frac12\E_{\mathbf{z}\sim\mathcal D_n}\|F_{{\boldsymbol{\theta}},n}(\mathbf{z})-F_{\tau,n}(\mathbf{z})\|_2^2.
\end{align}
We prove Theorem~\ref{thm:main} in three parts. First, we characterize the local geometry and normal curvature of the optimal parameter manifold. Second, we establish a one-step contraction of gradient descent toward this manifold. Finally, we use these results to prove convergence to a single optimal parameter and derive the matching worst-case iteration complexity over the admissible initialization class.
\subsection{Geometry and Curvature of the Optimal Parameter Manifold}
A perturbation $\mathbf{U}\in\Theta$ has blocks denoted by $\Delta A,\ldots,\Delta C$ in the same order as \eqref{eq:parameters}.
\begin{align*}
\|\mathbf{U}\|_\Theta^2=&\|\Delta A\|_F^2+\|\Delta L\|_F^2+\|\Delta b_-\|_2^2
+\|\Delta K_-\|_F^2+\|\Delta b_+\|_2^2\nonumber\\
&+\|\Delta K_+\|_F^2+|\Delta\beta|^2+\|\Delta q\|_2^2+\|\Delta C\|_F^2. 
\end{align*}
For a vector, $\|x\|_2^2=\sum_jx_j^2$; for a matrix, $\|M\|_F^2=\operatorname{tr}(M^\top M)$. The product inner product is the corresponding sum of vector dot products and matrix Frobenius inner products. We abbreviate $\|\mathbf{U}\|_\Theta$ to $\|\mathbf{U}\|$ when the space is unambiguous.  Norms of operators are induced by their domain and codomain norms, not entrywise norms unless specified.

For a fixed context $\mathbf{z}$, let
$J_{{\boldsymbol{\theta}},n}(\mathbf{z}):=D_{\boldsymbol{\theta}} F_{{\boldsymbol{\theta}},n}(\mathbf{z}):\Theta\to\mathbb R^d$
denote the Jacobian of the Transformer output with respect to its parameters. Thus, for a parameter perturbation $\mathbf{U}\in\Theta$,
$J_{{\boldsymbol{\theta}},n}(\mathbf{z})[\mathbf{U}]=\left.\frac{d}{dt}F_{{\boldsymbol{\theta}}+t\mathbf{U},n}(\mathbf{z})\right|_{t=0}$.
Equivalently, $F_{{\boldsymbol{\theta}}+t\mathbf{U},n}(\mathbf{z})=F_{{\boldsymbol{\theta}},n}(\mathbf{z})+tJ_{{\boldsymbol{\theta}},n}(\mathbf{z})[\mathbf{U}]+o(t)$.
Hence $J_{{\boldsymbol{\theta}},n}(\mathbf{z})[\mathbf{U}]$ measures the first-order change of the Transformer output, on the particular context $\mathbf{z}$, when the parameters are perturbed in direction $\mathbf{U}$.

Since our training objective is the population loss \eqref{eq:lossf}, we now introduce the population Jacobian. We therefore consider the Hilbert space
$$\mathcal H_n
=
L^2(\mathcal D_n;\mathbb R^d),
\quad
\langle f,h\rangle_{\mathcal H_n}=
\mathbb E\langle f(\mathbf{z}),h(\mathbf{z})\rangle_2,
$$
with norm
$|f|_{\mathcal H_n}^2=\mathbb E|f(\mathbf{z})|_2^2$,
and define the population Jacobian
$\mathcal J_{{\boldsymbol{\theta}},n}:\Theta\to\mathcal H_n$, $(\mathcal J_{{\boldsymbol{\theta}},n}\mathbf{U})(\mathbf{z})=J_{{\boldsymbol{\theta}},n}(\mathbf{z})[\mathbf{U}]$. Thus
$|\mathcal J_{{\boldsymbol{\theta}},n}\mathbf{U}|_{\mathcal H_n}^2=\mathbb E|J_{{\boldsymbol{\theta}},n}(\mathbf{z})[\mathbf{U}]|_2^2$
measures the average squared first-order output change caused by the parameter perturbation $\mathbf{U}$ over the entire context distribution.

Let $\mathcal{J}_{{\boldsymbol{\theta}},n}^{*}$ denote the adjoint of the population Jacobian, defined by
$\langle\mathcal{J}_{{\boldsymbol{\theta}},n}\mathbf{U},f\rangle_{\mathcal H_n}=\langle
\mathbf{U},\mathcal{J}_{{\boldsymbol{\theta}},n}^{*}f
\rangle_{\Theta}$.
Equivalently,
$\mathcal J_{{\boldsymbol{\theta}},n}^{*}f=\mathbb E\bigl[J_{{\boldsymbol{\theta}},n}(\mathbf{z})^{*}f(\mathbf{z})\bigr]$.

At an exact realization $\bar{\boldsymbol{\theta}}\in\mathcal M_\tau$, define the population Gauss-Newton operator
$G(\bar{\boldsymbol{\theta}}):=\mathcal {J}_{\bar{\boldsymbol{\theta}},n}^{*}\mathcal{J}_{\bar{\boldsymbol{\theta}},n}$. For every parameter direction $\mathbf{U}$,
$$\langle \mathbf{U},G(\bar{\boldsymbol{\theta}})\mathbf{U}\rangle=|\mathcal J_{\bar{\boldsymbol{\theta}},n}\mathbf{U}|_{\mathcal H_n}^2=\mathbb E|J_{\bar{\boldsymbol{\theta}},n}(\mathbf{z})[\mathbf{U}]|_2^2.
$$
Therefore $G(\bar{\boldsymbol{\theta}})$ quantifies how strongly the Transformer output responds, to first order, to perturbations of its parameters around an exact realization.
The minimum normal curvature over the chosen target patch $\mathcal{U}_\tau \subset\M_\tau$ is defined by
$\kappa:=\inf_{\substack{
\bar{\boldsymbol{\theta}}\in\mathcal U_\tau,\U\in N_{\bar{\boldsymbol{\theta}}}\mathcal M_\tau,|\mathbf{U}|=1}}|\mathcal J_{\bar{\boldsymbol{\theta}},n}\mathbf{U}|_{\mathcal H_n}^{2}$.

The quantity $\kappa$ has a direct optimization interpretation. A large $\kappa$ means that every unit perturbation away from the exact-realization manifold produces a substantial change in the Transformer output, so the population loss has appreciable curvature in every identifiable parameter direction. A small $\kappa$ means that at least one identifiable direction is nearly flat and is therefore harder to learn by training. 

Throughout the remainder of Appendix \ref{app:Theoretical}, \(n\) and \(\tau\) are fixed. We use the abbreviations \(J_\theta=J_{\theta,n}\), \(G=G_n\), \(\mathcal H=\mathcal H_n\), and \(\mathsf H=\mathsf H_n(\tau)\).
Below we mainly analyze the geometric and curvature information of the optimal parameter manifold, as detailed in Lemmas \ref{lem:Structure}, \ref{lem:uniform-normal-curvature} and \ref{lem:local-tubular-geometry} below.

\paragraph{Structure of the optimal parameter manifold.} The structure of the optimal parameter manifold establishes that every parameter on the manifold exactly realizes the target update and achieves zero loss, that the collection of such optimal parameters forms a smooth embedded manifold, and that the Hessian $\nabla^2L_n(\bar{\boldsymbol{\theta}})$ is degenerate exactly along the tangent directions of this manifold, which means motion within $\mathcal M_\tau$ preserves zero loss, whereas every nonzero normal direction has strictly positive Hessian curvature. 

\begin{lemma}[Structure of the optimal parameter manifold]\label{lem:Structure}
For every $\bar{\boldsymbol{\theta}}=\bar{\boldsymbol{\theta}}_\tau(L,c)\in\mathcal M_\tau$, $F_{\bar{\boldsymbol{\theta}},n}(\mathbf{z})=F_{\tau,n}(\mathbf{z})$ for every admissible input $\mathbf{z}$, and 

(i) $L_n(\bar{\boldsymbol{\theta}})=0
=\min_{{\boldsymbol{\theta}}\in\Theta}L_n({\boldsymbol{\theta}})$, and $\mathcal M_\tau
\subseteq
\arg\min_{{\boldsymbol{\theta}}\in\Theta}L_n({\boldsymbol{\theta}})$. 

(ii) $\mathcal M_\tau$ is a closed smooth embedded submanifold of $\Theta$, $\dim(\mathcal M_\tau)=d^2+1$.

(iii) Under Assumption \ref{assume}, $\ker G_n(\bar{\boldsymbol{\theta}})=T_{\bar{\boldsymbol{\theta}}}\mathcal M_\tau$ and $\nabla^2L_n(\bar{\boldsymbol{\theta}})=G_n(\bar{\boldsymbol{\theta}})$, where $\ker T$ denotes the set of directions that the operator $T$ maps to zero, i.e., the directions along which $T$ has no effect.  
\end{lemma}
\begin{proof}
(i) Since the population loss is
 $L_n({\boldsymbol{\theta}})=\frac12\mathbb E\left[\|F_{{\boldsymbol{\theta}},n}(\mathbf{z})-F_{\tau,n}(\mathbf{z})\|_2^2\right]$, Theorem \ref{Th:Qlearning1} immediately gives $L_n(\bar{\boldsymbol{\theta}})=0$ and $L_n({\boldsymbol{\theta}})\ge0$ for every ${\boldsymbol{\theta}}\in\Theta$. Therefore
$L_n(\bar{\boldsymbol{\theta}})=0=\min_{{\boldsymbol{\theta}}\in\Theta}L_n({\boldsymbol{\theta}})$. Since $\bar{\boldsymbol{\theta}}\in\mathcal M_\tau$ was arbitrary, $\mathcal M_\tau\subseteq
\arg\min_{{\boldsymbol{\theta}}\in\Theta}L_n({\boldsymbol{\theta}})$. 

(ii) Define $\mathcal G:=GL(d)\times (\mathbb R\setminus\{0\})$. Since $GL(d)=\{L\in\mathbb R^{d\times d}:\det(L)\neq 0\}$, and the determinant map is continuous, $GL(d)$ is an open subset of
$\mathbb R^{d\times d}\simeq\mathbb R^{d^2}$.
Likewise, $\mathbb R\setminus\{0\}$ is open in $\mathbb R$. Therefore $\mathcal G\subset\mathbb R^{d^2+1}$ is an open smooth manifold of dimension $\dim(\mathcal G)=d^2+1$.

Define the parameterization $\Phi_\tau:\mathcal G\to\Theta$,
$\Phi_\tau(L,c):=\bar{{\boldsymbol{\theta}}}_{\tau}(L,c)$. Then, by definition, $\mathcal M_\tau=\Phi_\tau(\mathcal G)$.
Matrix inversion $L\mapsto L^{-1}$ is smooth on $GL(d)$, while the scalar map $c\mapsto c^{-1}$ is smooth on $\mathbb R\setminus\{0\}$. Consequently, $(L,c)\mapsto c\gamma L^{-1}$, $(L,c) \mapsto \frac{\alpha}{c}I_d$ are smooth, and hence $\Phi_\tau$ is a smooth map.

In order to prove $\mathcal M_\tau = \Phi_\tau(\mathcal G)$ is an embedded submanifold of $\Theta$, it suffices to verify that $\Phi_\tau$ is an injective immersion which is a homeomorphism onto its image.
Since $\Phi_\tau(L_1,c_1)
=
\Phi_\tau(L_2,c_2) \rightarrow 
(L_1,c_1)=(L_2,c_2)$, $\Phi_\tau$ is injective. For a perturbation $(Z,\xi)\in
\mathbb R^{d\times d}\times\mathbb R
\simeq
T_{(L,c)}\mathcal G$, where $T_{(L,c)}\mathcal G$ is the tangent space of $\mathcal G$ at the point $(L,c)$, the differential of $\Phi_\tau$ is
\begin{align*}
D\Phi_\tau(L,c)[Z,\xi]=\Big(
0,\,Z,\,0,\,-\xi I_d,\,0,\gamma\xi L^{-\top}-c\gamma L^{-\top}Z^\top L^{-\top},\,\xi,\,0,\,-\alpha\xi c^{-2}I_d\Big).
\end{align*}
This follows from differentiating $L^{-\top}L^\top=I_d$ in the direction $Z$, which gives $D(L^{-\top})[Z]=-L^{-\top}Z^\top L^{-\top}$. 
Suppose $D\Phi_\tau(L,c)[Z,\xi]=0$, we have $(Z,\xi)=(0,0)$, which means $D\Phi_\tau(L,c)$ is injective at every
$(L,c)\in\mathcal G$, i.e. $\Phi_\tau$ is an immersion.
For any ${\boldsymbol{\theta}}
=
(A,L,b_-,K_-,b_+,K_+,\beta,q,C)
\in\mathcal M_\tau$, 
The inverse parameterization on $\mathcal M_\tau$ is simply
$\Phi_\tau^{-1}({\boldsymbol{\theta}})
=(L,\,\beta)$. Since coordinate projections are continuous,
$\Phi_\tau^{-1}$ is continuous on its image. Therefore, $\mathcal M_\tau = \Phi_\tau(\mathcal G)$ is an embedded submanifold of $\Theta$, and $\dim(\mathcal M_\tau)=d^2+1$ by the fact that an embedding preserves dimension.

It remains to show that $\mathcal M_\tau$ is closed as a subset of $\Theta$. Consider an arbitrary convergent sequence ${\boldsymbol{\theta}}_j=\bar{\boldsymbol{\theta}}_\tau(L_j,c_j)
\in\mathcal M_\tau$, ${\boldsymbol{\theta}}_j\to{\boldsymbol{\theta}}_\star$ in $\Theta$. 
We show that ${\boldsymbol{\theta}}_\star\in\mathcal M_\tau$.
Since convergence in $\Theta$ implies convergence of every parameter block,  
$L_j \rightarrow L_\star$. 
Since the seventh block satisfies
$\beta_j=c_j$, $c_j\rightarrow c_\star:=\beta_\star$. We need to verify that $c_\star\neq0$ and $L_\star\in GL(d)$. 

Suppose, for contradiction, that $c=0$, then $c_j\to0$. However, the ninth parameter block is $C_j=\frac{\alpha}{c_j}I_d$. 
Since $\alpha>0$, $\|C_j\|_F =\frac{\alpha}{|c_j|}\|I_d\|_F
=\frac{\alpha\sqrt d}{|c_j|}\rightarrow\infty$. 
This contradicts the assumed convergence of ${\boldsymbol{\theta}}_j$ in the finite-dimensional parameter space $\Theta$, thus $c_\star\neq0$. 

Similarly, suppose that $L_\star$ is singular. Since $L_j\in GL(d)$, $L_j\to L_\star$, with $L_\star$ singular, the smallest singular value satisfies
$\sigma_{\min}(L_j)\to0$.  
Consequently,
$\|L_j^{-1}\|_{\mathrm{op}}
=
\frac{1}{\sigma_{\min}(L_j)}
\rightarrow\infty$. 
On the other hand, $c_j\to c_\star\neq0$, so for sufficiently large $j$, $|c_j|\ge \frac{|c_\star|}{2}>0$. 
The sixth parameter block is $K_{+,j}=c_j\gamma L_j^{-\top}$. 
Therefore $\|K_{+,j}\|_{\mathrm{op}}
=\gamma|c_j|\,\|L_j^{-\top}\|_{\mathrm{op}}=\frac{\gamma|c_j|}{\sigma_{\min}(L_j)}\longrightarrow\infty$, contradicting convergence of the sixth block of ${\boldsymbol{\theta}}_j$, thus $L_\star\in GL(d)$. Since $L_\star\in GL(d)$ and $c_\star\neq0$, the pair $(L_\star,c_\star)$ belongs to
$\mathcal G$. Moreover, continuity of matrix inversion on $GL(d)$
gives $L_j^{-1}\to L_\star^{-1}$,
and continuity of $c_\star\mapsto1/c_\star$ gives $\frac1{c_j}\to\frac1c_\star$. 

Therefore every block of ${\boldsymbol{\theta}}_j$ converges to the corresponding block of
$\bar{\boldsymbol{\theta}}_\tau(L,c)$,
so that
${\boldsymbol{\theta}}_\star
=
\bar{\boldsymbol{\theta}}_\tau(L_\star,c_\star)
\in\mathcal M_\tau$.
Hence every convergent sequence in $\mathcal M_\tau$ whose limit
exists in $\Theta$ has its limit again in $\mathcal M_\tau$. Since
$\Theta$ is finite dimensional, this proves that $\mathcal M_\tau$
is closed.

(iii) We prove $\ker G(\bar{\boldsymbol{\theta}})=T_{\bar{\boldsymbol{\theta}}}\mathcal M_\tau$  and $\nabla^2L_n(\bar{\boldsymbol{\theta}})=G(\bar{\boldsymbol{\theta}}):=J_{\bar{\boldsymbol{\theta}}}^*J_{\bar{\boldsymbol{\theta}}}$ in the following four steps.

\textbf{Step 1}: Tangent space of the optimal parameter manifold $\M_\tau$.

Consider a smooth curve
$t\mapsto (L(t),c(t))$
in $GL(d)\times(\mathbb R\setminus\{0\})$ 
such that $L(0)=L$, $c(0)=c$. 
Write $Z:=\dot L(0)$, $\xi:=\dot c(0)$. Differentiating $L(t)^{-1}L(t)=I_d$ at $t=0$ gives
\[
\frac{d}{dt}L(t)^{-\top}\Big|_{t=0}
=
-L^{-\top}Z^\top L^{-\top}.
\]
Hence differentiating $\bar{\boldsymbol{\theta}}_\tau(L(t),c(t))$  gives
\begin{align*}
&\Delta A =0,\quad
\Delta L= Z,\quad
\Delta b_-=0,\quad\Delta K_-=-\xi I_d,
\\
&\Delta b_+=0,\quad
\Delta q=0,\quad
\Delta\beta=\xi,\quad
\Delta C=-\alpha\xi c^{-2}I_d,\\
&\Delta K_+=\frac{d}{dt}\left[\gamma c(t)L(t)^{-\top}\right]_{t=0}=\gamma\xi L^{-\top}-c\gamma L^{-\top}Z^\top L^{-\top}.
\end{align*}
Therefore
\begin{align}
T_{\bar{\boldsymbol{\theta}}}\mathcal M_\tau
=
\Big\{
\big(
0, Z,0,-\xi I_d,0,
\gamma\xi L^{-\top}-c\gamma L^{-\top}Z^\top L^{-\top},
\xi,0,-\alpha\xi c^{-2}I_d
\big):
Z\in\mathbb R^{d\times d},\
\xi\in\mathbb R
\Big\}.
\label{eq:tangent-gauge-space}
\end{align}
These are precisely the infinitesimal gauge directions.

\textbf{Step 2}: Sample Jacobian and collect the observable parameter combinations.

Recall $p_{i,a}(A)$ and $\bar\psi_i(A)$ defined in \eqref{eq:student-attention}. Let
$\Sigma_i^\tau=\sum_a\pi_{i,a}^\tau
(\psi_{i,a}-\bar\psi_i^\tau)(\psi_{i,a}-\bar\psi_i^\tau)^\top$.
The softmax derivative gives $Dp_{i,a}(\tau^{-1}I_d)[\Delta A]
=
\pi_{i,a}^\tau
(\psi_{i,a}-\bar\psi_i^\tau)^\top
\Delta A\,\mathbf{w}$. Consequently,
\begin{align}
D\bar\psi_i(\tau^{-1}I_d)[\Delta A]
=
\sum_a
\psi_{i,a}
Dp_{i,a}(\tau^{-1}I_d)[\Delta A]=
\Sigma_i^\tau\Delta A\,\mathbf{w}.
\label{eq:d-psi-bar}
\end{align}
Recall $s_i({\boldsymbol{\theta}};\mathbf{z})=\phi_i^\top(b_-+K_-\mathbf{w})+b_+^\top L\bar\psi_i(A)+\mathbf{w}^\top K^\top_+L\bar\psi_i(A)+r_{i+1}(\beta+q^\top \mathbf{w})$ defined in \eqref{eq:score}.
At $\bar\theta=\bar\theta_\tau(L,c)$, we have $b_+=0$ and $K_+=c\gamma L^{-\top}$. For an arbitrary perturbation $\mathbf{U}$, write $\xi:=\Delta\beta$.
Using the product rule and
\eqref{eq:d-psi-bar}, we obtain
\begin{align}
\dot s_i[\mathbf{U}]
=&
\phi_i^\top
(\Delta b_-+\Delta K_-\mathbf{w})+
(\bar\psi_i^\tau)^\top
\left[
L^\top\Delta b_+
+
\left(
 L^\top\Delta K_+ 
+
c\gamma\Delta L^\top L^{-\top}
\right)\mathbf{w}
\right]
\nonumber\\
&+
r_{i+1}
\left(
\xi+\Delta q^\top \mathbf{w}
\right)
+
\gamma c\,
\mathbf{w}^\top\Sigma_i^\tau\Delta A\,\mathbf{w}.
\label{eq:score-differential}
\end{align}
Since $F_{{\boldsymbol{\theta}},n}(\mathbf{z})=\frac1n\sum_iC\phi_i s_i({\boldsymbol{\theta}};\mathbf{z})$ and 
$C=\frac{\alpha}{c}I_d$, 
$s_i(\bar{\boldsymbol{\theta}};\mathbf{z})=c\delta_i^\tau$, 
the sample Jacobian satisfies
\begin{equation}
J_{\bar{\boldsymbol{\theta}}}(\mathbf{z})[\mathbf{U}]
=
\frac1n
\sum_i
\left[
c\Delta C\,\phi_i\delta_i^\tau
+
\frac{\alpha}{c}
\phi_i\dot s_i[\mathbf{U}]
\right].
\label{eq:jacobian-before-factorization}
\end{equation}
Then we collect the observable parameter combinations. 
Define the linear map $Q_{L,c}:\Theta\to\mathcal V$,
$Q_{L,c}\mathbf{U}
=
(E_C,e_q,e_-,E_-,e_+,E_+,E_A)$,
by
\begin{align}
E_C =\frac{c^2}{\alpha}\Delta C+\xi I_d,\quad e_q =\Delta q, \quad
e_- =\Delta b_-, \quad E_- =\Delta K_-+\xi I_d,\label{eq:Ee}\\ 
e_+ =  L^\top\Delta b_+,\quad 
E_A =
\gamma c\,\Delta A, \quad E_+ =
c\gamma\Delta L^\top L^{-\top}
+L^\top\Delta K_+ -\gamma\xi I_d.\label{eq:Q-plus}
\end{align}
Subtracting $\xi\delta_i^\tau$ from
\eqref{eq:score-differential}, where 
$\delta_i^\tau
=
r_{i+1}
-\mathbf{w}^\top\phi_i
+
\gamma \mathbf{w}^\top\bar\psi_i^\tau$,
gives
\begin{align*}
\dot s_i[\mathbf{U}]-\xi\delta_i^\tau
=
\phi_i^\top(e_-+E_-\mathbf{w})
+
(\bar\psi_i^\tau)^\top
(e_++E_+\mathbf{w})
+
r_{i+1}e_q^\top \mathbf{w}
+
\mathbf{w}^\top\Sigma_i^\tau E_A\mathbf{w}.
\end{align*}
Likewise, $c\Delta C=\frac{\alpha}{c}
(E_C-\xi I_d)$. Substitution into
\eqref{eq:jacobian-before-factorization} shows that the terms proportional to
$\xi\delta_i^\tau$ cancel exactly. Thus
\begin{equation}
J_{\bar{\boldsymbol{\theta}}}(\mathbf{z})[\mathbf{U}]
=
\frac{\alpha}{c}
\mathcal B_{\tau,n}(\mathbf{z})
[Q_{L,c}\mathbf{U}].
\label{eq:exact-J-factorization}
\end{equation}

\textbf{Step 3}: The kernels of $Q_{L,c}$ and $G_n(\bar{\boldsymbol{\theta}})$ are exactly the tangent space.

Suppose $Q_{L,c}\mathbf{U}=0$.  Then
$E_A=e_q=e_-=e_+=0$. 
Since $\gamma>0$, $c\neq0$, and $L$ is invertible, this implies $\Delta A=0$, $\Delta q=0$, $\Delta b_-=0$, 
$\Delta b_+=0$. 
Moreover, $E_C=0$ implies $\Delta C=-\alpha\xi c^{-2}I_d$, while $E_-=0$ 
gives $\Delta K_-=-\xi I_d$. 
Finally, $E_+=0$  gives $ L^\top\Delta K_+=\gamma\xi I_d-
c\gamma\Delta L^\top L^{-\top}$.
Then $\Delta K_+=\gamma\xi L^{-\top}-{c}{\gamma}L^{-\top}\Delta L^{\top} L^{-\top}$.
Set $Z:=\Delta L$.
Together with the preceding identities, this is precisely the tangent
form in \eqref{eq:tangent-gauge-space}. Hence
$\ker Q_{L,c}\subset T_{\bar{\boldsymbol{\theta}}}\mathcal M_\tau$. 

Conversely, take any
$\mathbf{U}\in T_{\bar{\boldsymbol{\theta}}}\mathcal M_\tau$. 
Then for some $Z$ and $\xi$,  
direct substitution of \eqref{eq:tangent-gauge-space} into
\eqref{eq:Ee}-\eqref{eq:Q-plus}
gives $Q_{L,c}\mathbf{U}=0$. 
Therefore $T_{\bar{\boldsymbol{\theta}}}\mathcal M_\tau \subset\ker Q_{L,c}$, which means
\begin{equation}
\ker Q_{L,c}
=
T_{\bar{\boldsymbol{\theta}}}\mathcal M_\tau.
\label{eq:ker-Q-tangent}
\end{equation}
Since $\mathcal H
=
L^2(\mathcal D_n;\mathbb R^d)$
is the population output space, and 
$G(\bar{\boldsymbol{\theta}})
=
J_{\bar{\boldsymbol{\theta}}}^*J_{\bar{\boldsymbol{\theta}}}$, then, for every perturbation $\mathbf{U}$, by \eqref{eq:exact-J-factorization}, we have
\begin{align}
\langle \mathbf{U},G(\bar{\boldsymbol{\theta}})\mathbf{U}\rangle =\|J_{\bar{\boldsymbol{\theta}}}\mathbf{U}\|_{\mathcal H}^2 =\frac{\alpha^2}{c^2}\,\mathbb E\left[\|\mathcal B_{\tau,n}(\mathbf{z})[Q_{L,c}\mathbf{U}]\|_2^2\right] =\frac{\alpha^2}{c^2}
\left\langle Q_{L,c}\mathbf{U},\,H\,Q_{L,c}\mathbf{U}
\right\rangle_{\mathcal V},
\label{eq:G-factorization}
\end{align}
where $H=\mathbb E[\mathcal B_{\tau,n}(\mathbf{z})^*
\mathcal B_{\tau,n}(\mathbf{z})]$.
By Assumption~\ref{assume} (ii), $H\succ0$.  Therefore $\left\langle
Q_{L,c}\mathbf{U}, H Q_{L,c}\mathbf{U}
\right\rangle=0$ if and only if $Q_{L,c}\mathbf{U}=0$. Since $G(\bar{\boldsymbol{\theta}})=J_{\bar{\boldsymbol{\theta}}}^*J_{\bar{\boldsymbol{\theta}}}$ is positive semi-definite,
$$
\langle \mathbf{U},G(\bar{\boldsymbol{\theta}})\mathbf{U}\rangle=0
\quad\Longleftrightarrow\quad
G(\bar{\boldsymbol{\theta}})\mathbf{U}=0.
$$
Combining this with
\eqref{eq:G-factorization} and
\eqref{eq:ker-Q-tangent}, we obtain
\begin{equation}
\ker G(\bar{\boldsymbol{\theta}})
=
\ker Q_{L,c}
=
T_{\bar{\boldsymbol{\theta}}}\mathcal M_\tau.
\label{eq:kernel-G}
\end{equation}

\textbf{Step 4}: Prove $\nabla^2L_n(\bar{\boldsymbol{\theta}})=G(\bar{\boldsymbol{\theta}})$.

Recall that
$L_n({\boldsymbol{\theta}})
=
\frac12
\mathbb E
\|e_{\boldsymbol{\theta}}(\mathbf{z})\|_2^2$, where $e_{\boldsymbol{\theta}}(\mathbf{z}):=F_{{\boldsymbol{\theta}},n}(\mathbf{z})-F_{\tau,n}(\mathbf{z})$. 
Under Assumption~\ref{assume} (i), the data are bounded. On a sufficiently small
bounded parameter neighborhood, softmax smoothness gives common
integrable bounds on the predictor and its first two derivatives.
Hence differentiation under the expectation is valid.
For arbitrary directions $\mathbf{U},\mathbf{V}$,
\begin{align}
D^2L_n({\boldsymbol{\theta}})[\mathbf{U},\mathbf{V}]= \mathbb E\left[\langle J_{\boldsymbol{\theta}}(\mathbf{z})[\mathbf{U}],J_{\boldsymbol{\theta}}(\mathbf{z})[\mathbf{V}]\rangle\right] +\mathbb E
\left[\left\langle e_{\boldsymbol{\theta}}(\mathbf{z}),
D^2F_{{\boldsymbol{\theta}},n}(\mathbf{z})[\mathbf{U},\mathbf{V}]
\right\rangle\right].
\label{eq:hessian-general}
\end{align}
By part~(i), at every correct parameter
$\bar{\boldsymbol{\theta}}\in\mathcal M_\tau$ 
we have the samplewise identity $e_{\bar{\boldsymbol{\theta}}}(\mathbf{z})=0$. Therefore the second term in
\eqref{eq:hessian-general} vanishes identically, and
\begin{align*}
D^2L_n(\bar{\boldsymbol{\theta}})[\mathbf{U},\mathbf{V}] =
\mathbb E\left[\langle
J_{\bar{\boldsymbol{\theta}}}(\mathbf{z})[\mathbf{U}],
J_{\bar{\boldsymbol{\theta}}}(\mathbf{z})[\mathbf{V}]\rangle
\right] =\langle \mathbf{U},J_{\bar{\boldsymbol{\theta}}}^*J_{\bar{\boldsymbol{\theta}}}\mathbf{V}\rangle.
\end{align*}
Hence, as operators,
\begin{equation}
\nabla^2L_n(\bar{\boldsymbol{\theta}})
=J_{\bar{\boldsymbol{\theta}}}^*J_{\bar{\boldsymbol{\theta}}}
=G(\bar{\boldsymbol{\theta}}).
\label{eq:hessian-G}
\end{equation}
Combining \eqref{eq:hessian-G} with
\eqref{eq:kernel-G} yields
\begin{equation*}
\ker\nabla^2L_n(\bar{\boldsymbol{\theta}})=T_{\bar{\boldsymbol{\theta}}}\mathcal M_\tau.
\end{equation*}
Equivalently, if
$\mathbf{U}\notin T_{\bar{\boldsymbol{\theta}}}\mathcal M_\tau$, 
then $\langle \mathbf{U},\nabla^2L_n(\bar{\boldsymbol{\theta}})\mathbf{U}\rangle
>0$.  In particular, for every nonzero normal direction $\mathbf{U}\in N_{\bar{\boldsymbol{\theta}}}\mathcal M_\tau\setminus\{0\}$, 
we have $\langle \mathbf{U},\nabla^2L_n(\bar{\boldsymbol{\theta}})\mathbf{U}\rangle>0$.
\end{proof}

\paragraph{Uniform normal curvature of the optimal parameter manifold}
By Lemma \ref{lem:Structure}, we have $\nabla^2L_n(\bar{\boldsymbol{\theta}})=G(\bar{\boldsymbol{\theta}})$. It follows that the curvature $\kappa$ defined in \eqref{eq:kappa} characterizes the weakest second-order growth of the loss when moving away from the optimal manifold.

Lemma \ref{lem:uniform-normal-curvature} shows that the loss has uniformly positive curvature in all normal, non-gauge directions around the optimal parameter manifold, and that this curvature is quantitatively comparable to the data-richness constant: $\kappa \asymp \chi:=\lambda_{\min}(\mathsf H_n(\tau))$, where $\mathsf H:=\mathbb E\left[
\mathcal B_{\tau,n}(\mathbf{z})^*\mathcal B_{\tau,n}(\mathbf{z})\right]$. Thus, stronger data identifiability leads to stronger local curvature and faster contraction toward the optimal manifold.

Recall that $\overline{\boldsymbol{\Omega}}
\Subset GL(d)\times(\mathbb R\setminus\{0\})$, $\mathcal{U}_\tau:=\left\{\bar{\boldsymbol{\theta}}_\tau(L,c):(L,c)\in\boldsymbol{\Omega}\right\}\subset\mathcal M_\tau$.

\begin{lemma}[Uniform normal curvature]
\label{lem:uniform-normal-curvature}
Under Assumptions~\ref{assume}, 
there exists constants 
$0<c_\Omega\le C_\Omega<\infty$ 
such that, uniformly over $\bar{\boldsymbol{\theta}}=\bar{\boldsymbol{\theta}}_\tau(L,c)\in \mathcal{U}_\tau$,
$c_\Omega\chi\le\lambda_{\min}(\nabla^2L_n(\bar{\boldsymbol{\theta}})|_{N_{\bar{\boldsymbol{\theta}}}\mathcal M_\tau}
)\le C_\Omega\chi$. 
Equivalently, 
$0<c_\Omega\chi\le\kappa\le C_\Omega\chi$.
\end{lemma}

\begin{proof}
Fix $\bar{\boldsymbol{\theta}}=\bar{\boldsymbol{\theta}}_\tau(L,c)\in \mathcal{U}_\tau$. From \eqref{eq:exact-J-factorization},  $J_{\bar{\boldsymbol{\theta}}}(\mathbf{z})[\mathbf{U}]=\frac{\alpha}{c}\mathcal B_{\tau,n}(\mathbf{z})
[Q_{L,c}\mathbf{U}]$, where $Q_{L,c}: T_{\bar{\boldsymbol{\theta}}}\Theta\rightarrow\mathcal V$ is the linear map collecting the effective observable coefficients.
Moreover, $\ker Q_{L,c}=T_{\bar{\boldsymbol{\theta}}}\mathcal M_\tau$.

We first show that the restriction
$Q_{L,c}|_{N_{\bar{\boldsymbol{\theta}}}\mathcal M_\tau}: N_{\bar{\boldsymbol{\theta}}}\mathcal M_\tau\rightarrow\mathcal V$ is an isomorphism (a bijective linear map that preserves linear operations).
Since $N_{\bar{\boldsymbol{\theta}}}\mathcal M_\tau\cap\ker Q_{L,c}=\{0\}$, the restriction is injective.
Furthermore, $\dim N_{\bar{\boldsymbol{\theta}}}\mathcal M_\tau=
(5d^2+3d+1)-(d^2+1)=4d^2+3d$. The dimension of observable coefficient space is $\dim\mathcal V=4d^2+3d$, thus the restriction is bijective. For completeness, one may construct an explicit right inverse $R_{L,c}:\mathcal V\to T_{\bar{\boldsymbol{\theta}}}\Theta$. Specifically, for $\zeta=(E_C,e_q,e_-,E_-,e_+,E_+,E_A)\in\mathcal V$,
set
\begin{align*}
\Delta A&=(\gamma c)^{-1}E_A,&\Delta L&=0,&\Delta\beta&=0,\\
\Delta C&=\alpha c^{-2}E_C,&\Delta b_-&=e_-,& \Delta K_-&=E_-,\\
\Delta q&=e_q,&\Delta b_+&=L^{-\top}e_+,&\Delta K_+&= L^{-\top}E_+.
\end{align*}
Then $Q_{L,c}R_{L,c}=I_{\mathcal V}$. Let
$P_N$ denote the orthogonal projection onto
$N_{\bar{\boldsymbol{\theta}}}\mathcal M_\tau$.  Since \(Q_{L,c}\) vanishes on
$T_{\bar{\boldsymbol{\theta}}}\mathcal M_\tau$, $Q_{L,c}P_N
=Q_{L,c}$. Therefore $P_NR_{L,c}$ is the inverse of $Q_{L,c}|_{N_{\bar{\boldsymbol{\theta}}}\mathcal M_\tau}$.

Since $\overline{\boldsymbol{\Omega}}$ is compact and $L,\ L^{-1},\ c,\ c^{-1}$ are uniformly bounded on $\overline{\boldsymbol{\Omega}}$, the operator norms of
$Q_{L,c}$ and $R_{L,c}$ are uniformly bounded. Hence there exist
constants $0<t_-\le t_+<\infty$ such that, for every $\bar{\boldsymbol{\theta}}\in \mathcal{U}_\tau$ 
and every $\mathbf{U}\in N_{\bar{\boldsymbol{\theta}}}\mathcal M_\tau$, 
\begin{equation}
t_-\|\mathbf{U}\|
\le
\|Q_{L,c}\mathbf{U}\|_{\mathcal V}
\le
t_+\|\mathbf{U}\|.
\label{eq:A2-norm-equiv}
\end{equation}
where one may take
$t_-=(\sup_{(L,c)\in\overline{\boldsymbol{\Omega}}}
\|R_{L,c}\|_{\mathrm{op}})^{-1}$, $t_+=\sup_{(L,c)\in\overline{\boldsymbol{\Omega}}}\|Q_{L,c}\|_{\mathrm{op}}$.

We now transfer Assumption \ref{assume} from coefficient space to parameter space. 
Since $\mathsf H\succeq \chi I_{\mathcal V}$,
we can obtain $\langle Q_{L,c}\mathbf{U},\,\mathsf H Q_{L,c}\mathbf{U} \rangle\ge\chi\|Q_{L,c}\mathbf{U}\|_{\mathcal V}^2$ by \eqref{eq:G-factorization}. 
Thus, for every unit normal vector $\mathbf{U}\in N_{\bar{\boldsymbol{\theta}}}\mathcal M_\tau$, $\|\mathbf{U}\|=1$, yields 
$\langle\mathbf{U},G(\bar{\boldsymbol{\theta}})\mathbf{U}\rangle\ge\frac{\alpha^2}{c^2}\chi t_-^2$ from \eqref{eq:A2-norm-equiv}. Let $c_+:=\max_{(L,c)\in \overline{\boldsymbol{\Omega}}}|c|$, then $\frac{1}{c^2}\ge\frac{1}{c_+^2}$,
and hence $\langle \mathbf{U},G (\bar{\boldsymbol{\theta}})\mathbf{U}\rangle\ge\frac{\alpha^2t_-^2}{c_+^2}\chi$. Therefore, by taking $c_\Omega=\frac{\alpha^2t_-^2}{c_+^2}$, we have
\begin{equation}\label{eq:A2-lower}
\lambda_{\min}(G(\bar{\boldsymbol{\theta}}) |_{N_{\bar{\boldsymbol{\theta}}}\mathcal M_\tau} )\ge c_\Omega\chi.
\end{equation} 
Let $\zeta_*\in\mathcal V$, $\|\zeta_*\|_{\mathcal V}=1$, 
be a unit eigenvector of $\mathsf H$ corresponding to its smallest eigenvalue $\mathsf H\zeta_*=\chi\zeta_*$. 
Since $Q_{L,c} |_{N_{\bar{\boldsymbol{\theta}}}\mathcal M_\tau}$ 
is invertible, define 
$\mathbf{U}_* = ( Q_{L,c} |_{N_{\bar{\boldsymbol{\theta}}}\mathcal M_\tau} )^{-1}
\zeta_*$. Then $Q_{L,c}\mathbf{U}_* = \zeta_*$. 
By the upper estimate in \eqref{eq:A2-norm-equiv}, $1=\|Q_{L,c}\mathbf{U}_*\|_{\mathcal V}\le t_+\|\mathbf{U}_*\|$, so $\|\mathbf{U}_*\|\ge\frac1{t_+}$. 
Using \eqref{eq:G-factorization},
\begin{align*}
\langle\widehat{\mathbf{U}}_*,G(\bar{\boldsymbol{\theta}})\widehat{\mathbf{U}}_*
\rangle=\frac{\alpha^2}{c^2}\frac{\langle\zeta_*,\mathsf H\zeta_*\rangle}{\|\mathbf{U}_*\|^2}=\frac{\alpha^2}{c^2}\frac{\chi}{\|\mathbf{U}_*\|^2}\le\frac{\alpha^2t_+^2}{c^2}\chi, \quad \widehat{\mathbf{U}}_*=\frac{\mathbf{U}_*}{\|\mathbf{U}_*\|}
\end{align*}

Let $ c_-:=
\min_{(L,c)\in\overline{\boldsymbol{\Omega}}}|c|>0$.  Then
$\frac1{c^2}\le\frac1{c_-^2}$, and therefore, by taking $C_\Omega=\frac{\alpha^2t_+^2}{c_-^2}$,
\begin{equation}
\lambda_{\min}(G(\bar{\boldsymbol{\theta}})
|_{N_{\bar{\boldsymbol{\theta}}}\mathcal M_\tau}) \le C_\Omega\chi.
\label{eq:A2-upper}
\end{equation}
Finally, by Lemma~\ref{lem:Structure}, $\nabla^2L_n(\bar{\boldsymbol{\theta}}) = G(\bar{\boldsymbol{\theta}})$ for every $\bar{\boldsymbol{\theta}}\in\mathcal M_\tau$. 
Combining \eqref{eq:A2-lower} and \eqref{eq:A2-upper} gives
\[
c_\Omega\chi\le\lambda_{\min}(
\nabla^2L_n(\bar{\boldsymbol{\theta}})|_{N_{\bar{\boldsymbol{\theta}}}\mathcal M_\tau})
\le C_\Omega\chi.
\]

Taking the infimum over $\bar{\boldsymbol{\theta}}\in \mathcal{U}_\tau$ 
and over unit normal directions yields $c_\Omega\chi \le \kappa \le C_\Omega\chi$. 
Since $\chi>0$ by Assumption~\ref{assume} and
 $c_\Omega>0$,  we conclude that $\kappa>0$. 
\end{proof}

\paragraph{Local tubular geometry and unique projection.} Lemma~\ref{lem:local-tubular-geometry} provides a precise local notion of distance to the optimal parameter manifold. 

For any parameter $\boldsymbol{\theta}$ sufficiently close to $\mathcal M_\tau$, define, 
$$d(\boldsymbol{\theta})
:=
\operatorname{dist}(\boldsymbol{\theta},\mathcal M_\tau)
=
\inf_{\boldsymbol{\vartheta}\in\mathcal M_\tau}
\|\boldsymbol{\theta}-\boldsymbol{\vartheta}\|.$$
Lemma~\ref{lem:local-tubular-geometry} guarantees that this distance is attained at a unique nearest point $\bar{\boldsymbol{\theta}}=\Pi(\boldsymbol{\theta})$. Moreover, the displacement $\Delta=\boldsymbol{\theta}-\Pi(\boldsymbol{\theta})$ is normal to the manifold, i.e. $d(\boldsymbol{\theta}) = \|\boldsymbol{\theta}-\Pi(\boldsymbol{\theta})\|$.

\begin{lemma}[Local tubular geometry and unique projection]
\label{lem:local-tubular-geometry}
There exists a constant $r_{\rm geo}>0$ such that the following statements hold. 

(i) For every $\bar{\boldsymbol{\theta}}\in\overline{\mathcal U}_\tau$ and $\Delta\in N_{\bar{\boldsymbol{\theta}}}\mathcal M_\tau$ with $\|\Delta\|<r_{\rm geo}$, the point $\boldsymbol{\theta}=\bar{\boldsymbol{\theta}}+\Delta$ has $\bar{\boldsymbol{\theta}}$ as its unique nearest point on the full manifold $\mathcal M_\tau$. 

(ii) Define the normal tube $\mathcal T_{r_{\rm geo}}(\mathcal U_\tau) = \left\{ \bar{\boldsymbol{\theta}}+\Delta: \bar{\boldsymbol{\theta}}\in\mathcal U_\tau,\ \Delta\in N_{\bar{\boldsymbol{\theta}}}\mathcal M_\tau,\ \|\Delta\|<r_{\rm geo} \right\}$. For every $\boldsymbol{\theta}\in \mathcal T_{r_{\rm geo}}(\mathcal U_\tau)$, the nearest-point projection $\Pi(\boldsymbol{\theta})$ onto $\mathcal M_\tau$ is uniquely defined and satisfies $\boldsymbol{\theta}-\Pi(\boldsymbol{\theta}) \in N_{\Pi(\boldsymbol{\theta})}\mathcal M_\tau$ and $d(\boldsymbol{\theta}) = \|\boldsymbol{\theta}-\Pi(\boldsymbol{\theta})\|$. Moreover, $\mathcal T_{r_{\rm geo}}(\mathcal U_\tau)$ is an open neighborhood of $\mathcal U_\tau$ in the ambient parameter space.
\end{lemma}

\begin{proof}
By Lemma~\ref{lem:Structure}, $\mathcal M_\tau$ is a smooth closed
embedded submanifold of the finite-dimensional Euclidean parameter
space, and $\overline{\mathcal U}_\tau$ is compact. By the tubular
neighborhood theorem~\cite[Theorem~6.24]{lee2013smooth}, every point
sufficiently close to $\mathcal M_\tau$ can be uniquely represented as
$\boldsymbol{\theta}=\bar{\boldsymbol{\theta}}+\Delta$, where
$\bar{\boldsymbol{\theta}}\in\mathcal M_\tau$ and
$\Delta\in N_{\bar{\boldsymbol{\theta}}}\mathcal M_\tau$ is a
sufficiently small normal displacement.

For each
$\bar{\boldsymbol{\theta}}\in\overline{\mathcal U}_\tau$, there exist
a neighborhood
$\mathcal V_{\bar{\boldsymbol{\theta}}}\subset\mathcal M_\tau$
and a radius $r_{\bar{\boldsymbol{\theta}}}>0$ such that the above
normal representation remains unique for all normal perturbations of
size smaller than $r_{\bar{\boldsymbol{\theta}}}$ from points in
$\mathcal V_{\bar{\boldsymbol{\theta}}}$. Since
$\overline{\mathcal U}_\tau$ is compact, finitely many such
neighborhoods, say $\mathcal V_1,\ldots,\mathcal V_L$, cover
$\overline{\mathcal U}_\tau$. Define
$\mathcal V=\bigcup_{j=1}^L\mathcal V_j$ and
$r_0=\min_{1\le j\le L}r_j$. Since the minimum is taken over finitely
many positive radii, $r_0>0$. Thus every normal perturbation of norm
smaller than $r_0$ from a point in $\mathcal V$ lies in the region
where the normal representation is unique.

(i) We choose the radius small enough so that the nearest point cannot
leave $\mathcal V$. Since
$\overline{\mathcal U}_\tau\subset\mathcal V$,
$\overline{\mathcal U}_\tau$ is compact, and
$\mathcal M_\tau\setminus\mathcal V$ is closed, we have $\delta:=\operatorname{dist}
(\overline{\mathcal U}_\tau,\mathcal M_\tau\setminus\mathcal V)>0$.
Choose $r_{\rm geo}>0$ such that
$r_{\rm geo}<\min\{r_0,\delta/2\}$.
Now take
$\bar{\boldsymbol{\theta}}\in\overline{\mathcal U}_\tau$ and
$\Delta\in N_{\bar{\boldsymbol{\theta}}}\mathcal M_\tau$ with
$\|\Delta\|<r_{\rm geo}$, and set
$\boldsymbol{\theta}=\bar{\boldsymbol{\theta}}+\Delta$.
Since $\mathcal M_\tau$ is closed, there exists at least one nearest
point $\boldsymbol{\theta}^\star\in\mathcal M_\tau$ to
$\boldsymbol{\theta}$. Because
$\bar{\boldsymbol{\theta}}\in\mathcal M_\tau$,
$$
\|\boldsymbol{\theta}-\boldsymbol{\theta}^\star\|
\le
\|\boldsymbol{\theta}-\bar{\boldsymbol{\theta}}\|
=
\|\Delta\|
<
r_{\rm geo}.
$$
The triangle inequality then gives
$\|\boldsymbol{\theta}^\star-\bar{\boldsymbol{\theta}}\|
<2r_{\rm geo}<\delta$. Hence
$\boldsymbol{\theta}^\star\in\mathcal V$.

Since $\boldsymbol{\theta}^\star$ minimizes the Euclidean distance
from $\boldsymbol{\theta}$ to $\mathcal M_\tau$, the first-order
optimality condition implies
$\boldsymbol{\theta}-\boldsymbol{\theta}^\star
\in
N_{\boldsymbol{\theta}^\star}\mathcal M_\tau$.
Therefore $\boldsymbol{\theta}$ has two normal representations:
$\boldsymbol{\theta}=\bar{\boldsymbol{\theta}}+\Delta$ and
$\boldsymbol{\theta}
=\boldsymbol{\theta}^\star+
(\boldsymbol{\theta}-\boldsymbol{\theta}^\star)$.
Both base points lie in $\mathcal V$, and both normal displacements
have norm smaller than $r_0$. By uniqueness of the tubular
representation, we must have
$\boldsymbol{\theta}^\star=\bar{\boldsymbol{\theta}}$.
Thus $\bar{\boldsymbol{\theta}}$ is the unique nearest point to
$\boldsymbol{\theta}$ on the full manifold $\mathcal M_\tau$.
Consequently,
$$
\Pi(\boldsymbol{\theta})
=
\bar{\boldsymbol{\theta}},
\qquad
\boldsymbol{\theta}-\Pi(\boldsymbol{\theta})
\in
N_{\Pi(\boldsymbol{\theta})}\mathcal M_\tau,
\qquad
d(\boldsymbol{\theta})
=
\|\boldsymbol{\theta}-\Pi(\boldsymbol{\theta})\|.
$$

(ii) Since $\mathcal U_\tau$ is relatively open in
$\mathcal M_\tau$, the normal perturbations of norm smaller than
$r_{\rm geo}$ based at points of $\mathcal U_\tau$ form an open set
in the tubular coordinates. The tubular-neighborhood map sends this
set to an open subset of the ambient parameter space containing
$\mathcal U_\tau$. Hence the corresponding normal tube is an open
neighborhood of $\mathcal U_\tau$.
\end{proof} 

\subsection{A Direct Parameter-Distance Estimate for One GD Step}
\label{app:one-step}
The identity $d(\boldsymbol{\theta})=\|\Delta\|$ is essential for the convergence analysis, where $\Delta=\boldsymbol{\theta}-\Pi(\boldsymbol{\theta})$ is the normal displacement from the current parameter to its nearest point on the optimal manifold. The operator $\mathcal J_{\Pi(\boldsymbol{\theta}),n}[\Delta]$ describes the first-order change of the model output when the parameter moves from the nearest optimal point $\Pi(\boldsymbol{\theta})$ in the direction $\Delta$. Since $\Delta$ is normal to the manifold, the normal-curvature bound in Lemma~\ref{lem:uniform-normal-curvature} can be applied directly, giving
$$
\left\|
\mathcal J_{\Pi(\boldsymbol{\theta}),n}[\Delta]
\right\|_{L^2(D_n)}^2
\ge
\kappa\|\Delta\|^2
=
\kappa d(\boldsymbol{\theta})^2.
$$
Lemma \ref{lem:one-step-contraction} shows that one gradient-descent step moves the current parameter closer to the optimal manifold. We first prove that the gradient has a sufficiently large component in the direction that reduces the normal displacement from the manifold. We then control the overall gradient norm to ensure that the update does not overshoot. Combining these two estimates yields a one-step contraction of the distance to $\mathcal M_\tau$.

\begin{lemma}[One-step contraction toward the optimal manifold]
\label{lem:one-step-contraction}
Under Assumption~\ref{assume}, 

(i) There exist constants $M_J>0$, $L_J>0$, and $0<\rho_0<r_{\rm geo}$ such that, for all
$\boldsymbol{\theta},\boldsymbol{\theta}'
\in
\mathcal T_{\rho_0}(\mathcal U_\tau)$,
$\|\mathcal J_{\boldsymbol{\theta},n}\|_{\rm op}\le M_J$, and $\|\mathcal J_{\boldsymbol{\theta},n}
-\mathcal J_{\boldsymbol{\theta}',n}\|_{\rm op}\le L_J\|\boldsymbol{\theta}-\boldsymbol{\theta}'\|$. 

(ii) Let $r:=\min\left\{\rho_0,\frac{\sqrt{\kappa}}{4L_J}\right\}$, $\mathcal T:=\mathcal T_r(\mathcal U_\tau)$. For any $\boldsymbol{\theta}\in\mathcal T$ and every step size $0<\eta\le 1/(2M_J^2)$,
$d\left(\boldsymbol{\theta}-\eta\nabla L_n(\boldsymbol{\theta})\right)^2\le\left(1-\frac{\eta\kappa}{4}\right)d(\boldsymbol{\theta})^2
$, and $\eta\left\|\nabla L_n(\boldsymbol{\theta})\right\|\le\frac{4M_J}{\sqrt{\kappa}}\left[d(\boldsymbol{\theta})-d\left(\boldsymbol{\theta}-\eta\nabla L_n(\boldsymbol{\theta})\right)\right]$, where $1-\eta\kappa/4\in[7/8,1)$.
\end{lemma}

\begin{proof}
(i) 
Write the parameter as $\boldsymbol{\theta}=(A,x)$, where $x$ collects
all parameter blocks other than $A$. Since
$\overline{\mathcal U}_\tau$ is compact, there exists $M_0<\infty$
such that $\|x\|\le M_0$ on $\overline{\mathcal U}_\tau$.
Choose $0<\rho_0<\min\{r_{\mathrm{geo}},1\}$ sufficiently small so that
$\mathcal T_{\rho_0}(\mathcal U_\tau)$ and the line segment joining
each point in this tube to its projection onto $\mathcal M_\tau$
are contained in
\begin{equation*}
\mathcal C
=
\left\{
(A,x):
A\in\mathbb R^{d\times d},\
\|x\|<M_0+1
\right\}.
\end{equation*}

Although $A$ is unrestricted in $\mathcal C$, it enters the model
through the softmax logits. The derivatives of the softmax are
uniformly bounded, the features and weights are bounded by
Assumption~\ref{assume}, and all remaining parameter blocks are
bounded on $\mathcal C$. Hence there exist constants $M_J>0$ and
$L_J>0$ such that, for all
$\boldsymbol{\theta},\boldsymbol{\theta}'
\in\mathcal T_{\rho_0}(\mathcal U_\tau)$,
\begin{equation}
\|\mathcal J_{\boldsymbol{\theta},n}\|_{\rm op}
\le M_J,
\qquad
\|\mathcal J_{\boldsymbol{\theta},n}
-
\mathcal J_{\boldsymbol{\theta}',n}\|_{\rm op}
\le
L_J
\|\boldsymbol{\theta}-\boldsymbol{\theta}'\|.
\label{eq:local-jacobian-bounds}
\end{equation}
Since $\kappa$ is the minimum curvature over unit normal directions
on the selected optimal patch, the first bound in
\eqref{eq:local-jacobian-bounds} also gives
$\kappa\le M_J^2$. For a fixed gauge domain \(\Omega\) and temperature-independent data bounds, the constants \(M_J\) and \(L_J\) can be chosen independently of \(\tau\). Indeed, the non-\(A\) coordinates of the realization patch do not depend on \(\tau\), while \(A\) enters the model only through the softmax logits, whose derivatives with respect to those logits are uniformly bounded.

(ii) 
Fix $\boldsymbol{\theta}\in\mathcal T_r(\mathcal U_\tau)$ and let
$\bar{\boldsymbol{\theta}}=\Pi(\boldsymbol{\theta})$ and
$\Delta=\boldsymbol{\theta}-\bar{\boldsymbol{\theta}}$.
By Lemma~\ref{lem:local-tubular-geometry},
$\Delta\in N_{\bar{\boldsymbol{\theta}}}\mathcal M_\tau$ and
$\|\Delta\|=d(\boldsymbol{\theta})$.
If $d(\boldsymbol{\theta})=0$, then
$\boldsymbol{\theta}\in\mathcal M_\tau$ and
$\nabla L_n(\boldsymbol{\theta})=0$, so the conclusions are immediate.
Hence assume $d(\boldsymbol{\theta})>0$.
Define
$\boldsymbol{\theta}^{+}
=
\boldsymbol{\theta}
-
\eta\nabla L_n(\boldsymbol{\theta})$
and $d_{+}:=d(\boldsymbol{\theta}^{+})$.
Since $\bar{\boldsymbol{\theta}}\in\mathcal M_\tau$ is a valid
comparison point in the definition of $d_{+}$,
\begin{align}
d_{+}^2\le\left\|\boldsymbol{\theta}-\eta\nabla L_n(\boldsymbol{\theta})-\bar{\boldsymbol{\theta}}\right\|^2=
d(\boldsymbol{\theta})^2-2\eta
\left\langle\nabla L_n(\boldsymbol{\theta}),\Delta
\right\rangle+\eta^2\|\nabla L_n(\boldsymbol{\theta})\|^2.
\label{eq:distance-expansion}
\end{align}
Thus, it is enough to show that the directional component
$\langle\nabla L_n(\boldsymbol{\theta}),\Delta\rangle$ is uniformly
positive and that the total gradient norm is controlled by this same
quantity. We establish these two bounds below.

For $\left\langle
\nabla L_n(\boldsymbol{\theta}),\Delta
\right\rangle$. Since $\bar{\boldsymbol{\theta}}\in\mathcal M_\tau$,
$F_{\bar{\boldsymbol{\theta}},n}=F_{\tau,n}$, and the fundamental
theorem of calculus gives
\begin{equation}
F_{\boldsymbol{\theta},n}-F_{\tau,n}
=
\int_0^1
\mathcal J_{\bar{\boldsymbol{\theta}}+t\Delta,n}[\Delta]
\,dt.
\label{eq:prediction-error-integral}
\end{equation}
Moreover, for every parameter direction $\mathbf U$,
\begin{equation}
\left\langle
\nabla L_n(\boldsymbol{\theta}),\mathbf U
\right\rangle
=
\left\langle
F_{\boldsymbol{\theta},n}-F_{\tau,n},
\mathcal J_{\boldsymbol{\theta},n}[\mathbf U]
\right\rangle_{L^2(\mathcal D_n)}.
\label{eq:gradient-identity}
\end{equation}

Because $\Delta$ is normal to $\mathcal M_\tau$,
Lemma~\ref{lem:uniform-normal-curvature} gives
$\left\|
\mathcal J_{\bar{\boldsymbol{\theta}},n}[\Delta]
\right\|_{L^2(\mathcal D_n)}
\ge
\sqrt{\kappa}\,d(\boldsymbol{\theta})$.
Using the Jacobian Lipschitz bound in
\eqref{eq:local-jacobian-bounds},
\begin{align*}
\left\|
\mathcal J_{\boldsymbol{\theta},n}[\Delta]
\right\|_{L^2(\mathcal D_n)}\ge\left\|
\mathcal J_{\bar{\boldsymbol{\theta}},n}[\Delta]\right\|_{L^2(\mathcal D_n)}-L_J
\|\boldsymbol{\theta}-\bar{\boldsymbol{\theta}}\|
\|\Delta\|\ge\sqrt{\kappa}\,d(\boldsymbol{\theta})-L_J d(\boldsymbol{\theta})^2.
\end{align*}
Since $d(\boldsymbol{\theta})<r$ and
$r\le\sqrt{\kappa}/(4L_J)$,
\begin{equation}
\left\|
\mathcal J_{\boldsymbol{\theta},n}[\Delta]
\right\|_{L^2(\mathcal D_n)}
\ge
\frac34\sqrt{\kappa}\,
d(\boldsymbol{\theta}).
\label{eq:current-normal-response}
\end{equation}
Using \eqref{eq:prediction-error-integral}, write
\begin{equation}
F_{\boldsymbol{\theta},n}-F_{\tau,n}=\mathcal J_{\boldsymbol{\theta},n}[\Delta]+R,
\label{eq:error-decomposition}
\end{equation}
where $R:=\int_0^1\left(\mathcal J_{\bar{\boldsymbol{\theta}}+t\Delta,n}-\mathcal J_{\boldsymbol{\theta},n}\right)[\Delta]\,dt$. By the Jacobian Lipschitz bound,
\begin{align}
\|R\|_{L^2(\mathcal D_n)}
\le\int_0^1 L_J(1-t)\|\Delta\|^2\,dt=
\frac{L_J}{2}d(\boldsymbol{\theta})^2
\le\frac18\sqrt{\kappa}\,d(\boldsymbol{\theta}).
\end{align}
Combining this with \eqref{eq:current-normal-response} gives $\|R\|_{L^2(\mathcal D_n)}\le
\frac16\|\mathcal J_{\boldsymbol{\theta},n}[\Delta]\|_{L^2(\mathcal D_n)}$.
Taking $\mathbf U=\Delta$ in
\eqref{eq:gradient-identity} and using
\eqref{eq:error-decomposition}, we have
\begin{align}
\left\langle
\nabla L_n(\boldsymbol{\theta}),\Delta
\right\rangle=\left\langle
\mathcal J_{\boldsymbol{\theta},n}[\Delta]+R,\mathcal J_{\boldsymbol{\theta},n}[\Delta]\right\rangle_{L^2(\mathcal D_n)}\ge
\frac56\|\mathcal J_{\boldsymbol{\theta},n}[\Delta]\|_{L^2(\mathcal D_n)}^2.
\label{eq:directional-gradient-pre}
\end{align}
Together with \eqref{eq:current-normal-response}, this yields
\begin{equation}
\left\langle
\nabla L_n(\boldsymbol{\theta}),\Delta
\right\rangle
\ge
\frac{15}{32}
\kappa d(\boldsymbol{\theta})^2.
\label{eq:directional-gradient-bound}
\end{equation}

For $\|\nabla L_n(\boldsymbol{\theta})\|$. By \eqref{eq:error-decomposition}, $\|F_{\boldsymbol{\theta},n}-F_{\tau,n}\|_{L^2(\mathcal D_n)}\le\frac76\|\mathcal J_{\boldsymbol{\theta},n}[\Delta]\|_{L^2(\mathcal D_n)}$. Using \eqref{eq:gradient-identity},
\eqref{eq:local-jacobian-bounds}, and
\eqref{eq:directional-gradient-pre}, we obtain
\begin{equation}
\|\nabla L_n(\boldsymbol{\theta})\|^2
\le
\frac{49}{30}M_J^2
\left\langle
\nabla L_n(\boldsymbol{\theta}),\Delta
\right\rangle.
\label{eq:gradient-norm-directional}
\end{equation}

We now return to the distance expansion
\eqref{eq:distance-expansion}. Using
\eqref{eq:gradient-norm-directional},
\begin{align}
d_{+}^2
&\le
d(\boldsymbol{\theta})^2
-
\eta
\left(
2-\frac{49}{30}\eta M_J^2
\right)
\left\langle
\nabla L_n(\boldsymbol{\theta}),\Delta
\right\rangle.
\label{eq:one-step-distance}
\end{align}
If $\eta\le1/(2M_J^2)$, then
$2-(49/30)\eta M_J^2\ge71/60>1$.
Therefore, by \eqref{eq:directional-gradient-bound},
\begin{align}
d_{+}^2\le d(\boldsymbol{\theta})^2
-\eta\left\langle\nabla L_n(\boldsymbol{\theta}),\Delta\right\rangle\le\left(1-\frac{15}{32}\eta\kappa\right)
d(\boldsymbol{\theta})^2\le\left(1-\frac{\eta\kappa}{4}\right)d(\boldsymbol{\theta})^2.
\label{eq:one-step-contraction}
\end{align}

Since $\kappa\le M_J^2$ and
$\eta\le1/(2M_J^2)$, the contraction factor satisfies
$1-\eta\kappa/4\in[7/8,1)$.

It remains to control the step length. From
\eqref{eq:one-step-distance} and the same step-size condition,
\begin{equation}
\eta
\left\langle
\nabla L_n(\boldsymbol{\theta}),\Delta
\right\rangle
\le
d(\boldsymbol{\theta})^2-d_{+}^2
\le
2d(\boldsymbol{\theta})
\left(
d(\boldsymbol{\theta})-d_{+}
\right),
\label{eq:distance-drop}
\end{equation}
where $d_{+}\le d(\boldsymbol{\theta})$ follows from \eqref{eq:one-step-contraction}.
Furthermore, by
\eqref{eq:directional-gradient-pre},
\eqref{eq:current-normal-response}, and the bound
$\|F_{\boldsymbol{\theta},n}-F_{\tau,n}\|_{L^2(\mathcal D_n)}
\le(7/6)\|\mathcal J_{\boldsymbol{\theta},n}[\Delta]\|_{L^2(\mathcal D_n)}$,
\begin{equation}
\frac{
\|\nabla L_n(\boldsymbol{\theta})\|
}{
\left\langle
\nabla L_n(\boldsymbol{\theta}),\Delta
\right\rangle
}
\le
\frac{28M_J}
{15\sqrt{\kappa}\,d(\boldsymbol{\theta})}.
\label{eq:gradient-direction-ratio}
\end{equation}
Combining \eqref{eq:distance-drop} and
\eqref{eq:gradient-direction-ratio} yields
\begin{align*}
\eta\|\nabla L_n(\boldsymbol{\theta})\|\le
\frac{56M_J}{15\sqrt{\kappa}}\left(d(\boldsymbol{\theta})-d_{+}\right)\le\frac{4M_J}{\sqrt{\kappa}}
\left[d(\boldsymbol{\theta})-d\left(
\boldsymbol{\theta}
-
\eta\nabla L_n(\boldsymbol{\theta})
\right)
\right].
\end{align*}
This proves both the one-step contraction and the step-length estimate.
\end{proof}

\subsection{The proof of Theorem~\ref{thm:main}}
Let
$q_\eta:=1-\eta\kappa/4$.
By Lemma~\ref{lem:one-step-contraction}, if
$0<\eta\le 1/(2M_J^2)$, then $q_\eta\in(0,1)$. 

\paragraph{Convergence of the full parameter sequence and its limit.}
Let $d_k:=d(\boldsymbol{\theta}_k)$. Assume $d_0>0$.
Define $\Lambda_0:=\frac{4M_J}{\sqrt{\kappa}}\,d_0$.
By the initialization condition,
$\Lambda_0<b_{\mathcal T}(\boldsymbol{\theta}_0)$, where
$b_{\mathcal T}(\boldsymbol{\theta})
:=\operatorname{dist}(\boldsymbol{\theta},\mathcal T^c)$.
Therefore the closed ball
$\overline B(\boldsymbol{\theta}_0,\Lambda_0)$ is contained in
$\mathcal T$.

We prove inductively that every iterate and every update segment remain
in this ball and that $d_k\le q_\eta^{k/2}d_0$.
Suppose $\boldsymbol{\theta}_k\in\mathcal T$. By
Lemma~\ref{lem:one-step-contraction},
\begin{equation}
d_{k+1}
\le
\sqrt{q_\eta}\,d_k,
\qquad
\|\boldsymbol{\theta}_{k+1}-\boldsymbol{\theta}_k\|
\le
\frac{4M_J}{\sqrt{\kappa}}
(d_k-d_{k+1}).
\label{eq:proof-one-step-estimates}
\end{equation}
The first inequality gives
$d_{k+1}\le q_\eta^{(k+1)/2}d_0$.

Now consider any point
$\boldsymbol{y}
=
\boldsymbol{\theta}_k
+s(\boldsymbol{\theta}_{k+1}-\boldsymbol{\theta}_k)$,
with $0\le s\le1$, on the new update segment. Using
\eqref{eq:proof-one-step-estimates} and the induction hypothesis,
\begin{align}
\|\boldsymbol{y}-\boldsymbol{\theta}_0\|
&\le
\sum_{j=0}^{k-1}
\|\boldsymbol{\theta}_{j+1}-\boldsymbol{\theta}_j\|
+
s\|\boldsymbol{\theta}_{k+1}-\boldsymbol{\theta}_k\|
\nonumber\\
&\le
\frac{4M_J}{\sqrt{\kappa}}
\sum_{j=0}^{k}(d_j-d_{j+1})
=
\frac{4M_J}{\sqrt{\kappa}}
(d_0-d_{k+1})
\le
\Lambda_0.
\label{eq:proof-segment}
\end{align}
Hence the entire update segment remains in
$\overline B(\boldsymbol{\theta}_0,\Lambda_0)\subset\mathcal T$.
Every iterate remains in
$\mathcal T$, and $d_k\le q_\eta^{k/2}d_0$ holds for all
$k\ge0$.

The decay $d_k\to0$ alone does not guarantee convergence of the full
parameter sequence, since the iterates could in principle drift along
the optimal manifold. The step-length estimate in
\eqref{eq:proof-one-step-estimates} rules out this possibility. For integers $\ell>k$, summing the step-length bound gives
\begin{align}
\sum_{j=k}^{\ell-1}
\|\boldsymbol{\theta}_{j+1}-\boldsymbol{\theta}_j\|
&\le
\frac{4M_J}{\sqrt{\kappa}}
(d_k-d_\ell)
\le
\frac{4M_J}{\sqrt{\kappa}}d_k
\le
\Lambda_0 q_\eta^{k/2}.
\label{eq:proof-tail}
\end{align}
The right-hand side tends to zero as $k\to\infty$. Hence
$(\boldsymbol{\theta}_k)$ is a Cauchy sequence. Since the parameter
space is finite-dimensional, there exists
$\boldsymbol{\theta}_\infty$ such that
$\boldsymbol{\theta}_k\to\boldsymbol{\theta}_\infty$.

Letting $\ell\to\infty$ in \eqref{eq:proof-tail} yields
\begin{equation}
\|\boldsymbol{\theta}_k-\boldsymbol{\theta}_\infty\|
\le
\Lambda_0 q_\eta^{k/2}.
\label{eq:proof-limit-rate}
\end{equation}
Together with $d_k\le q_\eta^{k/2}d_0$, this proves geometric
convergence both to the optimal manifold and to a single parameter
limit. The distance function
$d(\boldsymbol{\theta})
=
\operatorname{dist}(\boldsymbol{\theta},\mathcal M_\tau)$
is $1$-Lipschitz. Since
$\boldsymbol{\theta}_k\to\boldsymbol{\theta}_\infty$ and
$d_k\to0$, we obtain
$d(\boldsymbol{\theta}_\infty)=0$. Because
$\mathcal M_\tau$ is closed,
$\boldsymbol{\theta}_\infty\in\mathcal M_\tau$.

Moreover, all iterates lie in the closed ball
$\overline B(\boldsymbol{\theta}_0,\Lambda_0)$, so the limit also lies
in this ball and therefore in $\mathcal T$.
Since $\mathcal T=\mathcal T_r(\mathcal U_\tau)$, there exist
$\bar{\boldsymbol{\theta}}\in\mathcal U_\tau$ and
$\Delta\in N_{\bar{\boldsymbol{\theta}}}\mathcal M_\tau$,
with $\|\Delta\|<r$, such that
$\boldsymbol{\theta}_\infty=\bar{\boldsymbol{\theta}}+\Delta$.

By Lemma~\ref{lem:local-tubular-geometry},
$\bar{\boldsymbol{\theta}}$ is the unique nearest point of
$\boldsymbol{\theta}_\infty$ on $\mathcal M_\tau$. But
$\boldsymbol{\theta}_\infty\in\mathcal M_\tau$, so
$\boldsymbol{\theta}_\infty$ itself is a nearest point. Uniqueness
therefore implies $\Delta=0$ and
$\boldsymbol{\theta}_\infty=\bar{\boldsymbol{\theta}}\in\mathcal U_\tau$.
Hence there exist $(L_\infty,c_\infty)\in\boldsymbol{\Omega}$ such that
$\boldsymbol{\theta}_\infty=\bar{\boldsymbol{\theta}}_\tau(L_\infty,c_\infty)$.
By exact realization on $\mathcal M_\tau$,
$F_{\boldsymbol{\theta}_\infty,n}=F_{\tau,n}$, and therefore
$L_n(\boldsymbol{\theta}_\infty)=0
=
\min_{\boldsymbol{\theta}\in\Theta}L_n(\boldsymbol{\theta})$.

\paragraph{Iteration-complexity bound.}
By $d_k\le q_\eta^{k/2}d_0$,
$d(\boldsymbol{\theta}_K)\le q_\eta^{K/2}d_0$.
Since $q_\eta=1-\eta\kappa/4$ and $1-x\le e^{-x}$, $d(\boldsymbol{\theta}_K)\le\exp (
-\frac{\eta\kappa K}{8} )d_0$. Thus
$K\ge \frac{8}{\eta\kappa}\log(1/\varepsilon)$
is sufficient to guarantee
$d(\boldsymbol{\theta}_K)\le\varepsilon d_0$. Consequently, $K_{\rm GD}(\varepsilon;\boldsymbol{\theta}_0)
\le\left\lceil\frac{8}{\eta\kappa}\log\frac1\varepsilon\right\rceil$.

We prove the lower bound by considering a small perturbation along the
normal direction with the smallest local curvature. This direction is
the locally slowest contraction direction of gradient descent.

Choose an interior point
$\bar{\boldsymbol{\theta}}^\circ\in\mathcal U_\tau$ and define
\begin{equation}
\ell
:=
\lambda_{\min}
\left(
\nabla^2 L_n(\bar{\boldsymbol{\theta}}^\circ)
\big|_{N_{\bar{\boldsymbol{\theta}}^\circ}\mathcal M_\tau}
\right).
\label{eq:proof-local-ell}
\end{equation}
Thus, $\ell$ is the smallest Hessian curvature over the normal
directions at $\bar{\boldsymbol{\theta}}^\circ$.
By Lemma~\ref{lem:uniform-normal-curvature},
$\kappa\le\ell\le(C_\Omega/c_\Omega)\kappa$.
At an optimal parameter,
$\nabla^2 L_n(\bar{\boldsymbol{\theta}}^\circ)
=
\mathcal J_{\bar{\boldsymbol{\theta}}^\circ,n}^{*}
\mathcal J_{\bar{\boldsymbol{\theta}}^\circ,n}$.
Hence the local Jacobian bound gives $\ell\le M_J^2$.
Since the Hessian is symmetric and
$\operatorname{ker}(\nabla^2 L_n(\bar{\boldsymbol{\theta}}^\circ))
=
T_{\bar{\boldsymbol{\theta}}^\circ}\mathcal M_\tau$,
the normal space is invariant under the Hessian. Therefore, there
exists a unit vector
$\mathbf{U}^\circ\in N_{\bar{\boldsymbol{\theta}}^\circ}\mathcal M_\tau$
such that
\begin{equation}
\nabla^2 L_n(\bar{\boldsymbol{\theta}}^\circ)\mathbf{U}^\circ=\ell \mathbf{U}^\circ.
\label{eq:proof-normal-eigenvector}
\end{equation}
Thus, $\mathbf{U}^\circ$ is a normal direction with the smallest local
curvature.

For $t>0$, define
$\boldsymbol{\theta}_t
=
\bar{\boldsymbol{\theta}}^\circ+t\mathbf{U}^\circ$.
For sufficiently small $t$, Lemma~\ref{lem:local-tubular-geometry}
gives $d(\boldsymbol{\theta}_t)=t$.
We first verify that these perturbations are admissible initializations.
Let
$b^\circ:=b_{\mathcal T}(\bar{\boldsymbol{\theta}}^\circ)>0$.
Since the distance to $\mathcal T^c$ is $1$-Lipschitz,
$b_{\mathcal T}(\boldsymbol{\theta}_t)\ge b^\circ-t$.
Thus, if $t$ is sufficiently small so that
$t<r$, $t<b^\circ/2$, and
$t<b^\circ\sqrt{\kappa}/(8M_J)$, then
\begin{equation}
\frac{4M_J}{\sqrt{\kappa}}
d(\boldsymbol{\theta}_t)
<
b_{\mathcal T}(\boldsymbol{\theta}_t).
\label{eq:proof-admissible-init}
\end{equation}
Hence all sufficiently small nonzero
$\boldsymbol{\theta}_t$ are admissible initializations.

We next determine how fast gradient descent contracts this perturbation.
Define
$T_\eta(\boldsymbol{\theta})
=
\boldsymbol{\theta}
-
\eta\nabla L_n(\boldsymbol{\theta})$.
Since
$\nabla L_n(\bar{\boldsymbol{\theta}}^\circ)=0$,
the point $\bar{\boldsymbol{\theta}}^\circ$ is a fixed point of
$T_\eta$. Moreover,
$D T_\eta(\bar{\boldsymbol{\theta}}^\circ)
=
I-\eta\nabla^2 L_n(\bar{\boldsymbol{\theta}}^\circ)$.
Therefore, using
\eqref{eq:proof-normal-eigenvector}, for every fixed finite integer
$K$,
\begin{equation}
D(T_\eta^K)(\bar{\boldsymbol{\theta}}^\circ)\mathbf{U}^\circ
=
(1-\eta\ell)^K \mathbf{U}^\circ.
\label{eq:proof-K-derivative}
\end{equation}
Hence, to first order, the perturbation along $\mathbf{U}^\circ$ is reduced by
the factor $(1-\eta\ell)^K$ after $K$ gradient-descent steps.

We now make this local linearization precise. Fix $K$ and then let
$t$ tend to zero. Since $L_n$ is twice continuously differentiable
in the local regularity region, $T_\eta$ is continuously
differentiable there. For sufficiently small $t$, the first $K$
iterates remain in this region, and
\begin{equation}
T_\eta^K(\boldsymbol{\theta}_t)
=
\bar{\boldsymbol{\theta}}^\circ
+
t(1-\eta\ell)^K \mathbf{U}^\circ
+
R_K(t),
\qquad
\frac{\|R_K(t)\|}{t}
\longrightarrow
0
\quad
\text{as } t\downarrow0.
\label{eq:proof-linearization}
\end{equation}
Here $K$ is fixed before taking the limit $t\downarrow0$.
Since $\ell\le M_J^2$ and
$\eta\le1/(2M_J^2)$, we have
$0<\eta\ell\le1/2$, and therefore
$(1-\eta\ell)^K>0$.
For sufficiently small $t$,
Lemma~\ref{lem:local-tubular-geometry} gives
$d(\bar{\boldsymbol{\theta}}^\circ
+t(1-\eta\ell)^K \mathbf{U}^\circ)
=
t(1-\eta\ell)^K$.
Because the distance to $\mathcal M_\tau$ is $1$-Lipschitz,
\eqref{eq:proof-linearization} implies
\begin{equation}
\lim_{t\downarrow0}
\frac{
d(T_\eta^K(\boldsymbol{\theta}_t))
}{
d(\boldsymbol{\theta}_t)
}
=
(1-\eta\ell)^K.
\label{eq:proof-distance-linearization}
\end{equation}

Suppose that $K$ iterations guarantee
$d(T_\eta^K(\boldsymbol{\theta}_0))
\le
\varepsilon d(\boldsymbol{\theta}_0)$
for every admissible initialization.
Since all sufficiently small
$\boldsymbol{\theta}_t$ constructed above are admissible,
\eqref{eq:proof-distance-linearization} implies
$(1-\eta\ell)^K\le\varepsilon$.
Therefore,
$K\ge
\log(1/\varepsilon)/[-\log(1-\eta\ell)]$.
Since $0<\eta\ell\le1/2$ and
$-\log(1-x)\le2x$ for $0<x\le1/2$, we obtain
\begin{equation}
\sup_{\boldsymbol{\theta}^0\in\mathcal{A}_\tau} K_{\text{GD}}(\varepsilon;\theta^0)
\ge
\frac{1}{2\eta\ell}
\log\frac{1}{\varepsilon}
\ge
\frac{c_\Omega}{2C_\Omega}
\frac{1}{\eta\kappa}
\log\frac{1}{\varepsilon}.
\label{eq:proof-lower-complexity}
\end{equation}

\section{Verification for a Concrete Population Model}
\label{app:verification}

In this section, we provide simple sufficient conditions under which the richness condition in Assumption~\ref{assume} holds, showing that it is satisfied by a broad class of task distributions, including our experimental setting. We then quantify its dependence on the temperature parameter for the experimental distribution.

Let a task consist of an initial distribution $P_0$, transition
kernel $P$, feature table
$\Phi:\mathcal S\times\mathcal A\to\mathbb R^d$, reward table
$R(a,s')$, and an input weight $\mathbf{w}$. Given a task, a trajectory is
generated by a fixed behavior policy $\pi_b$, and
$\phi_i=\Phi(s_i,a_i)$,
$\psi_{i,a}=\Phi(s_{i+1},a)$, and
$r_{i+1}=R(a_i,s_{i+1})$.

\subsection{Qualitative verification of observable richness}\label{app:verification1}

\begin{theorem}[Continuous-population verification]
\label{thm:continuous-richness}
Under Assumption~\ref{assume}(i), assume that:

(i) the finite-dimensional random vector formed by $\mathbf{w}$, the entries of the feature $\Phi$, and the entries of the reward $R$ has a nonzero absolutely continuous component with respect to Lebesgue measure;

(ii) there exist distinct states $s_0,s_1$ and, when $n>1$, distinct actions $a_x,a_z$ such that, almost surely over tasks,
$P_0(s_0)>0$, 
$P(s_1\mid s_0,a_x)>0$, 
$P(s_1\mid s_1,a_z)>0$, 
and
$\pi_b(a_x\mid s_0)>0$, 
$\pi_b(a_z\mid s_1)>0$.
For $n=1$, only $P_0(s_0)>0$, $P(s_1\mid s_0,a_x)>0$, and $\pi_b(a_x\mid s_0)>0$ are required.

Then, for every fixed $n\ge1$ and $\tau>0$, Assumption~\ref{assume} (ii) holds.
\end{theorem}

\begin{proof}
Fix $n\ge1$ and $\tau>0$. Recall that
$\mathsf H_{\tau,n}=\mathbb E_{\mathbf{z}}[\mathcal B_{\tau,n}(\mathbf z)^*
\mathcal B_{\tau,n}(\mathbf z)]$. Hence, for every coefficient vector $\zeta$,
$\langle \zeta,\mathsf H_{\tau,n}\zeta\rangle
=\mathbb E [\|\mathcal B_{\tau,n}(\mathbf z)[\zeta]\|^2 ]$.  Therefore, it is enough to prove that $\mathcal B_{\tau,n}(\mathbf z)[\zeta]=0$ almost surely implies $\zeta=0$. Let $T_n$ denote the random discrete trajectory and $\mathcal{E}_n=\{T_n=w_n\}$ denote the calibration family. Since
$\mathbb E|\mathcal B_{\tau,n}(\mathbf z)[\zeta]|^2
\ge
\mathbb E\left[|\mathcal B_{\tau,n}(\mathbf z)[\zeta]|^2\mathbf 1_{\mathcal{E}_n}\right]$, 
it suffices to show that every $\zeta\neq0$ is detected on ${\mathcal{E}_n}$, i.e.,
$\mathbb E\left[|\mathcal B_{\tau,n}(\mathbf z)[\zeta]|^2\mathbf 1_{\mathcal{E}_n}\right]>0$.

Consider the deterministic calibration path $\omega_n$ consisting of $s_0\xrightarrow{a_x}s_1$, followed, when $n>1$, by $n-1$ repetitions of $s_1\xrightarrow{a_z}s_1$. For $n=1$, only the first transition is retained. Let $X_n$ collect the continuous coordinates
$W=\mathbf w$, $Y=\Phi(s_0,a_x)$, $\Psi_a=\Phi(s_1,a)$, $a\in[m]$, together with $R_0=R(a_x,s_1)$ and, when $n>1$,
$R_1=R(a_z,s_1)$. Define the subprobability measure $\nu_n(E):=\mathbb P\left(X_n\in E,\,T_n=\omega_n\right)$.

The calibration path $\omega_n$ is not a null event. Let $\mu$
denote the task distribution, and define $p_{\omega_n}(\Xi)
:=
\Pr(T_n=\omega_n\mid\Xi)$, 
where $\Xi$ denotes the sampled task variables. By assumption (ii),
every factor in $p_{\omega_n}(\Xi)$ is strictly positive almost
surely. Hence $p_{\omega_n}(\Xi)>0$ for $\mu$-almost every task, and
$\Pr(T_n=\omega_n)
=
\mathbb{E}_{\mu}[p_{\omega_n}(\Xi)]
>0$.  Let $\mu_X$ denote the marginal distribution of $X_n$, and define $\ell_n(x):=\mathbb{E}_{\mu}
\bigl[p_{\omega_n}(\Xi)\mid X_n=x\bigr]$.
Since $p_{\omega_n}(\Xi)>0$ almost surely, we have
$\ell_n(x)>0$ for $\mu_X$-almost every $x$. By the tower property,
for every Borel set $E$, $\nu_n(E)
=\Pr(X_n\in E,\ T_n=\omega_n)=\int_E \ell_n(x)\,\mu_X(\dif x)$.
By assumption (i), the joint distribution of $(\mathbf{w},\Phi,R)$ has a
nonzero absolutely continuous component. Since $X_n$ is a
coordinate subvector of $(\mathbf{w},\Phi,R)$, its marginal distribution
$\mu_X$ also has a nonzero absolutely continuous component,
denoted by $\mu_X^{\mathrm{ac}}$. The measure 
$\widetilde{\nu}_n(\dif x)
:=
\ell_n(x)\,\mu_X^{\mathrm{ac}}(\dif x)$
is absolutely continuous with respect to Lebesgue measure.
Moreover, it is nonzero because
$\mu_X^{\mathrm{ac}}\ne0$ and $\ell_n>0$
$\mu_X^{\mathrm{ac}}$-almost everywhere. Since
$\mu_X^{\mathrm{ac}}\le\mu_X$, we also have
$\widetilde{\nu}_n\le\nu_n$. Therefore, $\nu_n$ has a nonzero
absolutely continuous component.

Suppose that
$\mathcal B_{\tau,n}(\mathbf z)[\zeta]=0$ almost surely under the joint
task-trajectory distribution. Restricting to the event $T_n=\omega_n$ gives $\mathcal B_{\tau,n}\bigl(\mathbf z_{\omega_n}(x)\bigr)[\zeta]
=0$ for $\nu_n$-almost every $x$. 
Each component of the left-hand side is real analytic in the
continuous coordinates $x$, since it is formed from finite sums,
products, exponentials, and divisions by strictly positive softmax
denominators. Because $\nu_n$ has a nonzero absolutely continuous
component, each component vanishes on a set of positive Lebesgue
measure. A nonzero real-analytic scalar function on $\mathbb R^q$
has a Lebesgue-null zero set. Hence every component must vanish
identically. Therefore, $\mathcal B_{\tau,n}(\mathbf z)[\zeta]=0$ implies
$$\mathcal B_{\tau,n}(\mathbf z_{\omega_n}(x))[\zeta]=0, \text{ for all } x.$$
We now use this identity to identify the
coefficient blocks in $\zeta$.

Let $a_x=1$, set $Y=y$, $\Psi_1=x$, and
$\Psi_a=0$ for every $a\neq1$. Since $a_z\neq a_x$ when $n>1$,
all later current-state features on the calibration path vanish.
Writing $r=R_0$, 
$n\mathcal B_{\tau,n}[\zeta]=0$ reduces to
\begin{align*}
0=E_Cy\bigl(r-\mathbf w^\top y+\gamma \mathbf w^\top\bar\psi\bigr)+y\Bigl[
y^\top(e_-+E_-\mathbf w)+\bar\psi^\top(e_++E_+\mathbf w)+r\,e_q^\top \mathbf w+\mathbf w^\top\Sigma E_A\mathbf w\Bigr],
\end{align*}
for every $x,y,\mathbf w\in\mathbb R^d$ and $r\in\mathbb R$. 

The coefficient of $r$ is
$[E_C+(e_q^\top\mathbf w)I_d]y$. Setting $\mathbf w=0$ and varying
$y$ gives $E_C=0$. The remaining identity
$(e_q^\top\mathbf w)y=0$ for every $\mathbf w,y$ then gives
$e_q=0$. After substituting $E_C=e_q=0$, the scalar multiplying $y$ must vanish for every $y\neq0$. Its
dependence on $y$ is affine, so its linear coefficient satisfies
$e_-+E_-\mathbf w=0$ for every $\mathbf w$. Setting $\mathbf w=0$
gives $e_-=0$, and varying $\mathbf w$ then yields $E_-=0$.
The remaining identity is
\begin{equation}
\bar\psi^\top(e_++E_+\mathbf w)
+
\mathbf w^\top\Sigma E_A\mathbf w
=
0.
\label{eq:E-successor-identity}
\end{equation}

For the candidate feature set $(x,0,\ldots,0)$, define $p=\frac{\exp(\mathbf w^\top x/\tau)}
{\exp(\mathbf w^\top x/\tau)+m-1}$. Then $\bar\psi=px$ and $\Sigma=p(1-p)xx^\top$. Since $m\ge2$, we have $0<p<1$.
Dividing \eqref{eq:E-successor-identity} by $p$ gives
\begin{equation}
x^\top(e_++E_+\mathbf w)
+
(1-p)(\mathbf w^\top x)(x^\top E_A\mathbf w)
=
0.
\label{eq:E-successor-reduced}
\end{equation}

Set $x=tv$ with $t\neq0$ and divide
\eqref{eq:E-successor-reduced} by $t$. The second term is $O(t)$ as
$t\to0$, while the first term becomes
$v^\top(e_++E_+\mathbf w)$. Letting $t\to0$ gives
$v^\top(e_++E_+\mathbf w)=0$ for every $v,\mathbf w$. Hence
$e_++E_+\mathbf w=0$ for every $\mathbf w$, and therefore
$e_+=0$ and $E_+=0$.

Equation~\eqref{eq:E-successor-reduced} now reduces to
\begin{equation}
(\mathbf w^\top x)(x^\top E_A\mathbf w)
=
0
\qquad
\text{for every }x,\mathbf w.
\label{eq:E-EA}
\end{equation}
Suppose, for contradiction, that $E_A\mathbf w\neq0$ for some
$\mathbf w$. Then $\mathbf w\neq0$, and the two proper hyperplanes
$\{x:\mathbf w^\top x=0\}$ and
$\{x:x^\top E_A\mathbf w=0\}$ cannot cover $\mathbb R^d$. Hence
there exists $x$ for which both factors in \eqref{eq:E-EA} are
nonzero, a contradiction. Therefore $E_A\mathbf w=0$ for every
$\mathbf w$, and thus $E_A=0$.
We have shown that all seven coefficient blocks vanish: $E_C=e_q=e_-=E_-=e_+=E_+=E_A=0$. Hence $\zeta=0$. 
\end{proof}

\begin{corollary}[Verification under the experimental generator]
\label{cor:experimental-richness}
Let $N_s\ge2$ and $m\ge2$. Suppose that $P_0$ and all transition
rows are mutually independent
$\operatorname{Dirichlet}(1,\ldots,1)$ random vectors, independent
of $(\mathbf w,\Phi,R)$, whose scalar entries are mutually independent
$\operatorname{Unif}(-1,1)$ random variables. Let the behavior policy
be $\epsilon_b$-greedy with uniform exploration for some
$0<\epsilon_b\le1$.
Then the hypotheses of
Theorem~\ref{thm:continuous-richness} are satisfied. Consequently, for
every fixed $n\ge1$ and $\tau>0$, Assumption~\ref{assume} holds.
\end{corollary}

\begin{proof}
We verify the two conditions of
Theorem~\ref{thm:continuous-richness}.
First, collect all scalar coordinates of $\mathbf w$, $\Phi$, and
$R$ into a finite-dimensional random vector
$X=(\mathbf w,\operatorname{vec}(\Phi),\operatorname{vec}(R))
\in\mathbb R^q$.
By assumption, the coordinates of $X$ are mutually independent
$\operatorname{Unif}(-1,1)$ random variables. Hence $X$ has density
$f_X(x)=2^{-q}\mathbf 1_{(-1,1)^q}(x)$
with respect to $q$-dimensional Lebesgue measure. In particular, the
law of $X$ is absolutely continuous, with a nonzero density on the
open set $(-1,1)^q$. Thus condition~(i) of
Theorem~\ref{thm:continuous-richness} holds.

We next verify condition~(ii). A
$\operatorname{Dirichlet}(1,\ldots,1)$ random vector is uniformly
distributed on the probability simplex and lies in its relative
interior almost surely. Therefore every coordinate is strictly positive
almost surely. In particular, for any two distinct states $s_0,s_1$, $P_0(s_0)>0$ a.s.
and, for every action $a$, $P(s_1\mid s_0,a)>0$, $P(s_1\mid s_1,a)>0$ a.s.  

It remains to check the behavior-policy probabilities. Because the
behavior policy is \(\epsilon_b\)-greedy with uniform exploration,
every action receives at least its exploration probability: 
$\pi_b(a\mid s)\ge\frac{\epsilon_b}{m}
>0$ for every $(s,a)$, independently of which action is greedy and of how ties are broken.

Hence, when $n>1$, choose any distinct states $s_0,s_1$ and any
distinct actions $a_x,a_z$. Then, almost surely, $P_0(s_0)>0$, $P(s_1\mid s_0,a_x)>0$, $P(s_1\mid s_1,a_z)>0$, 
and $\pi_b(a_x\mid s_0)>0$,
$\pi_b(a_z\mid s_1)>0$. 
Thus condition~(ii) holds. For $n=1$, only the first transition is
needed, and the same argument gives
$P_0(s_0)>0$, 
$P(s_1\mid s_0,a_x)>0$, 
$\pi_b(a_x\mid s_0)>0$ a.s. 
Therefore all hypotheses of
Theorem~\ref{thm:continuous-richness} are satisfied. Applying that
theorem yields Assumption~\ref{assume} for every fixed $n\ge1$ and
$\tau>0$.
\end{proof}

\subsection{Effect of Temperature on Convergence Rate}
In this section, we establish the lower bound on
$\kappa(\tau)$ and then construct a matching normal
direction for the upper bound. By the normal-curvature comparison
established in Lemma~\ref{lem:uniform-normal-curvature}, it is sufficient for the lower bound to control the smallest eigenvalue
$\chi_{\tau,n}:=\lambda_{\min}(\mathsf H_{\tau,n})$ of the
coefficient-space Gram matrix. Equivalently, we need to prove a
uniform lower bound of $\mathbb E_{\mathcal D_n}\left|
\mathcal B_{\tau,n}(z)[\zeta]
\right|_2^2$
for every coefficient direction $\zeta$.

Lemma~\ref{lem:calibration-reduction} further reduces this task to the
reference law $U_n$ on the calibration path. Hence it suffices to
establish the corresponding coercivity bound for
$\mathbb E_{U_n}|\mathcal B_{\tau,n}(z_{\omega_n})[\zeta]|_2^2$.
The only coefficient block whose response degenerates is the attention-score block in the first layer of Transformer. We therefore first isolate
the successor-value and attention contribution
$K_\tau
=\bar\Psi_\tau^\top(e_++E_+W)+W^\top\Sigma_\tau E_AW$
and derive a quantitative lower bound for
$\mathbb E_{U_n}K_\tau^2$, which shows in Lemma \ref{lem:attention-coercivity}. We then combine this estimate
with orthogonal projections controlling the remaining coefficient
blocks $(E_C,e_q,e_-,E_-)$, which yields
the lower bound of $\kappa(\tau)$.

\begin{lemma}[Calibration-path reduction]
\label{lem:calibration-reduction}
Under the experimental generator, let \(\omega_n\) be the calibration
path used in Theorem~\ref{thm:continuous-richness}, and let \(U_n\)
denote the product-uniform law of the continuous coordinates appearing
along this path. Then, for every coefficient direction \(\zeta\),
\[
\mathbb E_{\mathcal D_n}
\left\|
\mathcal B_{\tau,n}(\mathbf{z})[\zeta]
\right\|_2^2
\ge
p_n\,
\mathbb E_{U_n}
\left\|
\mathcal B_{\tau,n}(\mathbf{z}_{\omega_n})[\zeta]
\right\|_2^2,
\]
where
\[
p_n
=
\frac{p_{\min}^n}{N_s^2}
\frac{\Gamma(N_s)\Gamma(n)}
{\Gamma(N_s+n-1)}
>0,
\qquad
p_{\min}:=\frac{\epsilon_b}{m}.
\]
\end{lemma}
\begin{proof}
    We use the same calibration path as in Theorem~\ref{thm:continuous-richness}, $s_0\xrightarrow{a_x}s_1$, $s_1\xrightarrow{a_z}s_1$  for the remaining $n-1$ transitions, with the second part omitted when $n=1$. Let $U_n$ denote the product-uniform law of the continuous coordinates appearing along this path, and let $\mathbf{z}_{\omega_n}$ be the corresponding context.

Because the behavior policy selects every action with probability at least $p_{\min}$, the conditional probability of following the prescribed action sequence is at least $p_{\min}^n$. The remaining state-transition contribution is 
$$
Q_n = P_0(s_0) P(s_1\mid s_0,a_x) P(s_1\mid s_1,a_z)^{n-1}, 
$$
where the final factor is omitted for $n=1$. Hence the probability of the calibration path, conditional on the task, is bounded below by $p_{\min}^nQ_n$.  
Since the initial-state probability \(P_0(s_0)\) and the transition probability \(P(s_1\mid s_0,a_x)\) each have mean \(1/N_s\). The repeated self-transition probability \(P(s_1\mid s_1,a_z)\) has the \(\operatorname{Beta}(1,N_s-1)\) distribution, so
\[
\mathbb E Q_n =\mathbb E  P_0(s_0)\mathbb E  P(s_1\mid s_0,a_x) \mathbb E P(s_1\mid s_1,a_z)^{n-1}=
\frac{1}{N_s^2}
\frac{\Gamma(N_s)\Gamma(n)}
{\Gamma(N_s+n-1)}.
\] 

Under the experimental generator, the environment variables in \(Q_n\) are independent of the selected continuous coordinates $(\mathbf{w},\Phi,R)$. Therefore, for every nonnegative measurable function \(f\),
\begin{align*}
\mathbb{E}_{\mathcal{D}_n} f(\mathbf{Z}) &\ge \mathbb{E}_{\mathcal{D}_n}\left[f(\mathbf{Z})\mathbf{1}_{\{T_n=\omega_n\}}\right] = \mathbb{E}\left[f(\mathbf{z}_{\omega_n}(X_n))\mathbb{P}(T_n = \omega_n \mid \Xi)\right] \\
&\ge p_{\min}^n \mathbb{E}\left[f(\mathbf{z}_{\omega_n}(X_n))Q_n\right]=p_{\min}^n\mathbb{E}\left[f(\mathbf{z}_{\omega_n}(X_n))\right]\mathbb{E}[Q_n]\ge p_n \mathbb E_{U_n}f(\mathbf{z}_{\omega_n}).
\end{align*}
where $p_n
=
\frac{p_{\min}^n}{N_s^2}
\frac{\Gamma(N_s)\Gamma(n)}
{\Gamma(N_s+n-1)}
>0$.  Let $f(\mathbf{z}) = \|\mathcal{B}_{\tau,n}(\mathbf{z})[\zeta]\|^2$, we have the desired result.
\end{proof}

Let $\zeta=(E_C,e_q,e-,E_-,e_+,E_+,E_A)$ 
be the observable coefficient direction induced by $\mathbf{U}$. We decompose the attention coefficient $E_A$ as
$D:=E_A=D_0+aI_d$, 
$a:=\frac{\operatorname{tr}(D)}{d}$, $\operatorname{tr}(D_0)=0$. 
We first isolate the part of the response associated with the successor-feature and attention coefficients. Under the reference law $U_n$, let $W$ denote the weight vector and $\Psi_1,\ldots,\Psi_m$ the successor features. Their contribution is summarized by
$K_\tau
:=
\bar{\Psi}\tau^\top(e^+ + E_+W)
+
W^\top\Sigma_\tau D W$.

Lemma \ref{lem:attention-coercivity} shows that the temperature dependence of $K_\tau$ determines the weakest curvature contributed by the attention block and, in particular, reveals the $\tau^5$ scaling of the scalar attention direction.

\begin{lemma} 
\label{lem:attention-coercivity}
Assume $d\ge2$. Under the reference law $U_n$,  there exist constants $b_{\rm att}>0$ and $\tau_*>0$, independent
of $\tau$ and of the coefficient direction, such that for every
$0<\tau\le\tau_*$,
\begin{align}
\mathbb E_{U_n} K_\tau^2\ge b_{\rm att}
\left(\|e_+\|^2+\|E_+\|_F^2+\tau^3\|D_0\|_F^2+\tau^5 d a^2\right).
\label{eq:tau5-attention-coercivity}
\end{align}
\end{lemma}

\begin{proof}
(i) We first lower-bound $\mathbb E[(W^\top\Sigma_\tau D W)^2]$. The difficulty is that
$D_0$ may point in an arbitrary trace-free matrix direction. We first
construct a finite collection of measurements $u^\top D_0v$ that
controls the entire matrix $D_0$. Define
\begin{align*}
\mathcal F=\{(e_i,e_j):i\neq j\}\cup
\left\{\left(\frac{e_i-e_j}{\sqrt2},\frac{e_i+e_j}{\sqrt2}\right),\left(\frac{e_i+e_j}{\sqrt2},
\frac{e_i-e_j}{\sqrt2}
\right):i<j\right\}.
\end{align*}
The pairs $(e_i,e_j)$ directly measure the off-diagonal entries of
$D_0$, while the rotated pairs also measure differences between its
diagonal entries. Since $\operatorname{tr}(D_0)=0$, we have
\begin{align}
\sum_{(u,v)\in\mathcal{F}} \left(u^\top D_0 v\right)^2
&= \sum_{i\ne j} (D_0)_{ij}^2
+ \sum_{i<j}\frac12\bigl((D_0)_{ii}-(D_0)_{jj}\bigr)^2
+
\frac12\bigl((D_0)_{ij}-(D_0)_{ji}\bigr)^2\nonumber\\
&=\sum_{i\ne j} (D_0)_{ij}^2+ \frac{d}{2}\sum_i (D_0)_{ii}^2
+ \frac12\sum_{i<j} \big((D_0)_{ij}-(D_0)_{ji}\big)^2\nonumber\\
&\ge \sum_{i\ne j} (D_0)_{ij}^2 + \sum_{i} (D_0)_{ii}^2=
\|D_0\|_F^2.
\label{eq:tau5-frame}
\end{align}
Hence the finite family $\mathcal F$ detects every nonzero trace-free
direction.

For each $(u,v)\in\mathcal{F}$, we construct a local configuration in which the attention response contains the measurement $u^\top D_0v$. Choose fixed sufficiently small constants $s,t,b>0$ and set $W=sv$, $V:=\Psi_1-\Psi_2=tu$, and $M:=(\Psi_1+\Psi_2)/2=bv$, with $\Psi_a=-bv$ for $a\ge3$. Equivalently, $\Psi_1=bv+(t/2)u$ and $\Psi_2=bv-(t/2)u$. 

Since $u^\top v=0$, actions $1$ and $2$ have the same score, $W^\top\Psi_1=W^\top\Psi_2=sb$, whereas every action $a\ge3$ has score $-sb$. Thus actions $1$ and $2$ are tied, while all remaining actions are separated from them by a fixed positive margin. Let $g:=W^\top V=W^\top(\Psi_1-\Psi_2)$ denote the score gap between actions $1$ and $2$. At the center, $g=0$. In a sufficiently small neighborhood of each center, we use $g$, together with the remaining local coordinates $\upsilon$, as coordinates. The corresponding reference density can be written as $j_\ell(\upsilon,g)\,\dif\upsilon\,\dif g$, with $0<j_-\le j_\ell(\upsilon,g)\le j_+<\infty$. Since only finitely many charts are used, these constants can be chosen uniformly. 

We restrict each chart to the near-tie region $\tau\le |g|\le 2\tau$. In this region, actions $1$ and $2$ both receive nonnegligible softmax mass, whereas all other actions remain separated by a fixed score gap and contribute only exponentially small terms. Using the pairwise representation of $\Sigma_\tau$, we obtain 
\begin{align}\label{eq:tau5-leading-response}
W^\top\Sigma_\tau D W  = \pi_1^\tau\pi_2^\tau\,g\,V^\top D W + R_\tau,  
|R_\tau|  \le C e^{-g_0/\tau}\|D\|_F . 
\end{align} 
Since $D=D_0+aI_d$ and $V^\top W=g$, \begin{align} gV^\top D W = gV^\top D_0W+ag^2. \label{eq:tau5-D-decomposition} 
\end{align} 
The trace-free component therefore contains one factor of the small gap $g$, whereas the scalar component contains two. To make the dependence on $\tau$ explicit, set $g=\tau x$ and $a_e:=\tau a$. Then 
\begin{align}\label{eq:tau5-rescaling} 
gV^\top D_0W+ag^2 = \tau \left( xV^\top D_0W+a_ex^2 \right). 
\end{align}
Squaring contributes a factor $\tau^2$, while $\dif g=\tau\,\dif x$ contributes one additional factor $\tau$. Hence the near-tie contribution has an overall factor $\tau^3$. After removing this common factor, the corresponding quadratic forms converge to \begin{align}\label{eq:tau5-G0} 
G_0(D_0,a_e) &= \sum_{\ell} \int_{Z_\ell} j_\ell(\upsilon,0) \int_{1\le |x|\le 2} k(x) \left( x f_\ell(\upsilon;D_0)+a_ex^2 \right)^2 \,\dif x\,\dif\upsilon, \end{align} where $k(x):=e^{2x}/(1+e^x)^4>0$ and $f_\ell(\upsilon;D_0) :=V_\ell(\upsilon,0)^\top D_0W_\ell(\upsilon,0)$. The role of $G_0$ is to rule out cancellation between the trace-free part $D_0$ and the scalar part $aI_d$. Suppose $G_0(D_0,a_e)=0$. Since the integrand is nonnegative and the weights are strictly positive, we must have $x f_\ell(\upsilon;D_0)+a_ex^2=0$ throughout each chart. Evaluating this identity at $x$ and $-x$, with $1<x<2$, gives  
$x f_\ell+a_ex^2 = 0, -x f_\ell+a_ex^2 = 0$.  
Hence $a_e=0$ and $f_\ell=0$. At the center of the $\ell$-th chart, $V_\ell=tu_\ell$ and $W_\ell=sv_\ell$, and therefore $f_\ell(0;D_0)=st\,u_\ell^\top D_0v_\ell=0$. Since this holds for every $(u_\ell,v_\ell)\in\mathcal{F}$, \eqref{eq:tau5-frame} implies $D_0=0$. Thus $G_0$ is positive definite. Because the coefficient space is finite dimensional, compactness of its unit sphere gives a constant $\lambda_0>0$ such that 
\begin{align} 
G_0(D_0,a_e) &\ge \lambda_0 \left( \|D_0\|_F^2+d a_e^2 \right). \label{eq:tau5-G0-lower} 
\end{align} 
Moreover, after the rescaling $g=\tau x$ and $a=a_e/\tau$,
the contribution of the remaining actions to the rescaled
response is uniformly $O(\tau^{-2}e^{-g_0/\tau})=o(1)$ on
the coefficient unit sphere
$\operatorname{tr}(D_0)=0$,
$\|D_0\|_F^2+d a_e^2=1$.
The coordinate maps and densities also converge uniformly
on the fixed chart domains, chosen with compact closures.
Thus $G_\tau$ converges uniformly to $G_0$ on this unit sphere,
equivalently in operator norm for the associated Gram operators.
Hence, for all sufficiently small $\tau>0$,
\begin{align}
G_\tau(D_0,a_e) :=
\tau^{-3}\sum_{\ell}\int_{C_{\ell,\tau}}
\left[
W^\top\Sigma^\tau
\left(D_0+\frac{a_e}{\tau}I_d\right)W
\right]^2\,dP \ge
\frac{\lambda_0}{2}
\left(\|D_0\|_F^2+d a_e^2\right).
\label{eq:tau5-Gtau-lower}
\end{align}
Let $\mathcal{C}_{\ell,\tau}$ denote the near-tie region in the $\ell$-th chart and define $\mathcal{E}_{\rm tie}(\tau) :=\bigcup_\ell\mathcal{C}_{\ell,\tau}$. Since there are only $L:=|\mathcal{F}|$ charts, their overlap multiplicity is at most $L$. Therefore, for every nonnegative function $h$, \begin{align} 
\sum_\ell \int_{\mathcal{C}_{\ell,\tau}} h\,dP &\le L\, \mathbb{E} \left[ h\,\mathbf{1}_{\mathcal{E}_{\rm tie}(\tau)} \right]. \label{eq:tau5-chart-overlap} \end{align} Combining this finite-overlap bound with \eqref{eq:tau5-Gtau-lower}, and recalling that $a_e=\tau a$, gives \begin{align} \mathbb{E} \left[ \bigl(W^\top\Sigma_\tau D W\bigr)^2 \mathbf{1}_{\mathcal{E}_{\rm tie}(\tau)} \right] &\ge b_3\tau^3\|D_0\|_F^2 + b_5\tau^5 d a^2 - C e^{-g_0/\tau}\|D\|_F^2. \label{eq:tau5-D-lower} \end{align} Thus the near-tie regions provide order-$\tau^3$ control of the trace-free component $D_0$ and the weaker order-$\tau^5$ control of the scalar component $aI_d$.

(ii) We next control cancellation with \(E_+\). All expectations below are under the reference law $U_n$. Write $X_\tau:=\bar\Psi_\tau^{\mathsf T}E_+W$, 
$Y_\tau:=W^{\mathsf T}\Sigma_\tau DW$, 
$Z_\tau:=\bar\Psi_\tau^{\mathsf T}e_+$, 
where $D=D_0+aI_d$ and $\operatorname{tr}(D_0)=0$.
In particular, 
$K_\tau=Z_\tau+X_\tau+Y_\tau$, 
$\Frob{D}^{\,2}=\Frob{D_0}^{\,2}+da^2$. 
We choose all constants independently of $\tau$ and the coefficient
direction. The parameters $d$ and $m$ are fixed.

 Choose fixed, sufficiently small positive numbers $r_*,a_*,b_*$, and
consider the following $d^2$ configurations in the interior of the
product-uniform support:
\begin{align*}
&W=r_*e_j,\qquad \Psi_1=b_*e_j,
\qquad \Psi_a=-b_*e_j\quad(a\geq2),\qquad j=1,\ldots,d;\\
&W=r_*e_j,\qquad \Psi_1=a_*e_i+b_*e_j,
\qquad \Psi_a=-b_*e_j\quad(a\geq2),\qquad i\ne j.
\end{align*}
At each configuration, action~1 has a fixed positive winning gap.
Take sufficiently small, fixed open neighborhoods of these configurations
whose closures lie inside the support, and denote their union by
$\mathcal E_{\mathrm{sep}}$. On this union, every gap
$W^{\mathsf T}(\Psi_1-\Psi_a)$, $a\geq2$, is bounded below by a common
constant $g_{\mathrm{sep}}>0$.

For $X_0(E):=\Psi_1^{\mathsf T}EW$, the center measurements are
$b_*r_*E_{jj}$ and $a_*r_*E_{ij}+b_*r_*E_{jj}$.
They jointly determine every entry of $E$. Hence the quadratic form
$\E[X_0(E)^2\ind_{\mathcal E_{\mathrm{sep}}}]$ is positive definite:
if it vanished, positivity of the reference density and continuity
would force $X_0(E)$ to vanish on each neighborhood, including at every
center, and therefore $E=0$. The diagonal centers similarly show that
$\E[(\Psi_1^{\mathsf T}e)^2\ind_{\mathcal E_{\mathrm{sep}}}]$
is positive definite in $e$. Compactness of the respective unit spheres
therefore gives constants $\lambda_E,\lambda_e>0$ such that
\begin{align}
\E[X_0(E)^2\ind_{\mathcal E_{\mathrm{sep}}}]
 \geq\lambda_E\Frob{E}^{\,2}, \quad
\E[(\Psi_1^{\mathsf T}e)^2\ind_{\mathcal E_{\mathrm{sep}}}] \geq\lambda_e\norm{e}^{2}.
\label{eq:abs-separated-grams}
\end{align}
All features and $W$ are bounded. The uniform winning gap gives
\begin{equation}
\norm{\bar\Psi_\tau-\Psi_1}
+\norm{\Sigma_\tau}_{\mathrm{op}}
\leq C e^{-g_{\mathrm{sep}}/\tau}
\quad\text{on }\mathcal E_{\mathrm{sep}}.
\label{eq:abs-separated-approximation}
\end{equation}
For example, the total softmax mass outside action~1 is at most
$(m-1)e^{-g_{\mathrm{sep}}/\tau}$, which bounds both quantities in
\eqref{eq:abs-separated-approximation} up to fixed constants.
The reverse triangle inequality in the restricted $L^2$ norm, together
with \eqref{eq:abs-separated-grams}, implies that, for all sufficiently
small $\tau$,
\begin{align}
\E[X_\tau^2\ind_{\mathcal E_{\mathrm{sep}}}]
 \geq\frac{\lambda_E}{4}\Frob{E_+}^{\,2}, \quad
\E[Z_\tau^2\ind_{\mathcal E_{\mathrm{sep}}}] \geq\frac{\lambda_e}{4}\norm{e_+}^{2}.
\label{eq:abs-separated-finite-temperature}
\end{align}
These estimates hold simultaneously for all coefficient directions,
because the $L^2$ approximation errors are bounded by
$C e^{-g_{\mathrm{sep}}/\tau}\Frob{E_+}$ and
$C e^{-g_{\mathrm{sep}}/\tau}\norm{e_+}$, respectively.
Also, \eqref{eq:abs-separated-approximation} yields
$\E[Y_\tau^2\ind_{\mathcal E_{\mathrm{sep}}}]
\leq C_{\mathrm{sep}}e^{-2g_{\mathrm{sep}}/\tau}\Frob{D}^{\,2}$.
The elementary inequality $(u+v)^2\geq\tfrac12u^2-v^2$ for 
$u,v\in\mathbb R$ 
follows by completing the square. Applying it with $u=X_\tau$ and
$v=Y_\tau$, and weakening the exponential remainder if necessary, gives
\begin{equation}
\E[(X_\tau+Y_\tau)^2\ind_{\mathcal E_{\mathrm{sep}}}]
\geq b_E\Frob{E_+}^{\,2}
-C_{\mathrm{sep}}e^{-g_{\mathrm{sep}}/\tau}\Frob{D}^{\,2},
\qquad b_E:=\frac{\lambda_E}{8}>0.
\label{eq:abs-separated-bound}
\end{equation} 

Let $\mathcal E_{\mathrm{tie}}(\tau)$ be the union of the near-tie
charts from part~(i). Their bounded densities, fixed bounded coordinate
domains, and gap restriction $\tau\leq |g|\leq2\tau$ imply 
$\mathbb P(\mathcal E_{\mathrm{tie}}(\tau))\leq C_t\tau$. 
Indeed, each chart contributes at most its upper density bound times
its coordinate-domain volume times the gap-interval length $2\tau$;
summing over the finitely many charts gives this bound regardless of
overlap. Since $\norm{\bar\Psi_\tau},\norm{W}\leq\sqrt d$,
\begin{equation}
\E[X_\tau^2\ind_{\mathcal E_{\mathrm{tie}}(\tau)}]
\leq C_X\tau\Frob{E_+}^{\,2},\qquad C_X:=d^2C_t.
\label{eq:abs-tie-leakage}
\end{equation}
Rewrite (E.13), retaining its coefficient-independent constants, as
\begin{equation}
\E[Y_\tau^2\ind_{\mathcal E_{\mathrm{tie}}(\tau)}]
\geq b_3\tau^3\Frob{D_0}^{\,2}+b_5\tau^5da^2
-C_{\mathrm{tie}}e^{-g_0/\tau}\Frob{D}^{\,2}.
\label{eq:abs-assumed-tie-bound}
\end{equation}
Apply $(u+v)^2\geq\tfrac12u^2-v^2$ with $u=Y_\tau$ and $v=X_\tau$.
Equations \eqref{eq:abs-tie-leakage} and
\eqref{eq:abs-assumed-tie-bound} give
\begin{align}
\E[(X_\tau+Y_\tau)^2\ind_{\mathcal E_{\mathrm{tie}}(\tau)}]
\geq \frac{b_3}{2}\tau^3\Frob{D_0}^{\,2}
+\frac{b_5}{2}\tau^5da^2-C_X\tau\Frob{E_+}^{\,2} -\frac{C_{\mathrm{tie}}}{2}e^{-g_0/\tau}\Frob{D}^{\,2}.
\label{eq:abs-tie-total}
\end{align}
No disjointness of the two regions is needed: their indicators sum to
at most two. Consequently,
\begin{align*}
\E[(X_\tau+Y_\tau)^2]
\geq\frac12\bigl( 
\E[(X_\tau+Y_\tau)^2\ind_{\mathcal E_{\mathrm{sep}}}] +\E[(X_\tau+Y_\tau)^2\ind_{\mathcal E_{\mathrm{tie}}(\tau)}]
\bigr).
\end{align*}
Set $c=\min\{g_0,g_{\mathrm{sep}}\}>0$ and
$C_0=C_{\mathrm{sep}}/2+C_{\mathrm{tie}}/4$.
Combining the two regional lower bounds yields
\begin{align}
\E[(X_\tau+Y_\tau)^2]\geq 
\left(\frac{b_E}{2}-\frac{C_X\tau}{2}\right)\Frob{E_+}^{\,2}
+\frac{b_3}{4}\tau^3\Frob{D_0}^{\,2}
+\frac{b_5}{4}\tau^5da^2 -C_0e^{-c/\tau}\bigl(\Frob{D_0}^{\,2}+da^2\bigr).
\label{eq:abs-before-absorption}
\end{align}
Choose $\tau_*>0$ smaller than the thresholds used above and in
part~(i), so that for every $0<\tau\leq\tau_*$,
\begin{equation}
\tau\leq1,\qquad
C_X\tau\leq\frac{b_E}{2},\qquad
C_0e^{-c/\tau}\leq\frac{\min\{b_3,b_5\}}{8}\tau^5.
\label{eq:abs-threshold}
\end{equation}
Such a choice exists because $e^{-c/\tau}/\tau^5\to0$ as
$\tau\downarrow0$. Every quantity used to choose this threshold is
independent of $e_+$, $E_+$, and $D$.
Using $\tau^5\leq\tau^3$ and
\eqref{eq:abs-threshold} in \eqref{eq:abs-before-absorption} gives
\begin{equation}
\displaystyle
\E[(X_\tau+Y_\tau)^2]
\geq\frac{b_E}{4}\Frob{E_+}^{\,2}
+\frac{b_3}{8}\tau^3\Frob{D_0}^{\,2}
+\frac{b_5}{8}\tau^5da^2. 
\label{eq:abs-even-lower-bound}
\end{equation} 

The product-uniform law of $(W,\Psi_1,\ldots,\Psi_m)$ is invariant
under simultaneous sign reversal. This transformation preserves all
softmax logits and $\Sigma_\tau$, reverses $\bar\Psi_\tau$, and makes
$Z_\tau$ odd while $X_\tau+Y_\tau$ is even. Therefore $\E[Z_\tau(X_\tau+Y_\tau)]=0$,
$\E[K_\tau^2]=\E[Z_\tau^2]+\E[(X_\tau+Y_\tau)^2]$.
The second estimate in \eqref{eq:abs-separated-finite-temperature}
implies $\E[Z_\tau^2]\geq(\lambda_e/4)\norm{e_+}^2$.
Combining this with \eqref{eq:abs-even-lower-bound} proves
\begin{equation}
\E[K_\tau^2]\geq b_{\mathrm{att}}
\left(\norm{e_+}^2+\Frob{E_+}^{\,2}
+\tau^3\Frob{D_0}^{\,2}+\tau^5da^2\right),
\label{eq:abs-final}
\end{equation}
where one may take $b_{\mathrm{att}}=
\min\left\{\frac{\lambda_e}{4},\frac{b_E}{4},
\frac{b_3}{8},\frac{b_5}{8}\right\}>0$. 
 Both $b_{\mathrm{att}}$ and $\tau_*$ are independent
of the coefficient direction. They can also be chosen independently
of $n$, since all quantities in this argument, including the near-tie
construction in part~(i), depend only on the same product-uniform
marginal of $(W,\Psi_1,\ldots,\Psi_m)$ under every $U_n$.
\end{proof}

\begin{proof}[The proof of Corollary 
\ref{thm:sharp-temperature1}]
    (i) Since $\Obs(\mathbf{z})[\zeta]=0$, $\mathcal D_n\text{-a.s.}\Longrightarrow \zeta=0$ has already been established in Corollary 
\ref{cor:experimental-richness}; This step is quantitative only. Lemma \ref{lem:attention-coercivity} analyzed $e_+,E_+,E_A$, we now incorporate $E_C,e_q,e_-,E_-$.  

On the calibration path, let 
$Y:=\Phi(s_0,a_x)$, 
$R_0:=R(a_x,s_1)$, 
$Z_Y:=YY^\top-\frac13I_d$, and
$v_\tau:=W^\top\bar\Psi_\tau$. 
Conditional on all variables other than \(Y\) and \(R_0\), the
observable response can be written as
\begin{equation}
n\mathcal B_{\tau,n}(z_{\omega_n})[\zeta]
=
R_0A_r(Y)
+
A_o(Y)
+
A_2(Y)
+
A_{\rm const},
\label{eq:tau5-orthogonal-decomp}
\end{equation}
where $A_r(Y)=[E_C+(e_q^\top W)I_d]Y$,
$A_o(Y)=\gamma v_\tau E_CY+K_\tau Y$,
$A_2(Y)=-E_CZ_YW+Z_Y(e_-+E_-W)$,
and $A_{\rm const}=\frac13(-E_CW+e_-+E_-W)+A_{\rm later}$.
Here $A_{\rm later}$ is the sum of the contributions from transitions
$1,\ldots,n-1$, and is zero when $n=1$.

The calibration path has \(s_0\neq s_1\), and all later transitions use
features associated with \(s_1\). Hence, under the product reference
law, \(A_{\rm later}\) is independent of \(Y=\Phi(s_0,a_x)\) and of
\(R_0\). Thus \(A_{\rm const}\) is constant with respect to the
conditional variables \((Y,R_0)\).

The four terms in \eqref{eq:tau5-orthogonal-decomp} are mutually
orthogonal in $L^2(Y,R_0)$. Indeed, $\mathbb ER_0=0$ and $R_0$
independent on $Y$; $A_o$ is odd in $Y$, i.e., $A_o(-Y)=-A_o(Y)$; Since $Z_{-Y}=Z_Y$, $A_2$ is even in $Y$ and
centered because $\mathbb E Z_Y=0$; and $A_{\rm const}$ is
constant in $(Y,R_0)$. Odd and even functions are orthogonal under symmetric distributions. In particular, later transitions cannot
cancel any of the first three components.

The uniform moments satisfy $\mathbb E[YY^\top]=\mathbb E[WW^\top]=\frac13I_d$,
$\mathbb E R_0^2=\frac13$,
and $\mathbb E Z_Y^2=\nu I_d$,
$\nu:=\frac{5d-1}{45}>0$.
For the reward projection,
\begin{align}
\mathbb E\|R_0A_r(Y)\|^2=\frac19
\mathbb E\left\|E_C+(e_q^\top W)I_d
\right\|_F^2=\frac19\|E_C\|_F^2+\frac{d}{27}\|e_q\|^2,
\label{eq:tau5-reward-proj}
\end{align}
where the cross term vanishes because \(\mathbb EW=0\).

For \(0<\delta<1\), we use
$\|x+y\|^2\ge\delta\|y\|^2-\frac{\delta}{1-\delta}\|x\|^2$. Since $|v_\tau|\le d$,
\begin{equation}
\mathbb E\|A_o(Y)\|^2
\ge
\frac{\delta_od}{3}\,
\mathbb E K_\tau^2
-
\frac{\delta_o}{1-\delta_o}
\frac{\gamma^2d^2}{3}
\|E_C\|_F^2.
\label{eq:tau5-odd-proj}
\end{equation}
Moreover, $Z_Ye_-$ is orthogonal to the terms linear in the centered
variable $W$. Using
$\mathbb E\|Z_YE_-W\|^2=\frac{\nu}{3}\|E_-\|_F^2$, $\mathbb E\|E_CZ_YW\|^2
=\frac{\nu}{3}\|E_C\|_F^2$, we obtain
\begin{equation}
\mathbb E\|A_2(Y)\|^2\ge\nu\|e_-\|^2+\frac{\delta_2\nu}{3}\|E_-\|_F^2-\frac{\delta_2}{1-\delta_2}\frac{\nu}{3}\|E_C\|_F^2.
\label{eq:tau5-quadratic-proj}
\end{equation}
Choose $\delta_o=\min\left\{\frac12,
\frac{1}{48\gamma^2d^2}\right\}$, $\delta_2=\min\left\{\frac12,\frac{1}{48\nu}\right\}$. Since $\delta/(1-\delta)\le2\delta$ for $\delta\le1/2$,
each negative coefficient multiplying $\|E_C\|_F^2$ in
\eqref{eq:tau5-odd-proj} and
\eqref{eq:tau5-quadratic-proj} is at most \(1/72\). Hence the
coefficient \(1/9\) in \eqref{eq:tau5-reward-proj} retains at least $\frac1{12}$.
Combining the orthogonal decomposition,
\eqref{eq:tau5-attention-coercivity},
\eqref{eq:tau5-reward-proj},
\eqref{eq:tau5-odd-proj}, and
\eqref{eq:tau5-quadratic-proj}, there exists
\(b_{\rm ref}>0\), independent of \(n,\tau\), and \(\zeta\), such that
for all sufficiently small \(0<\tau\le1\),
\begin{align}
\mathbb E_{U_n}
\left\|
\mathcal B_{\tau,n}(\mathbf{z}_{\omega_n})[\zeta]
\right\|_2^2
\ge
\frac{b_{\rm ref}}{n^2}
\Bigl(
&
\|E_C\|_F^2+\|e_q\|^2
+\|e_-\|^2+\|E_-\|_F^2
\nonumber\\
&
+\|e_+\|^2+\|E_+\|_F^2
+\tau^3\|D_0\|_F^2
+\tau^5d\,a^2
\Bigr).
\label{eq:tau5-reference-lower}
\end{align}
Since $\|D\|_F^2=\|D_0\|_F^2+da^2$
and \(0<\tau\le1\), the right-hand side is bounded below by 
$\frac{b_{\rm ref}}{n^2}
\tau^5
\|\zeta\|_{\mathcal V}^2$.
Using Lemma \ref{lem:calibration-reduction},
\begin{equation}
\mathbb E_{\mathcal D_n}
\left\|
\mathcal B_{\tau,n}(\mathbf{z})[\zeta]
\right\|_2^2
\ge
\frac{p_nb_{\rm ref}}{n^2}
\tau^5
\|\zeta\|_{\mathcal V}^2.
\label{eq:tau5-population-lower}
\end{equation}
Therefore $\chi_{\tau,n}\ge\frac{p_nb_{\rm ref}}{n^2}\tau^5$. By Lemma~\ref{lem:uniform-normal-curvature}, 
$\kappa(\tau,n)
\ge
c_\Omega
\frac{p_nb_{\rm ref}}{n^2}
\tau^5$.
Since $\mathcal N$ is finite, 
$c_{\rm exp}:=c_\Omega b_{\rm ref}\min_{n\in\mathcal N}\frac{p_n}{n^2}>0$,
and hence $\kappa(\tau)
\ge c_{\rm exp}\tau^5$, 
$n\in\mathcal N$.

(ii) For the upper bound, it suffices to exhibit a single unit normal
direction whose response is \(O(\tau^5)\). Consider the pure scalar
attention perturbation 
$\mathbf{U}_{\rm sc}
=(\frac{I_d}{\sqrt d},0,\ldots,0)$.
Every tangent vector to \(\mathcal M_\tau\) has zero \(A\)-block, so
\(\mathbf{U}_{\rm sc}\in N_{\bar\theta}\mathcal M_\tau\) at every
\(\bar\theta\in\mathcal M_\tau\), and
\(\|\mathbf{U}_{\rm sc}\|=1\).

For a fixed successor state, define $g_{ab}
=
W^\top(\Psi_a-\Psi_b)$. 
This random variable has a uniformly bounded density. Conditional on
\(W\neq0\), let $M_W:=\max_j|W_j|$
and choose a coordinate attaining this maximum. Since the difference
of two independent \(\operatorname{Unif}(-1,1)\) variables has density
bounded by \(1/2\), convolution with the remaining independent
coordinates yields
$\|f_{g_{ab}\mid W}\|_\infty\le\frac{1}{2M_W}$. Furthermore, $\mathbb P(M_W\le t)=t^d$, $\mathbb E M_W^{-1}=\frac{d}{d-1}$, so, because $d\ge2$,
$\|f_{g_{ab}}\|_\infty\le L_d:=\frac{d}{2(d-1)}<\infty$. 

Set $V_\tau:=W^\top\Sigma_\tau W$.
The pairwise covariance identity gives
$V_\tau=\sum_{a<b}\pi_a^\tau\pi_b^\tau g_{ab}^2$. Moreover, 
$\pi_a^\tau\pi_b^\tau\le e^{-|g_{ab}|/\tau}$.
Indeed, if the larger of the two logits is \(q_a\), the softmax
denominator is at least \(e^{2q_a/\tau}\) after multiplying the two
probabilities.

Let \(P_m:=m(m-1)/2\). By Cauchy--Schwarz,
\begin{align}
\mathbb EV_\tau^2
\le P_m\sum_{a<b}\mathbb E\left[g_{ab}^4
e^{-2|g_{ab}|/\tau}\right]\le 2L_dP_m^2
\int_0^\infty x^4e^{-2x/\tau}\,\dif x=C_V\tau^5,
\label{eq:tau5-V-upper}
\end{align}
where \(C_V<\infty\) depends only on \(d\) and \(m\).

Direct differentiation in the direction $\mathbf{U}_{\rm sc}$ gives 
$J_{\bar\theta}(\mathbf{z})[\mathbf{U}_{\rm sc}]
=
\frac{\alpha\gamma}{n\sqrt d}
\sum_{i=0}^{n-1}
\phi_iV_{\tau,i}$.
Using convexity of the squared norm and
$\|\phi_i\|^2\le d$,
$\mathbb E\|J_{\bar\theta}(\mathbf{z})[\mathbf{U}_{\rm sc}]\|_2^2\le\frac{\alpha^2\gamma^2}{n}
\sum_{i=0}^{n-1}\mathbb EV_{\tau,i}^2$. 

The visited successor state may depend on the task, so we do not
condition on it. If \(V_\tau(s)\) denotes the same scalar covariance
computed using the feature table at a fixed state \(s\), then
\[
\mathbb EV_{\tau,i}^2
=
\sum_{s\in\mathcal S}
\mathbb E
\left[
\mathbf 1_{\{s_{i+1}=s\}}
V_\tau(s)^2
\right]
\le
\sum_{s\in\mathcal S}
\mathbb EV_\tau(s)^2
\le
N_sC_V\tau^5.
\]
Therefore 
$\mathbb E
\|J_{\bar\theta}(\mathbf{z})[\mathbf{U}_{\rm sc}]\|_2^2
\le
\alpha^2\gamma^2N_sC_V\tau^5$.
Since \(\kappa(\tau)\) is the minimum squared response over
unit normal directions, $\kappa(\tau,n)
\le C_{\rm exp}\tau^5$, 
$C_{\rm exp}:=\alpha^2\gamma^2N_sC_V$, 
uniformly over all $n\ge1$.

Finally, by the uniform local derivative bounds established in
Lemma~\ref{lem:one-step-contraction}, the same admissible step size
\(\eta\) may be used throughout the small-temperature regime. Applying
the parameter-distance complexity result of
Theorem~\ref{thm:main} and substituting
$\kappa(\tau)=\Theta(\tau^5)$ yields $\sup_{\boldsymbol{\theta}^0\in\mathcal{A}_\tau} K_{\text{GD}}(\varepsilon;\boldsymbol{\theta}^0,\tau,n)=\Theta_{\boldsymbol{\Omega},\mathcal N}(
\frac{1}{\eta\tau^5}
\log\frac{1}{\varepsilon})$.
\end{proof}

\end{document}